\documentclass{article}

    \PassOptionsToPackage{numbers, sort}{natbib}

\usepackage[final]{neurips_2022}

\usepackage[utf8]{inputenc} 
\usepackage[T1]{fontenc}    
\usepackage{hyperref}       
\usepackage{url}            
\usepackage{booktabs}       
\usepackage{amsfonts}       
\usepackage{nicefrac}       
\usepackage{microtype}      
\usepackage{xcolor}         

\usepackage{aligned-overset}
\usepackage{bbm}
\usepackage{bm}
\usepackage{calrsfs}
\usepackage{colonequals}
\usepackage{centernot}
\usepackage{xfrac}
\usepackage{url}

\usepackage{amsmath}
\usepackage{amssymb}
\usepackage{mathtools}
\usepackage{amsthm}

\usepackage{microtype}
\usepackage{graphicx}
\usepackage{subcaption}
\usepackage{wrapfig}
\usepackage[rightcaption]{sidecap}

\title{
Better Uncertainty Calibration via Proper Scores for Classification and Beyond
}

\author{%
  Sebastian G. Gruber \\
  German Cancer Research Center (DKFZ), Germany \\
  German Cancer Consortium (DKTK), Frankfurt, Germany \\
  Goethe University Frankfurt, Germany \\
  \texttt{sebastian.gruber@dkfz.de} \\
    \And
    Florian Buettner \\
    German Cancer Research Center (DKFZ), Germany \\
    German Cancer Consortium (DKTK), Frankfurt, Germany \\
    Frankfurt Cancer Institute, Germany \\
    Goethe University Frankfurt, Germany \\
   \texttt{florian.buettner@dkfz.de} \\
}

\begin{document}

\theoremstyle{plain}
\newtheorem{theorem}{Theorem}[section]
\newtheorem{proposition}[theorem]{Proposition}
\newtheorem{lemma}[theorem]{Lemma}
\newtheorem{corollary}[theorem]{Corollary}
\theoremstyle{definition}
\newtheorem{definition}[theorem]{Definition}
\newtheorem{assumption}[theorem]{Assumption}
\theoremstyle{remark}
\newtheorem{remark}[theorem]{Remark}
\newtheorem{example}[theorem]{Example}

\maketitle

\begin{abstract}
With model trustworthiness being crucial for sensitive real-world applications, practitioners are putting more and more focus on improving the uncertainty calibration of deep neural networks.
Calibration errors are designed to quantify the reliability of probabilistic predictions but their estimators are usually biased and inconsistent.
In this work, we introduce the framework of \textit{proper calibration errors}, which relates every calibration error to a proper score and provides a respective upper bound with optimal estimation properties.
This relationship can be used to reliably quantify the model calibration improvement.
We theoretically and empirically demonstrate the shortcomings of commonly used estimators compared to our approach.
Due to the wide applicability of proper scores, this gives a natural extension of recalibration beyond classification.
\end{abstract}

\section{Introduction}

Deep learning became a dominant cornerstone of machine learning research in the last decade and deep neural networks can surpass human-level predictive performance on a wide range of tasks \citep{he2015delving, schrittwieser2020mastering, HAGGENMULLER2021202}.
However, \citet{guo2017calibration} have shown that for modern neural networks, better classification accuracy can come at the cost of systematic overconfidence in their predictions.
Practitioners in sensitive forecasting domains, such as cancer diagnostics \citep{HAGGENMULLER2021202}, genotype-based disease prediction \citep{KatsaouniTashkandiWieseSchulz+2021+871+885} or climate prediction \citep{yen2019application}, require for models to not only have high predictive power but also to reliably communicate uncertainty.
This raises the need to quantify and improve the quality of predictive uncertainty, ideally via a dedicated metric.
An uncertainty-aware model should give probabilistic predictions which represent the true likelihood of events depending on the very prediction.
To quantify the extend to which this condition is violated, calibration errors have been introduced.
In general, their estimators are usually biased \citep{roelofs2021mitigating} and inconsistent \citep{vaicenavicius2019evaluating}.
This, in turn, is highly problematic since we cannot quantify how reliable a model is if we do not know how reliable the metric is.
Especially the medical field is a domain that requires high model trustworthiness, but with low expert availability and/or disease frequency we often encounter a small data regime.
Resampling strategies can be viable options for optimization on small datasets but also reduce the evaluation set size even more.
We will discover that little data exacerbates the estimation bias and propose reliable alternatives for improving uncertainty calibration.

\begin{figure*}[t]
\vskip 0.2in
\centering
    \begin{subfigure}{.13\textwidth}
    \centering
    \includegraphics[width=\columnwidth]{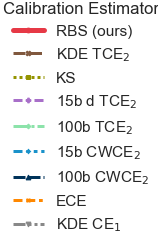}
    \end{subfigure}%
    \begin{subfigure}{.31\textwidth}
    \centering
    \includegraphics[width=\columnwidth]{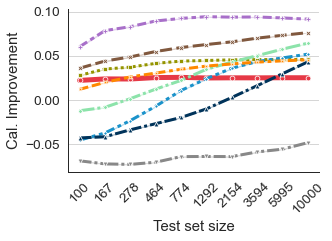}
    \caption{TS of DenseNet 40 on \\ CIFAR10}
    \end{subfigure}%
    \begin{subfigure}{.28\textwidth}
    \centering
    \includegraphics[width=\columnwidth]{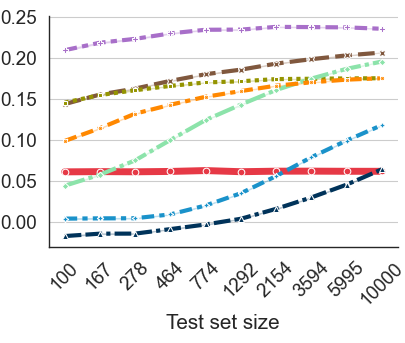}
    \caption{ETS of Wide ResNet 32 \\ on CIFAR100}
    \end{subfigure}%
    \begin{subfigure}{.28\textwidth}
    \centering
    \includegraphics[width=\columnwidth]{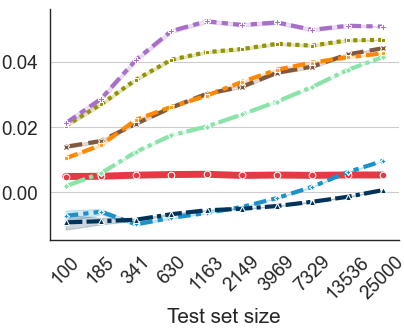}
    \caption{DIAG of DenseNet 161 on ImageNet}
    \end{subfigure}
\caption[]{
    Estimated calibration improvement for various settings. The calibration error is estimated before and after a recalibration method (TS / ETS / DIAG) is applied and the difference (i.e. calibration improvement) is shown for increasing test set size.
    All common calibration estimators are sensitive with respect to the test set size and can substantially over- or underestimate the effect of performing recalibration.\footnotemark Only RBS robustly estimates the improvement in calibration error for all test set sizes. 
}
\label{fig:RC_errors}
\vskip -0.2in
\end{figure*}

Since deep neural networks often yield uncalibrated confidence scores \citep{minderer2021revisiting}, a variety of different post-hoc recalibration approaches have been proposed \citep{doi:10.1080/01621459.1982.10477856, kumar2019verified}.\footnotetext{For consistency with other calibration estimators, we refer to ECE$^{\mathrm{KDE}}$ proposed by \citep{popordanoska2022} as KDE CE$_1$.}
These methods use the validation set to transform predictions returned by a trained neural network such that they become better calibrated.
A key desired property of recalibration methods is to not reduce the accuracy after the transformation.
Therefore, most modern approaches are restricted to accuracy preserving transformations of the model outputs \citep{Platt99probabilisticoutputs, guo2017calibration, zhang2020mix, rahimi2020intra}.
When recalibrating a model, it is crucial to have a reliable estimate of how much the chosen method improves the underlying model.
However, when using current estimators for calibration errors, their biased nature results in estimates that are highly sensitive to the number of samples in the test set that are used to compute the calibration error before and after recalibration (Fig. \ref{fig:RC_errors}; c.f. Section \ref{sec:exp}). 
The source code is openly available at \url{https://github.com/MLO-lab/better_uncertainty_calibration}.

Our \textbf{contributions} for better
uncertainty calibration are summarized in the following. We...

\begin{itemize}
    \item ... give an overview of current calibration error literature, place the errors into a taxonomy, and show which are insufficient for calibration quantification. This also includes several theoretical results, which highlight the shortcomings of current approaches.
    \item ... introduce the framework of \textbf{proper calibration errors}, which gives important guarantees and relates every element to a proper score. We can reliably estimate the improvement of an injective recalibration method w.r.t. a proper calibration error via its related proper score - even in non-classification settings.
    \item ... show that common calibration estimators are highly sensitive w.r.t. the test set size. We demonstrate that for commonly used estimators, the estimated improvement of recalibration methods is heavily biased and becomes monotonically worse with fewer test data.
\end{itemize}

\section{Related Work}

In this section, we give an extensive overview of published work regarding quantifying model calibration and model recalibration.
Important definitions will be directly given, while others are placed in Appendix \ref{app:defs}.
These will be the basis for our theoretical findings.
Further, we will motivate the definition of proper calibration errors, which are directly related to proper scores.
Consequently, we will also present important aspects of the framework around proper scores in later parts of this section.

\subsection{Calibration errors}

For clarity, we introduce shortly the notation used throughout this work.
Assume we have random variables $X$ and $Y$ corresponding to feature and target variable, and feature and target space $\mathcal{X}$ and $\mathcal{Y}$.
We have $\mathbb{P}_Y, \mathbb{P}_{Y \mid X} \in \mathcal{P}$, where $\mathbb{P}_{Y}$ refers to the distribution of $Y$, $\mathbb{P}_{Y \mid X}$ to the conditional distribution given $X$, and $\mathcal{P}$ a set of distributions on $\mathcal{Y}$.
Even though some approaches explore calibration for regression tasks \citep{song2019distribution, widmann2021calibration, chung2021pinball}, it is most dominantly considered for classification.
To distinguish between the general case and $n$-class classification, we refer to $\mathcal{P}_n$ as the $n$-dimensional simplex of corresponding categorical distributions.

A popular task is the calibration of the predicted top-label $C = \arg\max_k f_k \left( X \right)$ of a model $f \colon \mathcal{X} \to \mathcal{P}_n$ \citep{guo2017calibration, pmlr-v80-kumar18a, maddox2019simple, nixon2019measuring, vaicenavicius2019evaluating, Tomani_2021_CVPR, rahimi2020intra}.
Here, the top-label confidence should represent the accuracy of this very prediction.
Formally, we try to reach the condition 
$f_C \left( X \right) = \mathbb{P} \left(Y = C \mid f_C \left( X \right) \right)$.
However, the condition is weaker as one might expect, and referring to a model fulfilling this condition as (perfectly) calibrated can give a false sense of security \citep{vaicenavicius2019evaluating, widmann2019calibration}.
This holds especially in forecasting domains, where low probability estimates can still be highly relevant.
For example, assigning probability mass to an aggressive type of cancer can still trigger action even if it is not predicted as the most likely outcome.
Further, we might also be interested in predictive uncertainty for non-classification tasks.
Consequently, we use the stricter and more general condition that the complete prediction $f \left( X \right)$ should be equal to the complete conditional distribution $\mathbb{P}_{Y \mid f \left( X \right)}$ of the target given this very prediction as introduced by \citet{widmann2021calibration}.
More formally, we state:

\begin{definition}
    A model $f \colon \mathcal{X} \to \mathcal{P}$ is \textbf{calibrated} if and only if $f \left( X \right) = \mathbb{P}_{Y \mid f \left( X \right)}$.
\label{def:strongly_cal}
\end{definition}

Note that $\mathcal{P}$ can be a set of arbitrary distributions, which incorporates $\mathcal{P}_n$ (classification) as a special case.


One of the first metrics for assessing model calibration that is still widely used in recent literature is the Brier score (BS) \citep{ANewVectorPartitionoftheProbabilityScore, Platt99probabilisticoutputs, gupta2020calibration, rahimi2020intra}.
For a model $f \colon \mathcal{X} \to \mathcal{P}_n$ the \textbf{Brier score} \citep{VERIFICATIONOFFORECASTSEXPRESSEDINTERMSOFPROBABILITY} is defined as 
\begin{equation}
    \text{BS} \left( f \right) = \mathbb{E} \left[ \left\| f \left( X \right) - Y^\prime \right\|^2_2 \right],
\label{def:bs}
\end{equation}
where $Y^\prime$ is one-hot-encoded $Y$.
The estimator of the BS is equivalent to the mean squared error, illustrating that it does not purely capture model calibration.
Rather, the BS can be interpreted as a comprehensive measure of model performance, simultaneously capturing model fit and calibration.
This becomes more obvious via the canonical decomposition of the BS into a calibration and sharpness term \citep{ANewVectorPartitionoftheProbabilityScore}.
Based on this decomposition, we can derive the following error.
For $1 \leq p \in \mathbb{R}$, the \textbf{$L^p$ calibration error} (CE$_p$) of model $f \colon \mathcal{X} \to \mathcal{P}_n$ is defined as 
\begin{equation}
    \text{CE}_p \left( f \right) = \left( \mathbb{E} \left[ \left\| f \left( X \right) - \mathbb{P}_{Y \mid f \left( X \right)} \right\|^p_p \right] \right)^{\frac{1}{p}}.
\label{def:lp_ce}
\end{equation}
The BS decomposition only supports the squared case, but a general $L^p$ formulation became more common in recent years \citep{vaicenavicius2019evaluating, widmann2019calibration, zhang2020mix, popordanoska2022}.
\citet{popordanoska2022} proposed to estimate CE$_p$ via multivariate kernel density estimation.
In general, calibration estimation is difficult due to the term $\mathbb{P}_{Y \mid f \left( X \right)}$ since we never have samples of every possible prediction for continuous models.
This is in contrast to the original work of \citet{ANewVectorPartitionoftheProbabilityScore}, where only models with a finite prediction space are considered. 
To assess the calibration of a continuous binary model \citet{Platt99probabilisticoutputs} used histogram estimation, transforming the infinite prediction space to a finite one.
This is also referred to as equal width binning.
Similarly, \citet{nguyen2015posterior} introduced an equal mass binning scheme for continuous binary models.
Both, equal width and equal mass binning schemes, suffer from the requirement of setting a hyperparameter.
This can significantly influence the estimated value \citep{kumar2019verified} and there is no optimal default since every setting has a different bias-variance tradeoff \citep{nixon2019measuring}.
The first calibration estimator for a continuous one-vs-all multi-class model was given by \citet{10.5555/2888116.2888120} and is still the most commonly used measure to quantify calibration.
It is referred to as the \textbf{expected calibration error} (ECE) of model $f \colon \mathcal{X} \to \mathcal{P}_n$ and defined as
\begin{equation}
\begin{split}
    & \text{ECE} \left( f \right) = \sum_{i=1}^m p_i \left\lvert \text{conf}_i - \text{acc}_i \right\rvert
\label{def:ece}
\end{split}
\end{equation}
with the bin frequency $p_i = \mathbb{P} \left( f_C \left( X \right) \in B_i \right)$, the bin-wise mean confidence $\text{conf}_i = \mathbb{E} \left[ f_C \left( X \right) \mid f_C \left( X \right) \in B_i \right]$, and the bin-wise accuracy $\text{acc}_i = \mathbb{P} \left( Y = C \mid f_C \left( X \right) \in B_i \right)$, for $m \in \mathbb{N}$ bins $B_i \coloneqq \left( \frac{i - 1}{m}, \frac{i}{m} \right]$.
We can estimate $p_i$, $\text{conf}_i$, and $\text{acc}_i$ via the respective means.
These are then used in equation \eqref{def:ece} to estimate the ECE.
This estimator is biased \citep{kumar2019verified, vaicenavicius2019evaluating}.

\citet{kull2019beyond} and \citet{nixon2019measuring} independently introduced another calibration estimator, which also captures the extend to which the condition $\mathbb{P} \left(Y = k \mid f_k \left( X \right) \right) = f_k \left( X \right)$ is violated for each class $k \in \mathcal{Y}$.
They respectively use equal width and equal mass binning.
Similar to these estimators, \citet{kumar2019verified} introduced the \textbf{class-wise calibration error} (CWCE$_p$) and, similar to the ECE, the \textbf{top-label calibration error} (TCE$_p$).
They only formulated the squared case $p=2$ but suggested the extension of their definitions to general $p$-norms, which we provide in Appendix \ref{app:defs}.

Furthermore, \citet{kumar2019verified} and \citet{vaicenavicius2019evaluating} proved independently that using a fixed binning scheme for estimation leads to a lower bound of the expected error.
\citet{zhang2020mix} circumvent binning schemes by using kernel density estimation to estimate the TCE$_p$.

Orthogonal ways to quantify model miscalibration have been proposed to not depend on binning or kernel density estimation schemes.
\citet{gupta2020calibration} introduced the \textbf{Kolmogorov-Smirnov calibration error} (KS) (c.f. Appendix \ref{app:defs}), which is based on the Kolmogorov-Smirnov test between empirical cumulative distribution functions.
Its estimator does not require setting a hyperparameter but the authors did not provide further theoretical aspects.

Estimators of the TCE$_p$ and CWCE$_p$ are in general not differentiable.
Consequently, \citet{pmlr-v80-kumar18a} proposed the \textbf{Maximum mean calibration error} (MMCE) (c.f. Appendix \ref{app:defs}), which has a differentiable estimator.
It is a kernel-based error, comparing the top-label confidence and the conditional accuracy, similar to the ECE.

\citet{widmann2019calibration} argued that the MMCE is insufficient for quantifying calibration of a model, similar as the ECE.
They further proposed the \textbf{Kernel calibration error} (KCE) (c.f. Appendix \ref{app:defs}).
It is based on matrix-valued kernels and unlike the MMCE, which only uses the top-label prediction, includes the whole model prediction.
The squared KCE has an unbiased estimator based on a U-statistic with quadratic runtime complexity with respect to the data size.
However, even though the KCE is positive, the U-statistic estimator can give negative values.
To this end, the authors use the estimated value as a test statistic w.r.t. the null hypothesis that the model is calibrated.

Furthermore, \citet{widmann2019calibration} and \citet{widmann2021calibration} proposed to unify different definitions of calibration errors in a theoretical framework, which also includes variance regression calibration.
However, the framework allows calibration errors, which are zero even if the model is not calibrated at all.

\subsection{Recalibration}

A plethora of recalibration methods have been proposed to improve model calibration after training by transforming the model output probabilities \citep{hastie1998classification, Platt99probabilisticoutputs, zadrozny2001calib, 10.1145/775047.775151, 10.5555/2888116.2888120, nguyen2015posterior, pmlr-v54-kull17a, guo2017calibration, kull2019beyond, kumar2019verified, zhang2020mix, gupta2020calibration, rahimi2020intra}.
These methods are optimized on a specific calibration set, which is usually the validation set.
Key desiderata of these methods include for the algorithms to be accuracy-preserving and data-efficient \citep{zhang2020mix}, reflecting that typical use-cases include settings in sensitive domains where accuracy should remain unchanged and often little data is available to train and evaluate the models.
Such accuracy-preserving methods only adjust the probability estimate in such a way that the predicted top-label remains the same.
The most commonly used accuracy-preserving recalibration method is temperature scaling (TS) \citep{guo2017calibration}, where the model logits are divided by a single parameter $T \in \mathbb{R}_{>0}$ before computing the predictions via softmax.
A more expressive extension of TS is ensemble temperature scaling (ETS) \citep{zhang2020mix}, where a weighted ensemble of TS output, model output, and label smoothing is computed.
\citet{rahimi2020intra} proposed different classes of order-preserving transformations.
A specifically interesting one is the class of diagonal intra order-preserving functions (DIAG).
Here, the model logits are transformed elementwise with a scalar, monotonic, and continuous function, which is represented by neural networks of unconstrained monotonic functions \citep{wehenkel2019unconstrained}.

\subsection{Proper scores}

\citet{gneitingscores} give an extensive overview of proper scores.
Unfortunately, their presented definitions assume maximization as the model training objective.
To stay in line with recent machine learning literature, we flip the sign when it is required in the following definitions, similar as in \citep{Br_cker_2009}.
We specifically do not constrain ourselves to classification, which is a special case.
Assume we give a prediction in $\mathcal{P}$ for an event in $\mathcal{Y}$ and we want to score how good the prediction was.
A function $S \; \colon \; \mathcal{P} \times \mathcal{Y} \to \overline{\mathbb{R}}$ with $\overline{\mathbb{R}} \coloneqq \mathbb{R} \cup \left\{- \infty, \infty \right\}$ is called \textbf{scoring rule} or just \textbf{score}.
Examples are the Brier score and the log score for classification, or the Dawid-Sebastiani score (DSS), which extends the MSE for variance regression \citep{dawid1999coherent, gneiting2014probabilistic}.
It is defined as $S \left( P, y \right) = \frac{\left(\mu_P - y \right)^2}{\sigma_P^2} + \log \sigma_P^2$.
To compare distributions, we use the expected score $s_S \colon \mathcal{P} \times \mathcal{P} \to \mathbb{R}$ defined as $s_S \left(P, Q \right) = \mathbb{E}_{Y \sim Q} \left[ S \left(P, Y \right) \right]$. 
A scoring rule $S$ is defined to be \textbf{proper} if and only if $s_S \left(P, Q \right) \geq s_S \left(Q, Q \right)$ holds for all $P, Q \in \mathcal{P}$, 
and \textbf{strictly proper} if and only if $P \neq Q \implies s_S \left(P, Q \right) > s_S \left(Q, Q \right)$.
In other words, a score is proper if predicting the target distribution gives the best expected value and strictly proper if no other prediction can achieve this value.
Given a proper score $S$ and $P, Q \in \mathcal{P}$, the associated \textbf{divergence} $d_S \colon \mathcal{P} \times \mathcal{P} \to \mathbb{R}_{\geq 0}$ is defined as
 $d_S \left(P, Q \right) = s_S \left(P, Q \right) - s_S \left(Q, Q \right)$
and the associated \textbf{generalized entropy} $g_S \colon \mathcal{P} \to \mathbb{R}$ as
$g_S \left( Q \right) = s_S \left(Q, Q \right)$.
For strictly proper $S$, $d_S$ is only zero if $P = Q$; for (strictly) proper $S$, $g_S$ is (strictly) concave.
An example of a divergence-entropy combination is the Kullback-Leibler divergence and the Shannon entropy associated to the log score.

Associated entropies and divergences are used in the calibration-sharpness decomposition introduced by \citet{Br_cker_2009} for proper scores of categorical distributions.
In Lemma \ref{lemma:cal_sharp_decomp} we will prove that this result holds for proper scores of arbitrary distributions, as long as additional conditions are met.
Further, associated divergences will be the backbone for our definition of \textit{proper calibration errors} in Section \ref{sec:pce}.

\section{Theoretical analysis of calibration errors}

In this section, we present theoretical results regarding calibration errors used in current literature.
First, we place these calibration errors into a taxonomy, which we introduce in Theorem \ref{th:ce_relations}.
Next, we give an example of how much errors lower in the hierarchy can differ from CE$_1$ in Proposition \ref{prop:neg_example}.
Last, we analyse the ECE estimate with respect to the data size in Proposition \ref{prop:ece_estimator}.
This proposition is the basis for explaining the empirical (miss-)behaviour of the ECE observed in Section \ref{sec:exp}.
All proofs are presented in Appendix \ref{app:proofs}.

To the best of the authors' knowledge, other publications provided relations between at most two distinct calibration errors or none at all while introducing a new one.
The taxonomy in the following theorem is a novel contribution, improving overview of a whole body of work regarding quantifying calibration.

\begin{theorem}
    Given a model $f \; \colon \; \mathcal{X} \to \mathcal{P}_n$ and $1 \leq p \in \mathbb{R}$, we have
    \begin{equation*}
    \begin{split}
        \text{BS} \left( f \right) = 0
        \! \implies \!\!\! \begin{Bmatrix}
        \text{CE}_p \left( f \right) = 0 \\
        \text{KCE} \left( f \right) = 0 \\
        f \text{ calibrated}
        \end{Bmatrix} 
        \!\!\! \implies \! 
        \begin{cases}
            \text{CWCE}_p \left( f \right) = 0 
            \\ \quad \\
            \begin{Bmatrix}
            \text{TCE}_p \left( f \right) = 0 \\
            \text{MMCE} \left( f \right) = 0 \\
            \text{KS} \left( f \right) = 0 \\
            \end{Bmatrix} 
            \!\!\! \implies \! \text{ECE} \left( f \right) = 0,
        \end{cases}
    \label{implies_chain}
    \end{split}
    \end{equation*}
    where statements inside curly brackets $\begin{Bmatrix} \dots \end{Bmatrix}$ are equivalent. Further, we have
    \begin{equation*}
    \begin{split}
        n^{\frac{1}{p} - \frac{1}{2}} \sqrt{\text{BS} \left( f \right)} 
        \overset{*}{\geq} \text{CE}_p \left( f \right)
        \geq 
        \begin{cases}
            \text{CWCE}_p \left( f \right) \\
            \text{TCE}_p \left( f \right) 
            \geq \text{TCE}_1 \left( f \right) 
            \geq \begin{cases}
            \text{KS} \left( f \right) \\
            \text{ECE} \left( f \right) \\
            c \cdot \text{MMCE} \left( f \right)
            \end{cases}
            \geq 0,            
        \end{cases}
    \end{split}
    \label{unequalities}
    \end{equation*}

    where * only holds for $p \leq 2$. The kernel dependent constant $c \in \mathbb{R}$ is given in Appendix \ref{th:app:ce_relations} according to \citet{pmlr-v80-kumar18a}.
\label{th:ce_relations}
\end{theorem}

From this theorem follows that it is ambiguous to refer to \textit{perfect calibration} just because there exists a calibration error which is zero for a model.
The proof of this theorem also contains $n^{\frac{1}{p} - \frac{1}{q}} \text{CE}_q \left( f \right) \geq \text{CE}_p \left( f \right)$ for $p \leq q$, which is a contradiction to Theorem 1 in \citep{wenger2020non}.
Next, we demonstrate how large the gap between calibration errors can be.

\begin{proposition}
Assume the model $f \colon \mathcal{X} \to \mathcal{P}_n$ is surjective.
There exists a joint distribution of $X$ and $Y$ such that for all $E \in \left\{ \text{MMCE}, \text{KS}, \text{ECE}, \text{TCE}_p, \text{CWCE}_p \mid 1 \leq p \in \mathbb{R} \right\} \colon$
\begin{equation*}
\begin{split}
    E \left( f \right) = 0 \quad \text{and} \quad \text{CE}_2 \left( f \right) = \sqrt{1 - \frac{1}{n-1}} 0.49.
\end{split}
\end{equation*}
\label{prop:neg_example}
\end{proposition}

\begin{wrapfigure}{l}{0.35\textwidth}
    \begin{center}
    \includegraphics[width=0.35\textwidth]{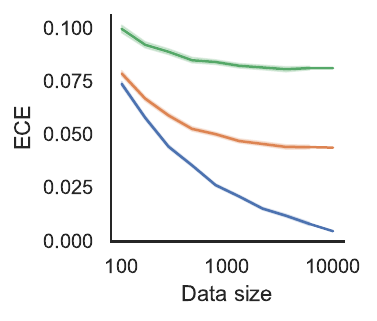}
    \caption{
    Estimated ECE of simulated models with perfect calibration (\textbf{blue}), mediocre calibration (\textbf{orange}), and bad calibration (\textbf{green}).
    Better calibration exacerbates ECE bias with respect to the data size, leading to unreliably calibration improvement quantification.
    }
    \label{fig:ece_sim}
    \end{center}
\vskip -0.2in
\end{wrapfigure}

Recall that $\text{CE}_2 \left( f \right) = 0$ if and only if $f$ is calibrated.
In other words, most used calibration errors can be zero, but the model is still far from calibrated.
An example of a model transformation with perfect ECE but uncalibrated outputs is given in Proposition \ref{prop:transform}.

Among all calibration errors, the ECE is still the most commonly used \citep{Joo2020BeingBA, Kristiadi2020BeingBE, rahimi2020intra, Tomani_2021_CVPR, minderer2021revisiting, Tian2021AGP, Islam2021ABS, Menon2021ASP, Moraleslvarez2021ActivationlevelUI, Gupta2021BERTF, Wang2021BeCT, Fan2022AcceleratingBN}, even though its estimator is knowingly biased \citep{kumar2019verified, roelofs2021mitigating}.
Let $\hat{\text{ECE}}_{\left(n \right)}$ denote the ECE estimator for $n$ data instances as originally defined in \citet{guo2017calibration}.
The following gives an analysis how this estimate behaves approximately with respect to $n$ and how this is further impacted by the ground truth ECE value.
For simplicity, we omit the fixed model from the notation.

\begin{proposition}
    For 
    $\mu_{\left(n\right)} \approx \mathbb{E} \left[ \hat{\text{ECE}}_{\left(n \right)} \right] \geq \text{ECE}$
    defined in Appendix \ref{app:sec:mu_n}, we have
    \begin{equation*}
    \begin{split}
        \frac{\mathrm{d} \mu_{\left(n\right)}}{\mathrm{d} n} < 0 \;, \quad \frac{\mathrm{d}^2 \mu_{\left(n\right)}}{\left(\mathrm{d} n \right)^2} > 0 \; \text{, and} \quad \frac{\mathrm{d}^2 \mu_{\left(n\right)}}{\mathrm{d} n \; \mathrm{d} \text{ECE}} > 0.
    \end{split}
    \end{equation*}
    
\label{prop:ece_estimator}
\end{proposition}

In words, the ECE estimator can be approximated by a differentiable function, which is strictly convex and monotonically decreasing in the data size.
The smaller the data size, the larger the change of the function.
Further, this change is also influenced by the true ECE value, such that, for low ground truth ECE, the function changes even more rapidly with the data size.
Analogous results can be found for CWCE$_1$, CWCE$_2$, and TCE$_2$ with binning estimators as used in \citet{kull2019beyond}, \citet{nixon2019measuring} and \citet{kumar2019verified}.

To confirm the goodness of the approximation, we evaluated the ECE estimator on simulated models in Figure \ref{fig:ece_sim}.
The models are designed such that their true level of calibration is known.
Including the results of Figure \ref{fig:RC_errors}, the empirical curves are consistent with Proposition \ref{prop:ece_estimator}.
Further details on the simulation are given in Appendix \ref{app:plots}.

The results in this section motivate a formal framework of proper calibration errors which are zero if and only if the model is calibrated and with robust estimators.

\section{Proper calibration errors}
\label{sec:pce}

In this section, we introduce the definition of \textit{proper calibration errors}.
We provide an easy-to-estimate upper bound and investigate some properties.
As a preliminary step, we generalize a proper score decomposition.
Again, all proofs are presented in Appendix \ref{app:proofs}.

\citet{Br_cker_2009} introduced a calibration-sharpness decomposition of proper scores w.r.t. categorical distributions.
We extend this decomposition to proper scores of arbitrary distributions.

\begin{lemma}
    Let $\mathcal{P}$ be a set of arbitrary distributions for which exists a proper score $S$ with some mild conditions. For random variables $Q$ and $Y$ with $Q, \mathbb{P}_Y, \mathbb{P}_{Y \mid Q} \in \mathcal{P}$, we have the decomposition 
    \begin{equation*}
        \mathbb{E} \left[ S \left(Q, Y \right) \right] = \underbrace{g_S \left(\mathbb{P}_Y \right)}_\text{generalized entropy} + \underbrace{\mathbb{E} \left[ d_S \left(Q, \mathbb{P}_{Y \mid Q} \right) \right]}_\text{calibration} - \underbrace{\mathbb{E} \left[ d_S \left(\mathbb{P}_Y, \mathbb{P}_{Y \mid Q} \right) \right]}_\text{sharpness}.
    \end{equation*}
\label{lemma:cal_sharp_decomp}
\end{lemma}

Substituting $Q$ with $f \left( X \right)$ and $S$ with the Brier score, the calibration term equals the previously defined CE$_2$ of a model $f$.
Lemma \ref{lemma:cal_sharp_decomp} motivates the following definition, which we introduce:

\begin{definition}
    Given a model $f \colon \mathcal{X} \to \mathcal{P}$, we say $$\text{CE}_S \left( f \right) \coloneqq \mathbb{E} \left[ d_S \left(f \left( X \right), \mathbb{P}_{Y \mid f \left( X \right)} \right) \right]$$ is a \textbf{(strictly) proper calibration error} if and only if $d_S$ is a divergence associated with a (strictly) proper score $S$.
\label{def:pce}
\end{definition}

This gives CE$_\text{BS} \equiv$ CE$_2^2$ as an example of a strictly proper calibration error for classification since the Brier score is a strictly proper score on $\mathcal{P}_n$.
Strictly proper calibration errors have the highly desired property: $\text{CE}_S \left( f \right) = 0 \overset{a.s.}{\iff} f$ is calibrated.
Since proper scores are not restricted to classification, the above definition gives a natural extension of calibration errors beyond classification.

Additionally, by generalizing the definition of proper scores, we can show that the squared KCE is a strictly proper calibration error (Appendix \ref{app:u-scores}).
But, in general, there does not exist an unbiased estimator of a proper calibration error, since we cannot estimate $\mathbb{E} \left[ g_S \left( \mathbb{P}_{Y \mid f \left( X \right)} \right) \right]$ in an unbiased manner.
Because we do not want lower bounds for errors used in sensitive applications, we introduce the following theorem about how to construct an upper bound.

\begin{theorem}
    For all proper calibration errors with $\inf_{P \in \mathcal{P}} g_S \left( P \right) \in \mathbb{R}$, there exists an associated \textbf{calibration upper bound} $$ \mathcal{U}_S \left( f \right) \geq \text{CE}_S \left( f \right)$$
    defined as $\mathcal{U}_S \left( f \right) = \mathbb{E} \left[ S \left( f \left( X \right), Y \right) \right] - \inf_{P \in \mathcal{P}} g_S \left(P\right)$.
    Under a classification setting and further mild conditions, 
    we have $ \lim_{\text{ACC} \left(f\right) \to 1} \mathcal{U}_S \left( f \right) - \text{CE}_S \left( f \right) = 0. $
\label{th:ub}
\end{theorem}

In other words, we can always construct a non-trivial upper bound of a proper calibration error as long as the generalized entropy function has a finite infimum.
The calibration upper bound approaches the true calibration error for models with high accuracy.
Our proposed calibration upper bounds are provably reliable to use since they all have a minimum-variance unbiased estimator.
In the following example, we derive the calibration upper bound $\mathcal{U}_\text{BS}$ of the Brier Score.

\begin{example}


The scoring rule induced by the Brier score is defined as $S_\text{BS} \left( f \left( X \right), Y \right) = \left\| f \left( X \right) - Y^\prime \right\|^2$, where $Y^\prime$ is the one-hot encoding of $Y$.
Using the definition of the associated entropy gives us
$g_{\text{BS}} \left( Q \right) = \mathbb{E}_{Y \sim Q} \left[ S_\text{BS} \left( Q, Y \right) \right] = \mathbb{E}_{Y \sim Q} \left[ \left\| Q - Y^\prime \right\|^2 \right]$.
To find its infimum, note that $\left\| . \right\|^2 \geq 0$ and 
$g_{\text{BS}} \left( \left(1, 0, \dots, 0 \right)^\intercal \right) = 0$.
Thus, $\inf_{P \in \mathcal{P}} g_{\text{BS}} \left(P \right) = 0$, which gives $\mathcal{U}_\text{BS} \left( f \right) = \mathbb{E} \left[ \left\| f \left( X \right) - Y^\prime \right\|^2 \right] = \mathrm{BS} \left(f\right)$.
This makes the Brier score itself an upper bound of its induced calibration error.
\end{example}

Additionally, Theorem \ref{th:ce_relations} motivates the usage of $\sqrt{\mathcal{U}_\text{BS} \left( f \right)}$.
Given a dataset $\left\{ \left(X_1, Y_1 \right), \dots, \left(X_n, Y_n \right) \right\}$ and a model $f$, we will estimate this quantity via $\sqrt{\mathcal{U}_{\text{BS}} \left( f \right)} \approx \sqrt{\frac{1}{n} \sum_{i=1}^n \left\| f \left( X_i \right) - Y^\prime_i \right\|^2}$.
In general, any unbiased estimator $\hat{\theta}$ becomes biased after a non-linear transformation $t$, since $\mathbb{E} \left[ t \left( \hat{\theta} \right) \right] \neq t \left( \mathbb{E} \left[ \hat{\theta} \right] \right)$.
But, if $t$ is continuous, our estimator is still asymptotically unbiased and consistent \citep{takeshi1985advanced}.\footnote{follows from Continuous Mapping Theorem and Theorem 3.2.6 of \citet{takeshi1985advanced}}
We will further investigate the empirical robustness w.r.t. data size in Section \ref{sec:exp} with $t$ as the square root and $\sqrt{\mathcal{U}_\text{BS} \left( f \right)}$ as the root calibration upper bound (RBS).\\

Furthermore, $\mathcal{U}_S$ has the following properties, which are helpful for the application of recalibration method optimization and selection.

\begin{proposition}
    Given injective functions $h, h' \; \colon \; \mathcal{P} \to \mathcal{P}$ we have 
    $$ \mathcal{U}_S \left( h \circ f \right) - \mathcal{U}_S \left( f \right) = \text{CE}_S \left( h \circ f \right) - \text{CE}_S \left( f \right) \quad \text{, }$$
    $$ \mathcal{U}_S \left( h \circ f \right) > \mathcal{U}_S \left( h' \circ f \right) \iff \text{CE}_S \left( h \circ f \right) > \text{CE}_S \left( h' \circ f \right)$$
    and (assuming $S$ is differentiable)
    $$\frac{\mathrm{d} \mathcal{U}_S \left( h \circ f \right)}{\mathrm{d} h} = \frac{\mathrm{d} \text{CE}_S \left( h \circ f \right)}{\mathrm{d} h}.$$
\label{prop:ub_grad}
\end{proposition}

This is a generalization of Proposition 4.2 presented in \citep{zhang2020mix}.
It tells us that we can reliably estimate the improvement of any injective recalibration method via the upper bound.
Furthermore, we get access to the calibration gradient and can compare different transformations.
At first, injectivity seems like a significant restriction.
But, we argue in the following that injectivity - rather than being accuracy-preserving - is a desired property of general recalibration methods.
For example, we can construct a recalibration method, which is calibrated and accuracy-preserving,
but only predicts a finite set of distinct values (see Appendix \ref{app:recal_trafos}).
Specifically, we would only predict two distinct values for any input in binary classification.
To exclude such naive solutions which substantially reduce model sharpness, we restrict ourselves to injective transformations of $\mathcal{P}_n \to \mathcal{P}_n$.
These do provably not impact the model sharpness and preserve, at least partly, the continuity of the output space.
Examples of injective transformations are TS, ETS, and DIAG.
These state-of-the-art methods show very competitive performances even when compared to non-injective recalibration methods \citep{zhang2020mix, rahimi2020intra}.
Further, Proposition \ref{prop:ub_grad} also holds when replacing $\mathcal{U}_S$ with the expected score $s_S$ without further conditions.
This is useful when $\mathcal{U}_S$ does not exist, but we still want to perform recalibration as in Section \ref{sec:exp} in the case of the DSS.

\section{Experiments}
\label{sec:exp}

\begin{figure*}[t]
\vskip 0.2in
\centering
    \begin{subfigure}{.5\textwidth}
    \centering
    \includegraphics[width=\columnwidth]{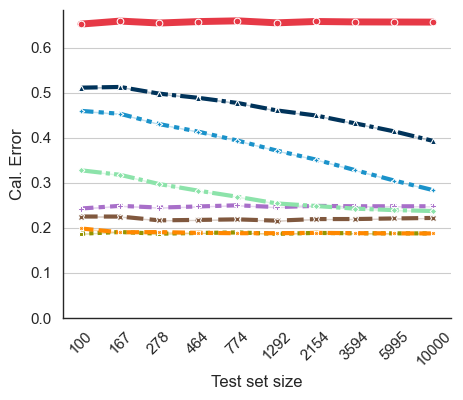}
    \end{subfigure}%
    \begin{subfigure}{.5\textwidth}
    \centering
    \includegraphics[width=\columnwidth]{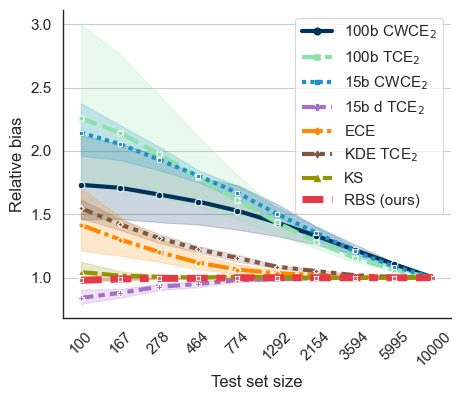}
    \end{subfigure}
\caption{
    \textbf{Left:} Different calibration error estimates versus the test set size of ResNet Wide 32 and CIFAR100. The red line corresponds to the square root of the Brier score which is an upper bound of the $\text{CE}_2$. The other errors are lower bounds.
    \textbf{Right:} Relative change versus data size with respect to error at full size. Averaging across a multitude of models shows a systematic trend. An unbiased estimator would give a flat line.
}
\label{fig:ce_bias}
\vskip -0.2in
\end{figure*}

In the following, we investigate the behavior of calibration error estimators in three settings.\\
First, we use varying test set sizes for the estimators and compare their values.
This will show how well the inequalities in Theorem \ref{th:ce_relations} hold in practical settings and how robust the estimators are.
Second, we explore what the estimated improvements of several recalibration methods are.
This is done after the recalibration methods are already optimized on a given validation set; we only vary the size of the test set and compute calibration errors on these test sets before and after recalibration.
In both settings, the straighter a line is, the more robust and, consequently, trustworthy is the estimator for practical applications.
Third, we investigate how our framework can be used to improve calibration for tasks beyond classification by performing probabilistic regression with subsequent recalibration.

In all experiments we evaluate the following estimators: CWCE$_2$ with 15 equal width bins ('15b CWCE$_2$'), CWCE$_2$ with 100 equal width bins ('100b CWCE$_2$'), 
ECE with 15 equal width bins ('ECE'), 
TCE$_2$ with 100 equal width bins ('100b TCE$_2$'), TCE$_2$ with 15 equal mass bins and debias term ('15b d TCE$_2$'), TCE$_2$ with kernel density estimation ('KDE TCE$_2$'), KS ('KS') and the root calibration upper bound $\sqrt{\mathcal{U}_{\text{BS}}}$ ('RBS (ours)').
The bin amounts are chosen based on past literature \citep{kumar2019verified, minderer2021revisiting}.
We also evaluate the KDE estimator of CE$_1$ ('KDE CE$_1$') with automatic bandwidth selection based on \citep{popordanoska2022} for CIFAR10.
The experiments are conducted across several model-dataset combinations, for which logit sets are openly accessible \citep{kull2019beyond, rahimi2020intra}.\footnote{\url{https://github.com/markus93/NN_calibration/} and \url{https://github.com/AmirooR/IntraOrderPreservingCalibration}}
This includes the models Wide ResNet 32 \citep{zagoruyko2016wide}, DenseNet 40, and DenseNet 161 \citep{huang2017densely} and the datasets CIFAR10, CIFAR100 \citep{krizhevsky2009learning},
and ImageNet \citep{deng2009imagenet}.
We did not conduct model training ourselves and refer to \citep{kull2019beyond} and \citep{rahimi2020intra} for further details.
We include TS, ETS, and DIAG as injective recalibration methods.
Further details and results on additional models and datasets are reported in the Appendix \ref{app:plots}.

\paragraph{Robustness of calibration errors to test set size}
We illustrate the estimated values of our introduced upper bound and the other errors, which are lower bounds of the unknown CE$_2$ on the left of Figure \ref{fig:ce_bias}.
On the right, we aggregate across several models to show the systematic drop-off according to Proposition \ref{prop:ece_estimator}.
The relative bias is computed by $Error(n) / Error (10000)$ and allows an aggregation of models with different calibration levels.
Included models are DenseNet 40, Wide ResNet 32, ResNet 110 SD, ResNet 110, and LeNet 5, all trained on CIFAR10.
All values represent the calibration of the given model without recalibration transformation.
Only our proposed upper bound and KS are stable, and Appendix \ref{app:plots} shows this holds across a wide range of different settings.
The theoretically highest lower bound (CWCE$_2$ with 100 bins) is also constantly the highest estimated lower bound, but it is sensitive to the test set size.
Results for further settings presented in Appendix \ref{app:plots} show similar results.

\paragraph{Quantifying recalibration improvement}
Next, we assessed how well all estimators were able to quantify the improvement in calibration error after applying different injective recalibration methods (Fig. \ref{fig:RC_errors}).
Only our proposed upper bound estimator RBS is again robust throughout all settings.
According to Proposition \ref{prop:ub_grad} and since RBS is asymptotically unbiased and consistent, it can be regarded as a reliable approximation of the real improvement of the presented recalibration methods.
For all other estimators, there is a general trend to estimate recalibration improvement higher for large test set sizes. 
In other settings, especially for small test set sizes, calibration improvement is underestimated to the extent that negative improvements (poorer calibration than before) are suggested. 
Results on other settings presented in Appendix \ref{app:plots} show similar results.
Taken together, these experiments demonstrate the unreliability of existing calibration estimators, in particular, when used to evaluate recalibration methods.  In contrast, our proposed upper bound estimator is stable across different settings.

\paragraph{Variance regression calibration}

We consider variance regression to demonstrate the usefulness of proper calibration errors outside the classification setting.
To this end, we predict sales prices with an uncertainty estimate in the UCI dataset \textit{Residential Building}, which consists of a high feature (107) to data instances (372) ratio \citep{rafiei2016novel}.
Our model of choice is a fully-connected mixture density network predicting mean and variance \citep{bishop1994mixture}.
Similar to classification, we are interested in recalibration of the predicted variance to adjust possible under- or overconfidence.
We use our proposed framework to derive a proper calibration error induced by the proper score DSS for recalibration.
Further, we compare DSS \citep{gneiting2014probabilistic} to squared KCE (SKCE) \citep{widmann2021calibration} and analyze the average predicted variance throughout model training.
We use Platt scaling ($x \mapsto wx + b$ with parameters $w,b \in \mathbb{R}$ \citep{Platt99probabilisticoutputs}) on the predicted variance in each training iteration to show how recalibration benefits uncertainty awareness.
We expect high uncertainty awareness at the start of the model training with a drop-off at later iterations.
As can be seen in Figure \ref{fig:recal_regr}, recalibration is able to adjust the uncertainty estimate of a model as desired.
Further, the DSS estimate, which captures predicted mean and variance correctness, directly communicates the improved variance fit.
Contrary, the SKCE estimate appears more erratic between iteration steps and seemingly ignores the changes in variance and the recalibration improvement.

\begin{figure*}[t]
\vskip 0.2in
\centering
    \begin{subfigure}{.33\textwidth}
    \centering
    \includegraphics[width=\columnwidth]{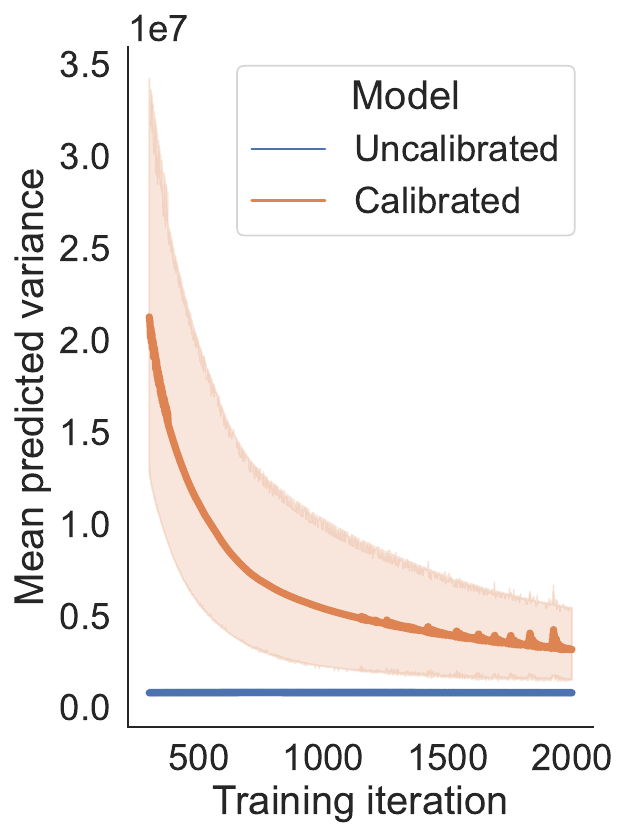}
    \end{subfigure}%
    \begin{subfigure}{.34\textwidth}
    \centering
    \includegraphics[width=\columnwidth]{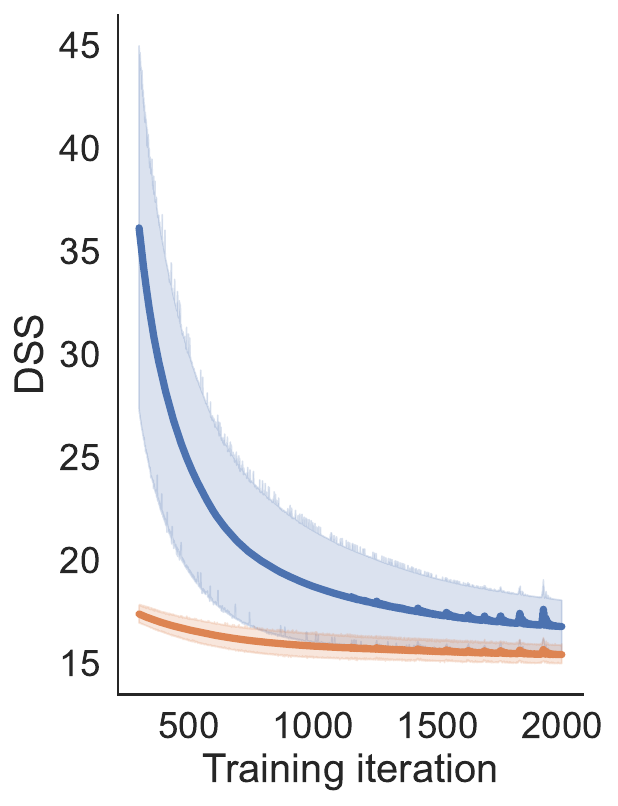}
    \end{subfigure}%
    \begin{subfigure}{.32\textwidth}
    \centering
    \includegraphics[width=\columnwidth]{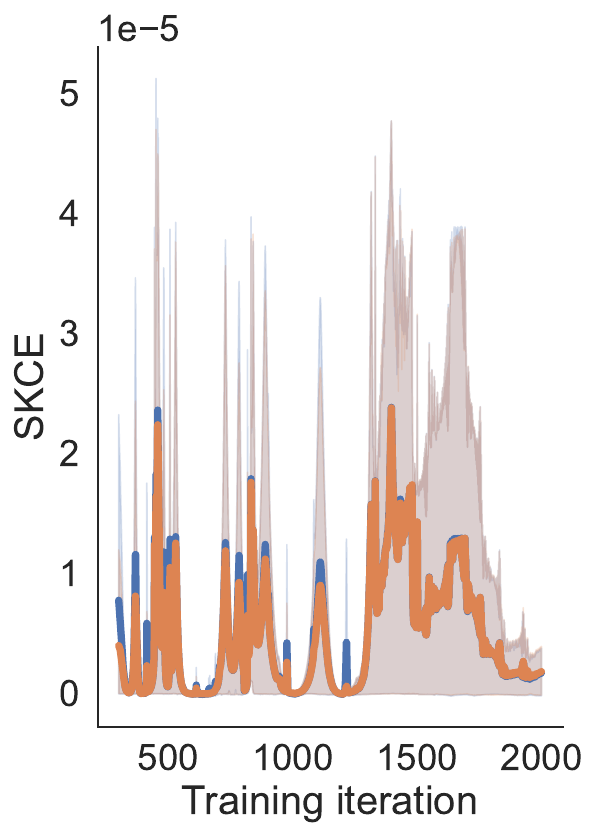}
    \end{subfigure}
\caption{
    \textbf{Left:} Average predicted variance throughout model training before and after recalibration.
    Initially, due to a bad fit, recalibration adjusts the variance accordingly for better communicated uncertainty.
    Once the model fit improves, the predicted variance requires less adjustment due to less uncertainty in each prediction.
    \textbf{Middle:} DSS communicates reasonably changes in the variance due to recalibration.
    \textbf{Right:} SKCE fails to capture the variance trend and behaves erratically.
}
\label{fig:recal_regr}
\vskip -0.2in
\end{figure*}

One might also be interested at how the predicted variance corresponds to the mean squared error throughout model training.
As we can see in Table \ref{tab:varregr1}, only after calibration using a proper calibration error, the average predicted variance (Avg Var) corresponds to the mean squared error.

\begin{table}
  \caption{Comparing the mean squared error (MSE) with the average predicted variance (Avg Var) before and after recalibration for various training iterations.
  Recalibration gives a better average match between prediction and real error.}
  \label{tab:varregr1}
  \centering
  \begin{tabular}{llllllll}
    \toprule
    Iteration              &   500 &   750 &   1000 &   1250 &   1500 &   1750 &   2000 \\
    \midrule
    MSE                    & 10.48 &  5.87 &   4.3  &   3.51 &   3.12 &   2.74 &   2.57 \\
    Avg Var (Calibrated)   & 11.04 &  6.89 &   5.4  &   4.55 &   4.11 &   3.63 &   3.32 \\
    Avg Var (Uncalibrated) &  0.83 &  0.83 &   0.83 &   0.83 &   0.83 &   0.82 &   0.82 \\
    \bottomrule
  \end{tabular}
\end{table}

Next, to assess how well the predicted variance corresponds to the instance-level error, we compute the ratio between the squared error of the predictive mean $\mu$ and the predicted variance $\sigma^2$ (SE Var Ratio) for each individual sample via $\frac{1}{n} \sum_{i=1}^n \frac{\left( \mu_i - y_i \right)^2}{\sigma^2_i}$.
An SE Var Ratio of '1' for a given instance means that the predictive uncertainty (variance) exactly matches the squared error, an SE Var Ratio of '10' means that the model is overconfident and the squared error is ten times as large as the predicted variance.
As we can see in Table \ref{tab:varregr2}, recalibration through our framework gives consistently conservative estimates on the squared error, whereas the uncalibrated uncertainties are highly overconfident (with errors more than ten times larger than the prediction).

\begin{table}
  \caption{Instance level ratio between the squared error and the predicted variance before and after recalibration for various training iterations.
  Recalibration improves the instance level prediction of the squared error.}

  \label{tab:varregr2}
  \centering
  \begin{tabular}{lllll}
    \toprule
    Iteration              &   500 &   1000 &   1500 &   2000 \\
    \midrule
    SE Var Ratio (Calibrated)     &  0.82 ± 2.17 &  0.79 ± 2.43 &   0.79 ± 2.56 &   0.79 ± 2.59 \\
    SE Var Ratio (Uncalibrated) & 11.33 ± 30.72 &   5.36 ± 15.09 &   4.07 ± 12.13 &   3.51 ± 11.22 \\
    \bottomrule
  \end{tabular}
\end{table}

We perform further variance regression experiments in Appendix \ref{app:plots}.

\section{Conclusion}

In this work, we address the problem of reliably quantifying the effect of recalibration on predictive uncertainty for classification and other probabilistic tasks.
This is critical for adjusting under- or overconfidence via recalibration.
To this end, we first provide a taxonomy of existing calibration errors.
We discover that most errors are lower bounds of a proper calibration error and fail to assess if a model is calibrated. This motivates our definition of \textit{proper calibration errors}, which provides a general class of errors for arbitrary probabilistic predictions.
Since proper calibration errors cannot be estimated in the general case, we introduce upper bounds, which directly measure the calibration change for injective transformations.
This allows us to reliably adjust model uncertainty via recalibration.
We demonstrate theoretically and empirically that the estimated calibration improvement can be highly misleading for commonly used estimators, including the ECE.
In stark contrast, our upper bound is robust to changes in data size and estimates robustly the improvement via injective recalibration.
We further show in additional experiments that our approach can be applied successfully to variance regression.



\medskip

{
\small

\bibliography{main}
\bibliographystyle{icml2022}
}

\section*{Checklist}

\begin{enumerate}

\item For all authors...
\begin{enumerate}
  \item Do the main claims made in the abstract and introduction accurately reflect the paper's contributions and scope?
    \answerYes{}
  \item Did you describe the limitations of your work?
    \answerYes{We mention required conditions for the proofs either directly or refer to Appendix \ref{app:proofs}.}
  \item Did you discuss any potential negative societal impacts of your work?
    \answerNo{We see positive societal impacts from improved uncertainty awareness via recalibration.}
  \item Have you read the ethics review guidelines and ensured that your paper conforms to them?
    \answerYes{}
\end{enumerate}

\item If you are including theoretical results...
\begin{enumerate}
  \item Did you state the full set of assumptions of all theoretical results?
    \answerYes{Major conditions are directly stated; minor conditions in Appendix \ref{app:proofs} as referred.}
        \item Did you include complete proofs of all theoretical results?
    \answerYes{In Appendix \ref{app:proofs}.}
\end{enumerate}

\item If you ran experiments...
\begin{enumerate}
  \item Did you include the code, data, and instructions needed to reproduce the main experimental results (either in the supplemental material or as a URL)?
    \answerYes{Full code is located in the supplementary material and further experimental details are in Appendix \ref{app:plots}.}
  \item Did you specify all the training details (e.g., data splits, hyperparameters, how they were chosen)?
    \answerYes{Relevant training details are located in Appendix \ref{app:plots}. A lot of models are not trained by ourselves. Instead, we loaded their results from publicly accessible URLs. We refer to each URL and the original work, where the training was conducted.}
    \item Did you report error bars (e.g., with respect to the random seed after running experiments multiple times)?
    \answerYes{We repeated experiments until the error bars are barely visible or not visible anymore. Only exception is Figure \ref{fig:recal_regr}, where aggregations hinder visibility due to high variances in training different seeds.}
        \item Did you include the total amount of compute and the type of resources used (e.g., type of GPUs, internal cluster, or cloud provider)?
    \answerYes{See Appendix \ref{app:plots}.}
\end{enumerate}

\item If you are using existing assets (e.g., code, data, models) or curating/releasing new assets...
\begin{enumerate}
  \item If your work uses existing assets, did you cite the creators?
    \answerYes{We used assets from \citep{kull2019beyond} and \citet{rahimi2020intra} located in https://github.com/markus93/NN\_calibration/ and https://github.com/AmirooR/IntraOrderPreservingCalibration .}
  \item Did you mention the license of the assets?
    \answerNo{Each lincense can be viewed in the respective github repository.}
  \item Did you include any new assets either in the supplemental material or as a URL?
    \answerYes{We include the code for reproducing the experiments and figures in the supplementary material.}
  \item Did you discuss whether and how consent was obtained from people whose data you're using/curating?
    \answerNo{The data is publicly accessible and we cited each respective source.}
  \item Did you discuss whether the data you are using/curating contains personally identifiable information or offensive content?
    \answerNo{The data is often abstract in nature (we used pretrained logits provided by a publicly accessible source) or downloaded the dataset from the UCI dataset database.}
\end{enumerate}

\item If you used crowdsourcing or conducted research with human subjects...
\begin{enumerate}
  \item Did you include the full text of instructions given to participants and screenshots, if applicable?
    \answerNA{No crowdsourcing or human subjects}
  \item Did you describe any potential participant risks, with links to Institutional Review Board (IRB) approvals, if applicable?
    \answerNA{No crowdsourcing or human subjects}
  \item Did you include the estimated hourly wage paid to participants and the total amount spent on participant compensation?
    \answerNA{No crowdsourcing or human subjects}
\end{enumerate}

\end{enumerate}

\newpage

\appendix

\section{Overview}
In this appendix we 
\begin{itemize}
    \item Introduce some notation in section \ref{app:notation} that we will use throughout the appendix.
    \item Give rigorous definitions of calibration errors omitted in the main paper in Section \ref{app:defs}
    \item Provide proofs for all claims that we make in the main text in Section \ref{app:proofs}.
    \item Provide details for specific recalibration transformation that illustrate the shortcomings of existing approaches (Section \ref{app:recal_trafos}).
    \item Give a detailed overview of proper U-scores that can be used to further generalize our proposed framework of proper calibration errors (Section \ref{app:u-scores}).
    \item Give more experimental details and report results from additional experiments (Section \ref{app:plots}).
\end{itemize}

\section{Notation}
\label{app:notation}

The following is implied throughout the appendix.
We will use 
\begin{itemize}
    \item The underlying probability space $\left( \Omega, \mathcal{F}, \mathbb{P} \right)$, $\mathcal{X}$ the feature space, and $\mathcal{Y}$ the target space.
    \item Random variables $X \colon \Omega \to \mathcal{X}$ and $Y \colon \Omega \to \mathcal{Y}$.
    \item $\mathbb{P}_{Y \mid X=x} \left( y \right) \coloneqq \frac{\mathbb{P} \left( \left\{ \omega \in \Omega \mid X \left( \omega \right) = x \land Y \left( \omega \right) = y \right\} \right)}{\mathbb{P} \left( \left\{ \omega \in \Omega \mid X \left( \omega \right) = x \right\} \right)}$ and $\mathbb{P}_{Y} \left( y \right) \coloneqq \mathbb{P} \left( \left\{ \omega \in \Omega \mid Y \left( \omega \right) = y \right\} \right)$ for $x \in \mathcal{X}$ and $y \in \mathcal{Y}$.
    \item $\mathbb{P}_{Y}, \mathbb{P}_{Y \mid X=x} \in \mathcal{P}_n$ with $\mathcal{P}_n = \left\{ p \in \left[0, 1 \right]^n \mid \sum_k p_k = 1 \right\}$, and $\mathcal{Y} = \left\{1, \dots, n \right\}$ for categorical $Y$ with $n \in \mathbb{N}$ classes.
    \item The index '$-k$' on a finite vector to denote the removal of index $k$.
    \item The random variable $C \colon \Omega \to \mathcal{Y}$ defined as $C \coloneqq \arg\max_k f_k(X)$ for $f \colon \mathcal{X} \to \mathcal{P}_n$. It can be regarded as the top-label prediction of $f$.
\end{itemize}

The notation regarding the (conditional) probability measures will be used for arbitrary random variables.

\section{Definitions}
\label{app:defs}

A systematic overview of the multitude of calibration errors proposed in the recent literature requires a common notation that can be used to harmonize definitions.
For the sake of clarity, we use formulations close to the notation introduced in \citet{kumar2019verified} and adjust the other errors accordingly, while retaining the notation of the original work whenever possible.

We follow \citet{kumar2019verified} and define top-label and class-wise calibration errors in expectation:

\begin{definition}
    The \textbf{top-label calibration error} of model $f \colon \mathcal{X} \to \mathcal{P}_n$ is defined as $$ \text{TCE}_p \left( f \right) = \left( \mathbb{E} \left[ \left\lvert f_C \left( X \right) - \mathbb{P} \left( Y = C \mid f_C \left( X \right) \right) \right\rvert^p \right] \right)^{\frac{1}{p}} $$
    with $C \coloneqq \arg\max_k f_k \left( X \right)$ and the \textbf{class-wise calibration error} is defined as $$ \text{CWCE}_p \left( f \right) = \left( \sum_{k \in \mathcal{Y}} \mathbb{E} \left[ \left\lvert f_k \left( X \right) - \mathbb{P} \left( Y = k \mid f_k \left( X \right) \right) \right\rvert^p \right] \right)^{\frac{1}{p}} $$
    for $1 \leq p \in \mathbb{R}$.
\label{def:cw_tce}
\end{definition}

Note that we removed the weighting factors from the original definition in \citet{kumar2019verified} for easier comparison with the other errors and a fixed upper limit (we will show that CWCE$_p \leq 2^{\frac{1}{p}}$).

\begin{definition}
    The \textbf{Kolmogorov-Smirnov calibration error} \citep{gupta2020calibration} of model $f \colon \mathcal{X} \to \mathcal{P}_n$ 
    is given by $$\text{KS} \left( f \right) = \mathbb{E} \left[ \text{KS} \left( f, C \right) \right],$$ where $C = \arg\max_k f_k \left( X \right)$ and $\text{KS} \left( f, k \right) = \max_{\sigma \in \left[0, 1 \right]} \left\lvert \int_{\left[0, \sigma\right]} z - \mathbb{P} \left( Y = k \mid f_k \left( X \right) = z \right) \mathrm{d} \mathbb{P}_{f_k \left(X\right)} \left( z \right) \right\rvert.$
\label{def:ks_error}
\end{definition}

\begin{definition}
    Given a reproducing kernel Hilbert space $\mathcal{H}$ with kernel $k \colon \left[0, 1\right] \times \left[0, 1\right] \to \mathbb{R}$ the \textbf{maximum mean calibration error} \citep{pmlr-v80-kumar18a} of model $f \colon \mathcal{X} \to \mathcal{P}_n$ is 
    \begin{equation*}
    \begin{split}
    & \text{MMCE} \left( f \right) = \\
    & \quad \quad \left\| \mathbb{E} \left[ \left( f_C \left( X \right) - \mathbb{P} \left( Y = C \mid f_C \left( X \right) \right) \right) k \left( f_C \left( X \right), . \right) \right] \right\|_{\mathcal{H}}.
    \end{split}
    \end{equation*}
\label{def:mmce}
\end{definition}

\begin{definition}
    Given a reproducing kernel Hilbert space $\mathcal{H}$ with kernel $k \colon \mathcal{P}_n \times \mathcal{P}_n \to \mathbb{R}^n \times \mathbb{R}^n$ 
    the \textbf{kernel calibration error} \citep{widmann2019calibration} of model $f \colon \mathcal{X} \to \mathcal{P}_n$ is 
    \begin{equation*}
    \begin{split}
    & \text{KCE} \left( f \right) = 
    \left\| \mathbb{E} \left[ \left( f \left( X \right) - \mathbb{P}_{Y \mid f \left( X \right)} \right) k \left( f \left( X \right), . \right) \right] \right\|_{\mathcal{H}}.
    \end{split}
    \end{equation*}
\label{def:kce}
\end{definition}

We also need the following score related definitions in the proofs.
These are simply a repetition from the main paper.

\begin{definition}
Given a proper score $S$ and $P, Q \in \mathcal{P}$, the expected score $s_S \colon \mathcal{P} \times \mathcal{P} \to \mathcal{R}$ is defined as $s_S \left(P, Q \right) = \mathbb{E}_{Y \sim Q} \left[ S \left(P, Y \right) \right] = \int_\mathcal{Y} S \left(P, y \right) \mathrm{d} Q \left( y \right)$.
\label{def:s_S}
\end{definition}

\begin{definition}
Given a proper score $S$ and $P, Q \in \mathcal{P}$, the associated \textbf{divergence} $d_S \colon \mathcal{P} \times \mathcal{P} \to \mathbb{R}_{\geq 0}$ is defined as
 $d_S \left(P, Q \right) = s_S \left(P, Q \right) - s_S \left(Q, Q \right)$
and the associated \textbf{generalized entropy} $g_S \colon \mathcal{P} \to \mathbb{R}$ as
$g_S \left( Q \right) = s_S \left(Q, Q \right)$.
\label{def:div_ent}
\end{definition}

\section{Proofs}
\label{app:proofs}

\subsection{Helpers}

The following will be of use in several proofs.

\begin{lemma}
    Assume that $S$ is a proper score for which CE$_S$ exists, then we have $$\text{CE}_S \left( f \right) = \mathbb{E} \left[ S \left( f \left( X \right), Y \right) \right] - \mathbb{E} \left[ g_S \left( \mathbb{P}_{Y \mid f \left( X \right)} \right) \right].$$
\label{le:cal_decomp}
\end{lemma}

\begin{proof}
\begin{equation}
\begin{split}
    \text{CE}_S \left( f \right) \overset{\text{def }\ref{def:pce}}&{=} \mathbb{E} \left[ d_S \left( f \left( X \right), \mathbb{P}_{Y \mid f \left( X \right)} \right) \right] \\
    \overset{\text{def }\ref{def:div_ent}}&{=} \mathbb{E} \left[ s_S \left( f \left( X \right), \mathbb{P}_{Y \mid f \left( X \right)} \right) - s_S \left( \mathbb{P}_{Y \mid f \left( X \right)}, \mathbb{P}_{Y \mid f \left( X \right)} \right) \right] \\
    \overset{\text{def }\ref{def:div_ent}}&{=} \mathbb{E} \left[ s_S \left( f \left( X \right), \mathbb{P}_{Y \mid f \left( X \right)} \right) \right] - \mathbb{E} \left[ g_S \left( \mathbb{P}_{Y \mid f \left( X \right)} \right) \right] \\
    & = \int s_S \left( z, \mathbb{P}_{Y \mid f \left( X \right) = z} \right) \mathrm{d} \mathbb{P}_{f \left( X \right)} \left(z \right) - \mathbb{E} \left[ g_S \left( \mathbb{P}_{Y \mid f \left( X \right)} \right) \right] \\
    \overset{\text{def }\ref{def:s_S}}&{=} \int \int S \left(z, y \right) \mathrm{d} \mathbb{P}_{Y \mid f \left( X \right) = z} \left(y \right) \mathrm{d} \mathbb{P}_{f \left( X \right)} \left(z \right) - \mathbb{E} \left[ g_S \left( \mathbb{P}_{Y \mid f \left( X \right)} \right) \right] \\
    & = \int S \left(z, y \right) \mathrm{d} \mathbb{P}_{Y, f \left( X \right)} \left(y, z \right) - \mathbb{E} \left[ g_S \left( \mathbb{P}_{Y \mid f \left( X \right)} \right) \right] \\
    & = \mathbb{E} \left[ S \left( f \left( X \right), Y \right) \right] - \mathbb{E} \left[ g_S \left( \mathbb{P}_{Y \mid f \left( X \right)} \right) \right] \\
\label{eq:cal_decomp}
\end{split}
\end{equation}
\end{proof}

\subsection{Theorem \ref{th:ce_relations}}
\label{th:app:ce_relations}

    Given a model $f \; \colon \; \mathcal{X} \to \mathcal{P}_n$ and the above defined errors, we have
    \begin{equation}
    \begin{split}
        & \text{BS} \left( f \right) = 0 \\
        \implies & \text{CE}_p \left( f \right) = 0 
        \iff \text{KCE} \left( f \right) = 0 
        \iff f \text{ is calibrated} \\
        \implies & 
        \begin{cases}
            \text{CWCE}_p \left( f \right) = 0 \\
            \text{TCE}_p \left( f \right) = 0 
            \iff \text{MMCE} \left( f \right) = 0 
            \iff \text{KS} \left( f \right) = 0 
            \implies \text{ECE} \left( f \right) = 0 
        \end{cases}
    \label{eq:app:implies_chain}
    \end{split}
    \end{equation}
    and
    \begin{equation}
    \begin{split}
        n^{\frac{1}{p} - \frac{1}{2}} \sqrt{2} & \geq n^{\frac{1}{p} - \frac{1}{2}} \sqrt{\text{BS} \left( f \right)} \\
        \overset{*}&{\geq} \text{CE}_p \left( f \right) \\
        & \geq 
        \begin{cases}
            \text{CWCE}_p \left( f \right) \\
            \text{TCE}_p \left( f \right) \geq \text{TCE}_1 \left( f \right) \geq \begin{cases}
            \text{KS} \left( f \right) \\
            \text{ECE} \left( f \right) \\
            c \cdot \text{MMCE} \left( f \right)
            \end{cases}
            \geq 0            
        \end{cases}
    \end{split}
    \label{eq:app:unequalities}
    \end{equation}
    
    for $1 \leq p \in \mathbb{R}$.
    * BS is only included for $p \leq 2$. We define $c = \sqrt{\max_r k \left( r, r\right)}$ as given in Theorem 3 of \citet{pmlr-v80-kumar18a}.

\begin{proof}

\textbf{Regarding }
BS$\left( f \right) = 0 \implies$ CE$_p \left( f \right) =0 \iff \text{KCE} \left( f \right) = 0 \iff f \text{ is calibrated}$:

\begin{equation}
\begin{split}
    \text{BS} \left( f \right) = 0
    & \iff \mathbb{P}_{Y \mid X } \overset{\text{a.s.}}{=} f \left( X \right) \\
    & \implies \mathbb{P}_{Y \mid f \left( X \right)} \overset{\text{a.s.}}{=} f \left( X \right) \\
    & \iff \begin{cases}
        \text{CE}_p \left( f \right) = 0 \\
        \text{KCE} \left( f \right) = 0 \\
        f \text{ is calibrated}
    \end{cases}
\label{BS==>CE0}
\end{split}
\end{equation}

The last equivalence follows from Definition \ref{def:strongly_cal} and \ref{def:lp_ce}, and according to \citet{widmann2019calibration}.
Since the equivalence in the last line holds for each, it follows CE$_p \left( f \right) =0 \iff \text{KCE} \left( f \right) = 0 \iff f \text{ is calibrated}$.
Example sketch for BS$\left( f \right) = 0 \centernot\impliedby$ CE$_p \left( f \right) =0$:
Set $f \left(.\right) = \mathbb{P}_Y$, then $f \left( X \right) = \mathbb{P}_Y = \mathbb{P}_{Y \mid \mathbb{P}_Y} = \mathbb{P}_{Y \mid f \left(X \right)}$, but BS$\left( f \right) > 0$.

\textbf{Regarding }
CE$_p \left( f \right) = 0 \implies$ CWCE$_p \left( f \right) =0$:

\begin{equation}
\begin{split}
    \text{CE}_p \left( f \right) = 0
    & \iff \mathbb{P}_{Y \mid f \left( X \right)} \overset{\text{a.s.}}{=} f \left( X \right) \\
    & \iff \mathbb{P} \left(Y = k \mid f \left( X \right) \right) \overset{\text{a.s.}}{=} f_k \left( X \right) \quad \forall k \\
    & \implies \mathbb{E}_{f_{-k} \left( X \right)} \left[ \mathbb{P} \left(Y = k \mid f \left( X \right) \right) \mid f_k \left(X \right) \right] \overset{\text{a.s.}}{=} \mathbb{E}_{f_{-k}} \left[ f_k \left( X \right) \mid f_k \left(X \right) \right] \quad \forall k \\
    & \iff \mathbb{P} \left(Y = k \mid f_k \left( X \right) \right) \overset{\text{a.s.}}{=} f_k \left( X \right) \quad \forall k \\
    & \iff \sum_{k \in \mathcal{Y}} \mathbb{E} \left[ \left( \mathbb{P} \left( Y = k \mid f_k \left( X \right) \right) - f_k \left( X \right) \right)^p \right] = 0 \\
    & \iff \text{CWCE}_p \left( f \right) = 0
\label{CE0==>CWCE0}
\end{split}
\end{equation}

An example for CE$_p \left( f \right) = 0 \centernot\impliedby$ CWCE$_p \left( f \right) =0$ is given in the proof of Proposition \ref{prop:neg_example} located in Appendix \ref{app:prop:neg_example}.

\textbf{Regarding }
CE$_p \left( f \right) = 0 \implies$ TCE$_p \left( f \right) = 0$:

\begin{equation}
\begin{split}
    \text{CE}_p \left( f \right) = 0
    & \iff \mathbb{P}_{Y \mid f \left( X \right)} \overset{\text{a.s.}}{=} f \left( X \right) \\
    & \iff \mathbb{P} \left(Y = k \mid f \left( X \right) \right) \overset{\text{a.s.}}{=} f_k \left( X \right) \quad \forall k \\
    & \implies \mathbb{E} \left[ \mathbb{P} \left(Y = k \mid f \left( X \right) \right) \mid f_C \left( X \right) \right] \overset{\text{a.s.}}{=} f_k \left( X \right) \quad \forall k \\
    & \iff \mathbb{P} \left(Y = k \mid f_C \left( X \right) \right) \overset{\text{a.s.}}{=} f_k \left( X \right) \quad \forall k \\
    & \implies \mathbb{P} \left(Y = C \mid f_C \left( X \right) \right) \overset{\text{a.s.}}{=} f_C \left( X \right) \\
    & \iff \mathbb{E} \left[ \left( \mathbb{P} \left( Y = C \mid f_C \left( X \right) \right) - f_C \left( X \right) \right)^p \right] = 0 \\
    & \iff \text{TCE}_p \left( f \right) = 0
\label{CE0==>TCE0}
\end{split}
\end{equation}

\textbf{Regarding }
CWCE$_p \left( f \right) =0 \centernot\implies$ TCE$_p \left( f \right) =0 $:

In the original version of this work, we claimed CWCE$_p \left( f \right) =0 \implies$ TCE$_p \left( f \right) =0$, which is a false statement as demonstrated by \citet{chen2024on} in their Appendix F.1 via an example.







\textbf{Regarding }
TCE$_p \left( f \right) =0 \iff \text{MMCE} \left( f \right) = 0 $:

See Theorem 1 in \citet{pmlr-v80-kumar18a} and ther appendix for the proof.
Note that their claim that MMCE is a proper score is contradictive to the definition of proper scores.
By definition, a proper score can be evaluated with only a single target observation, while the MMCE needs at least two target observations.

\textbf{Regarding }
TCE$_p \left( f \right) = 0 \iff \text{KS} \left( f \right) = 0 $:

\begin{equation}
\begin{split}
    & \text{TCE}_p \left( f \right) = 0 \\
    & \iff \mathbb{P} \left( Y = \arg\max_k f_k \left( X \right) \mid \max_l f_l \left( X \right) \right) \overset{\text{a.s.}}{=} \max_m f_m \left( X \right) \\
    & \iff \mathbb{P} \left( Y = C \mid f_C \left( X \right) \right) \overset{\text{a.s.}}{=} f_C \left( X \right) \\
    \overset{(i)}&{\iff} \int_{\sigma'} \mathbb{P} \left( Y = C \mid f_C \left( X \right) = z \right) \mathrm{d} \mathbb{P}_{f_C \left( X \right) \mid C} \left( z \right) \overset{\text{a.s.}}{=} \int_{\sigma'} z \mathrm{d} \mathbb{P}_{f_C \left( X \right) \mid C} \left( z \right), \quad \forall \sigma' \subset \left[0, 1 \right] \\
    & \iff \int_{\left[0, \sigma \right]} \mathbb{P} \left( Y = C \mid f_C \left( X \right) = z \right) \mathrm{d} \mathbb{P}_{f_C \left( X \right) \mid C} \left( z \right) \overset{\text{a.s.}}{=} \int_{\left[0, \sigma \right]} z \mathrm{d} \mathbb{P}_{f_C \left( X \right) \mid C} \left( z \right), \quad \forall \sigma \in \left[0, 1 \right] \\
    & \iff \mathbb{E} \left[ \max_{\sigma \in \left[0, 1 \right]} \left\lvert \int_{\left[0, \sigma\right]} z - \mathbb{P} \left( Y = C \mid f_C \left( X \right) = z \right) \mathrm{d} \mathbb{P}_{f_C \left(X\right) \mid C} \left( z \right) \right\rvert \right] = 0 \\
    & \iff \mathbb{E} \left[ \text{KS} \left( f, C \right) \right] = 0 \\
    & \iff \text{KS} \left( f\right) = 0 \\
\label{TCE0==>KS}
\end{split}
\end{equation}

(i) according to Theorem 4.22 of \citet{capinski2004measure}.

\textbf{Regarding }
TCE$_p \left( f \right) = 0 \implies \text{ECE} \left( f \right) = 0 $:

\begin{equation}
\begin{split}
    & \text{TCE}_p \left( f \right) = 0 \\
    & \iff \mathbb{P} \left( Y = C \mid f_C \left( X \right) \right) \overset{\text{a.s.}}{=} f_C \left( X \right) \\
    \overset{\text{(i)}}&{\implies} \forall i = 1, \dots, m \colon \; \mathbb{P} \left( Y = C \mid f_C \left( X \right) \in B_i \right) \overset{\text{a.s.}}{=} \mathbb{E} \left[ f_C \left( X \right) \mid f_C \left( X \right) \in B_i \right] \\
    \overset{\text{def }\ref{def:ece}}&{\iff} \text{ECE} \left( f \right) = 0
\label{TCE0==>ECE0}
\end{split}
\end{equation}

(i) with $B_i$ defined as in definition \ref{def:ece}; follows since $\mathbb{P} \left( Y = C \mid f_C \left( X \right) \in B_i \right) = \int_{B_i} \mathbb{P} \left( Y = C \mid f_C \left( X \right) = z \right) \mathrm{d} \mathbb{P}_{f_C \left( X \right)} \left( z \right) \overset{\text{a.s.}}{=} \int_{B_i} f_C \left( X \right) \mathrm{d} \mathbb{P}_{f_C \left( X \right)} \left( z \right) = \mathbb{E} \left[ f_C \left( X \right) \mid f_C \left( X \right) \in B_i \right]$.

An intuition of why TCE$_1 \left( f \right) = 0 \centernot\impliedby \text{ECE} \left( f \right) = 0$ is given in example 3.2 of \citet{kumar2019verified}.

\textbf{Regarding }
$2 \geq \text{BS} \left( f \right) \geq \left( \text{CE}_2 \left( f \right) \right)^2$:

\begin{equation}
\begin{split}
    & 2 = \left\| e_1 - e_2 \right\|^2_2 \\
    & \geq \mathbb{E} \left[\left\| f \left( X \right) - e_Y \right\|^2_2 \right] \\
    \overset{\text{def }\ref{def:bs}}&{=} \text{BS} \left( f \right) \\
    \overset{\text{(i)}}&{\geq} \text{BS} \left( f \right) - \mathbb{E} \left[ g_\text{BS} \left( \mathbb{P}_{Y \mid f \left(X \right)} \right) \right] \\
    \overset{\text{le }\ref{le:cal_decomp}}&{=} \text{CE}_{\text{BS}} \left( f \right) \\
    \overset{\text{(ii)}}&{=} \left( \text{CE}_2 \left( f \right) \right)^2 \\
\label{eq:BS>CE22}
\end{split}
\end{equation}
(i) $g_\text{BS}$ non-negative, follows from definition \ref{def:div_ent}. \\
(ii) compare definition \ref{def:lp_ce} with the squared calibration term in \citep{ANewVectorPartitionoftheProbabilityScore}. \\

\textbf{Regarding }
$n^{\frac{1}{p} - \frac{1}{q}} \text{CE}_q \left( f \right) \geq \text{CE}_p \left( f \right)$ for $0 < p \leq q < \infty$:

We use $\Delta \coloneqq f \left( X \right) - \mathbb{P}_{Y \mid f \left( X \right)}$ for shorter equations.
Further, we use the $p$-norm inequality $n^{\frac{1}{p} - \frac{1}{q}} \left\| x \right\|_q \geq \left\| x \right\|_p$ for a vector $x \in \mathbb{R}^n$ and the $L^p$ space inequality $\left( \mathbb{E} \left[ \left\lvert X \right\rvert^q \right] \right)^{\frac{1}{q}} \geq \left( \mathbb{E} \left[ \left\lvert X \right\rvert^p \right] \right)^{\frac{1}{p}}$ for $X \in L^q \subset L^p$ \citep{capinski2004measure}.

\begin{equation}
\begin{split}
    & \text{CE}_p \left( f \right) \\
    \overset{\text{def }\ref{def:lp_ce}}{=} & \left( \mathbb{E} \left[ \left\| \Delta \right\|^p_p \right] \right)^{\frac{1}{p}} \\
    \leq & \left( \mathbb{E} \left[ \left( n^{\frac{1}{p} - \frac{1}{q}} \left\| \Delta \right\|_q \right)^p \right] \right)^{\frac{1}{p}} \\
    = & n^{\frac{1}{p} - \frac{1}{q}} \left( \mathbb{E} \left[ \left\| \Delta \right\|_q^p \right] \right)^{\frac{1}{p}} \\
    \leq & n^{\frac{1}{p} - \frac{1}{q}} \left( \mathbb{E} \left[ \left\| \Delta \right\|_q^q \right] \right)^{\frac{1}{q}} \\
    \overset{\text{def }\ref{def:lp_ce}}{=} & n^{\frac{1}{p} - \frac{1}{q}} \text{CE}_q \left( f \right) \\
\label{eq:CEq<CEp}
\end{split}
\end{equation}

Note that this result is a direct contradiction to Theorem 1 of \citep{wenger2020non} since $n^{\frac{1}{p} - \frac{1}{q}} > 1$.

Further, the name '$L^p$ calibration error' is unambiguous for canonical calibration since the following holds. 
Let $\left\lVert . \right\rVert_{\mathbb{R}^n,p}$ be the vector $p$-norm and $\left\lVert . \right\rVert_{L^{p}}$ the norm of the $L^p$ space.
Then we have
\begin{equation}
\text{CE}_p \left( f \right) = \left\lVert \left\lVert \Delta \right\rVert_{\mathbb{R}^n,p} \right\rVert_{L^{p}} = \left\lVert \left(\left\lVert \Delta_1 \right\rVert_{L^{p}}, \dots, \left\lVert \Delta_n \right\rVert_{L^{p}} \right)^\intercal \right\rVert_{\mathbb{R}^n,p}.
\end{equation}
Thus, there is no ambiguity if we first compute the vector norm or the $L^p$ norm and there cannot be another $L^p$ calibration error with a different norm order.

\textbf{Regarding }
$n^{\frac{1}{p} - \frac{1}{2}} \sqrt{\text{BS} \left( f \right) } \geq \text{CE}_p \left( f \right)$ for $0 < p \leq 2$:

Combining equations \eqref{eq:BS>CE22} and \eqref{eq:CEq<CEp} (with $q=2$) gives the result.

\textbf{Regarding }
$\text{CE}_p \left( f \right) \geq \text{CWCE}_p \left( f \right)$:

In the following, we will use Tonelli's theorem to split the expectation into two and the Jensen's inequality for the convex function $\lvert.\rvert^p$.

\begin{equation}
\begin{split}
    \left( \text{CE}_p \left( f \right) \right)^p & = \mathbb{E} \left[ \left\| f \left( X \right) - \mathbb{P}_{Y \mid f \left( X \right)} \right\|^p_p \right] \\
    & = \sum_{k \in \mathcal{Y}} \mathbb{E} \left[ \left\lvert f_k \left( X \right) - \mathbb{P} \left( Y=k \mid f \left( X \right) \right) \right\rvert^p \right] \\
    \overset{\text{Tonelli}}&{=} \sum_{k \in \mathcal{Y}} \mathbb{E}_{f_k \left( X \right)} \left[ \mathbb{E}_{f_{-k} \left( X \right)} \left[ \left\lvert f_k \left( X \right) - \mathbb{P} \left( Y=k \mid f \left( X \right) \right) \right\rvert^p \mid f_k \left( X \right) \right] \right] \\
    \overset{\text{Jensen}}&{\geq} \sum_{k \in \mathcal{Y}} \mathbb{E}_{f_k \left( X \right)} \left[ \left\lvert \mathbb{E}_{f_{-k} \left( X \right)} \left[ f_k \left( X \right) - \mathbb{P} \left( Y=k \mid f \left( X \right) \right) \mid f_k \left( X \right) \right] \right\rvert^p \right] \\
    & = \sum_{k \in \mathcal{Y}} \mathbb{E}_{f_k \left( X \right)} \left[ \left\lvert f_k \left( X \right) - \mathbb{P} \left( Y=k \mid f_k \left( X \right) \right) \right\rvert^p \right] \\
    \overset{\text{def }\ref{def:cw_tce}}&{=} \left( \text{CWCE}_p \left( f \right) \right)^p
\end{split}
\end{equation}

\textbf{Regarding }
$\text{CE}_p \left( f \right) \geq \text{TCE}_p \left( f \right)$:

We will use $F \coloneqq f \left( X \right)$ for shorter notation.
Further, for a vector $v \in \mathbb{R}^n$ of length $n \in \mathbb{N}$ and an index integer $i \in \left\{1, \dots n \right\}$ define $(i)_v$ as the index of the $i$-th largest element in $v$.
From these definitions follows that $(n)_F = (n)_{f \left( X \right)} = \arg\max_{i=1 \dots n} f_i \left( X \right) = C$ with $n = \lvert \mathcal{Y} \rvert$.

\begin{equation}
\begin{split}
    \left( \text{CE}_p \left( f \right) \right)^p & = \mathbb{E} \left[ \left\| f \left( X \right) - \mathbb{P}_{Y \mid f \left( X \right)} \right\|^p_p \right] \\
    & = \mathbb{E} \left[ \left\| F - \mathbb{P}_{Y \mid F} \right\|^p_p \right] \\
    & = \mathbb{E} \left[ \sum_{k \in \mathcal{Y}} \left\lvert F_k - \mathbb{P} \left( Y=k \mid F \right) \right\rvert^p \right] \\
    \overset{\text{(i)}}&{=} \mathbb{E} \left[ \sum_{k \in \mathcal{Y}} \left\lvert F_{(k)_F} - \mathbb{P} \left( Y=(k)_F \mid F \right) \right\rvert^p \right] \\
    \overset{\text{}}&{\geq} \mathbb{E} \left[ \left\lvert F_{(n)_F} - \mathbb{P} \left( Y=(n)_F \mid F \right) \right\rvert^p \right] \\
    \overset{\text{}}&{=} \mathbb{E} \left[ \left\lvert F_C - \mathbb{P} \left( Y=C \mid F \right) \right\rvert^p \right] \\
    \overset{\text{Tonelli}}&{=} \mathbb{E}_{F_C, C} \left[ \mathbb{E}_{F_{-C}} \left[ \left\lvert F_C - \mathbb{P} \left( Y=C \mid F \right) \right\rvert^p \mid F_C, C \right] \right] \\
    \overset{\text{Jensen}}&{\geq} \mathbb{E}_{F_C, C} \left[ \left\lvert \mathbb{E}_{F_{-C}} \left[ F_C - \mathbb{P} \left( Y=C \mid F \right) \mid F_C, C \right] \right\rvert^p \right] \\
    & = \mathbb{E}_{F_C, C} \left[ \left\lvert F_C - \mathbb{P} \left( Y=C \mid F_C \right) \right\rvert^p \right] \\
    & = \mathbb{E} \left[ \left\lvert F_C - \mathbb{P} \left( Y=C \mid F_C \right) \right\rvert^p \right] \\
    & = \mathbb{E} \left[ \left\lvert f_C \left( X \right) - \mathbb{P} \left( Y=C \mid f_C \left( X \right) \right) \right\rvert^p \right] \\
    \overset{\text{def }\ref{def:cw_tce}}&{=} \left( \text{TCE}_p \left( f \right) \right)^p
\end{split}
\end{equation}

(i) we re-order the sum according to the values of $F$

\textbf{Regarding }
$\text{TCE}_p \left( f \right) \geq \text{TCE}_1 \left( f \right)$:

Let $p \geq q \geq 1$.
This makes $\left( . \right)^{\frac{p}{q}}$ a convex function for positive arguments.
We will show the more general $\text{TCE}_p \left( f \right) \geq \text{TCE}_q \left( f \right)$.
From this directly follows $ \text{TCE}_p \left( f \right) \geq \text{TCE}_1 \left( f \right)$.

\begin{equation}
\begin{split}
    & \text{TCE}_p \left( f \right) \\
    & = \left( \mathbb{E} \left[ \left\lvert f_C \left( X \right) - \mathbb{P} \left( Y=C \mid f_C \left( X \right) \right) \right\rvert^p \right] \right)^{\frac{1}{p}} \\
    & = \left( \mathbb{E} \left[ \left\lvert f_C \left( X \right) - \mathbb{P} \left( Y=C \mid f_C \left( X \right) \right) \right\rvert^{q\frac{p}{q}} \right] \right)^{\frac{1}{p}} \\
    \overset{\text{Jensen}}&{\geq} \left( \mathbb{E} \left[ \left\lvert f_C \left( X \right) - \mathbb{P} \left( Y=C \mid f_C \left( X \right) \right) \right\rvert^q \right] \right)^{\frac{1}{p}\frac{p}{q}} \\
    & = \left( \mathbb{E} \left[ \left\lvert f_C \left( X \right) - \mathbb{P} \left( Y=C \mid f_C \left( X \right) \right) \right\rvert^q \right] \right)^{\frac{1}{q}} \\
    & = \text{TCE}_q \left( f \right) \\
\end{split}
\end{equation}

\textbf{Regarding }
$\text{TCE}_1 \left( f \right) \geq \text{KS} \left( f \right)$:

We will show the more general $\text{TCE}_p \left( f \right) \geq \text{KS} \left( f \right)$, from which $\text{TCE}_1 \left( f \right) \geq \text{KS} \left( f \right)$ follows.

We will make use of the indicator function for a set $A$ defined as $\mathbbm{1}_A \left( a \right) = \begin{cases}
1, & a \in A\\
0, & \text{ else.}
\end{cases}$

\begin{equation}
\begin{split}
    \left( \text{TCE}_p \left( f \right) \right)^p & = \mathbb{E} \left[ \left\lvert f_C \left( X \right) - \mathbb{P} \left( Y=C \mid f_C \left( X \right) \right) \right\rvert^p \right] \\
    \overset{\text{Tonelli}}&{=} \mathbb{E}_C \left[ \mathbb{E}_{f_C \left( X \right)} \left[ \left\lvert f_C \left( X \right) - \mathbb{P} \left( Y=C \mid f_C \left( X \right) \right) \right\rvert^p \mid C \right] \right] \\
    & = \mathbb{E}_C \left[ \mathbb{E}_{f_C \left( X \right)} \left[ \mathbbm{1}_{\left[0, 1 \right]} \left( f_C \left( X \right) \right) \left\lvert f_C \left( X \right) - \mathbb{P} \left( Y=C \mid f_C \left( X \right) \right) \right\rvert^p \mid C \right] \right] \\
    \overset{\text{(i)}}&{=} \mathbb{E}_C \left[ \max_{\sigma \in \left[ 0, 1 \right]} \mathbb{E}_{f_C \left( X \right)} \left[ \mathbbm{1}_{\left[0, \sigma \right]} \left( f_C \left( X \right) \right) \left\lvert f_C \left( X \right) - \mathbb{P} \left( Y=C \mid f_C \left( X \right) \right) \right\rvert^p \mid C \right] \right] \\
    & = \mathbb{E}_C \left[ \max_{\sigma \in \left[ 0, 1 \right]} \mathbb{E}_{f_C \left( X \right)} \left[ \left\lvert \mathbbm{1}_{\left[0, \sigma \right]} \left( f_C \left( X \right) \right) \left(f_C \left( X \right) - \mathbb{P} \left( Y=C \mid f_C \left( X \right) \right) \right) \right\rvert^p \mid C \right] \right] \\
    \overset{\text{Jensen}}&{\geq} \mathbb{E}_C \left[ \max_{\sigma \in \left[ 0, 1 \right]} \left\lvert \mathbb{E}_{f_C \left( X \right)} \left[ \mathbbm{1}_{\left[0, \sigma \right]} \left( f_C \left( X \right) \right) \left(f_C \left( X \right) - \mathbb{P} \left( Y=C \mid f_C \left( X \right) \right) \right) \mid C \right] \right\rvert^p \right] \\
    & = \mathbb{E}_C \left[ \max_{\sigma \in \left[ 0, 1 \right]} \left\lvert \int_{\left[0, 1 \right]} \mathbbm{1}_{\left[0, \sigma \right]} \left( z \right) \left(z - \mathbb{P} \left( Y=C \mid f_C \left( X \right) = z \right) \right) \mathrm{d} \mathbb{P}_{f_C \left( X \right) \mid C} \left( z \right) \right\rvert^p \right] \\
    & = \mathbb{E}_C \left[ \max_{\sigma \in \left[ 0, 1 \right]} \left\lvert \int_{\left[0, \sigma \right]} z - \mathbb{P} \left( Y=C \mid f_C \left( X \right) = z \right) \mathrm{d} \mathbb{P}_{f_C \left( X \right) \mid C} \left( z \right) \right\rvert^p \right] \\
    \overset{\text{Jensen}}&{\geq} \left( \mathbb{E}_C \left[ \max_{\sigma \in \left[ 0, 1 \right]} \left\lvert \int_{\left[0, \sigma \right]} z - \mathbb{P} \left( Y=C \mid f_C \left( X \right) = z \right) \mathrm{d} \mathbb{P}_{f_C \left( X \right) \mid C} \left( z \right) \right\rvert \right] \right)^p \\
    \overset{\text{def }\ref{def:ks_error}}&{=} \left( \mathbb{E}_C \left[ \text{KS} \left(f, C \right) \right] \right)^p \\
    \overset{\text{def }\ref{def:ks_error}}&{=} \left( \text{KS} \left(f \right) \right)^p
\end{split}
\end{equation}

(i) $\sigma \geq \sigma' \implies \mathbbm{1}_{\left[0, \sigma \right]} \left( f_C \left( X \right) \right) \left\lvert f_C \left( X \right) - \mathbb{P} \left( Y=C \mid f_C \left( X \right) \right) \right\rvert^p \geq \mathbbm{1}_{\left[0, \sigma' \right]} \left( f_C \left( X \right) \right) \left\lvert f_C \left( X \right) - \mathbb{P} \left( Y=C \mid f_C \left( X \right) \right) \right\rvert^p \geq 0$.

\textbf{Regarding }
$\text{TCE}_1 \left( f \right) \geq c \cdot \text{MMCE} \left( f \right)$:

This is given in the proof of Theorem 3 of \citet{pmlr-v80-kumar18a}.
Note that \citet{pmlr-v80-kumar18a} used ECE in their theorem, but their proof is actually given for TCE$_1$.
Since ECE$\left( f \right) = 0 \centernot\implies \text{MMCE} \left( f \right) = 0$, we have $\text{ECE} \left( f \right) \centernot\geq c \cdot \text{MMCE} \left( f \right)$.

\textbf{Regarding }
$\text{TCE}_1 \left( f \right) \geq \text{ECE} \left( f \right)$:

A similar statement for binary models is given in Proposition 3.3 of \citet{kumar2019verified} or for general models in Theorem 2 of \citet{vaicenavicius2019evaluating}.
Since our formulations differ, we provide an independent proof.

We will write $B_i \coloneqq \left( \frac{i - 1}{m}, \frac{i}{m} \right]$.
Let $\mathcal{B} \coloneqq \sigma \left( \left\{ B_1, \dots, B_m \right\} \right)$ be the $\sigma$-algebra generated by the binning scheme of size $m \in \mathbb{N}$ used for the ECE.

\begin{equation}
\begin{split}
    \text{TCE}_1 \left( f \right) & = \mathbb{E} \left[ \left\lvert f_C \left( X \right) - \mathbb{P} \left( Y=C \mid f_C \left( X \right) \right) \right\rvert \right] \\
    & = \mathbb{E} \left[ \mathbb{E} \left[ \left\lvert f_C \left( X \right) - \mathbb{P} \left( Y=C \mid f_C \left( X \right) \right) \right\rvert \mid \mathcal{B} \right] \right] \\
    \overset{\text{(i)}}&{\geq} \mathbb{E} \left[ \left\lvert \mathbb{E} \left[ f_C \left( X \right) - \mathbb{P} \left( Y=C \mid f_C \left( X \right) \right) \mid \mathcal{B} \right] \right\rvert \right] \\ 
    & = \mathbb{E} \left[ \left\lvert \mathbb{E} \left[ f_C \left( X \right) \mid \mathcal{B} \right] - \mathbb{P} \left( Y=C \mid \mathbb{E} \left[ f_C \left( X \right) \mid \mathcal{B} \right] \right) \right\rvert \right] \\
    & = \sum_{i=1}^m \mathbb{P} \left( f \left( X \right) \in B_i \right) \cdot \\
    & \quad \quad \left\lvert \mathbb{E} \left[ f_C \left( X \right) \mid f \left( X \right) \in B_i  \right] - \mathbb{P} \left( Y=C \mid f \left( X \right) \in B_i \right) \right\rvert \\
    \overset{\text{def }\ref{def:ece}}&{=} \text{ECE} \left( f \right) \\
\end{split}
\end{equation}

(i) We use conditional Jensen's inequality \citep{capinski2004measure}.

\end{proof}

\subsection{Proposition \ref{prop:neg_example}}
\label{app:prop:neg_example}

For all $1 > \alpha > 0.5$ and surjective $f \colon \mathcal{X} \to \mathcal{P}_n$ there exists a joint distribution $\mathbb{P}_{X,Y}$ such that for all $E \in \left\{ \text{MMCE}, \text{KS}, \text{ECE}, \text{TCE}_p, \text{CWCE}_p \mid 1 \leq p \in \mathbb{R} \right\} \colon$
\begin{equation*}
\begin{split}
    E \left( f \right) = 0 \quad \land \quad \text{CE}_2 \left( f \right) = \sqrt{1 - \frac{1}{n-1}}(1 - \alpha ).
\end{split}
\end{equation*}

\begin{proof}
Based on Theorem \ref{th:ce_relations} we only have to find a $\mathbb{P}_{X,Y}$ such that $\operatorname{CWCE}_p \left( f \right) = 0$ and $\operatorname{TCE}_p \left( f \right) = 0$ while $\operatorname{CE}_p \left( f \right) \neq 0$. \\
For $\alpha \in (0.5, 1)$ and $\beta := \frac{1 - \alpha}{n-1}$ with $3 \leq n \in \mathbb{N}$ define the $n$-dimensional probability vectors

\begin{equation}
\begin{split}
    p_1 & = 
    \begin{pmatrix}
        \alpha & \beta & \beta & \cdots  & \beta & \beta
    \end{pmatrix}^\intercal \\
    p_2 & = 
    \begin{pmatrix}
        \beta & \alpha & \beta & \cdots  & \beta & \beta
    \end{pmatrix}^\intercal \\
    p_3 & = 
    \begin{pmatrix}
        \beta & \beta & \alpha & \cdots  & \beta & \beta
    \end{pmatrix}^\intercal \\
    & \vdots \\
    p_{n-1} & = 
    \begin{pmatrix}
        \beta & \beta & \beta & \cdots & \alpha & \beta
    \end{pmatrix}^\intercal \\
    p_n & = 
    \begin{pmatrix}
        \beta & \beta & \beta & \cdots & \beta & \alpha
    \end{pmatrix}^\intercal    
\end{split}
\end{equation}

Since $f$ is surjective, we can choose $\mathbb{P}_{X,Y}$ such that 
\begin{equation}
    \forall k \in \{1 \dots n \} \colon \quad \mathbb{P} \left( f \left( X \right) = p_k \right) = \frac{1}{n}
\end{equation}

and with $\gamma \coloneqq 1 - \alpha$

\begin{equation}
\begin{split}
    \mathbb{P}_{Y \mid f \left( X \right) = p_1} & = 
    \begin{pmatrix}
        \alpha & 0 & 0 & \cdots & 0 & 0 & \gamma
    \end{pmatrix}^\intercal \\
    \mathbb{P}_{Y \mid f \left( X \right) = p_2} & = 
    \begin{pmatrix}
        \gamma & \alpha & 0 & \cdots & 0 & 0 & 0
    \end{pmatrix}^\intercal \\
    \mathbb{P}_{Y \mid f \left( X \right) = p_3} & = 
    \begin{pmatrix}
        0 & \gamma & \alpha & \cdots & 0 & 0 & 0
    \end{pmatrix}^\intercal \\
    & \vdots \\
    \mathbb{P}_{Y \mid f \left( X \right) = p_{n-1}} & = 
    \begin{pmatrix}
        0 & 0 & 0 & \cdots & \gamma & \alpha & 0
    \end{pmatrix}^\intercal \\
    \mathbb{P}_{Y \mid f \left( X \right) = p_n} & = 
    \begin{pmatrix}
        0 & 0 & 0 & \cdots & 0 & \gamma & \alpha
    \end{pmatrix}^\intercal.
\end{split}
\end{equation}

Note that $\left[\mathbb{P}_{Y \mid f \left( X \right) = p} \right]_i = \mathbb{P} \left( Y = i \mid f \left( X \right) = p \right)$ for any $i=1 \dots n$ and $p \in \Delta^n$.

Further, it follows that $\mathbb{P} \left( f_C \left( X \right) = \alpha \right) = 1$ and

\begin{equation}
\begin{split}
\mathbb{P} \left( Y = C \mid f_C \left( X \right) = \alpha \right) & = \mathbb{P} \left( Y = C \right) \\
& = \mathbb{E}_{f\left(X\right)} \left[ \mathbb{P} \left( Y = C \mid f \left( X \right) \right) \right] \\
& = \mathbb{E}_{f\left(X\right)} \left[ \mathbb{P} \left( Y = \arg\max_i f_i \left( X \right) \mid f \left( X \right) \right) \right] \\
& = \frac{1}{n} \sum_j \mathbb{P} \left( Y = \arg\max_i \left[p_j\right]_i \mid f \left( X \right) = p_j \right) \\
& = \frac{1}{n} \sum_j \mathbb{P} \left( Y = j \mid f \left( X \right) = p_j \right) \\
& = \frac{1}{n} \sum_j \alpha \\
& = \alpha,
\end{split} 
\end{equation}
which gives $\operatorname{TCE}_p \left( f \right) = 0$.

For the CWCE case, we also have \begin{equation}
\begin{split}
\mathbb{P} \left( Y = k \mid f_k \left( X \right) = \alpha \right) & = \mathbb{P} \left( Y = k \mid f \left( X \right) = p_k \right) \\
& = \alpha,
\end{split} 
\end{equation}
and
\begin{equation}
\begin{split}
\mathbb{P} \left( Y = k \mid f_k \left( X \right) = \beta \right) & = \mathbb{P} \left( Y = k \mid f_k \left( X \right) \neq \alpha \right) \\
& = \mathbb{P} \left( Y = k \mid f \left( X \right) \neq p_k \right) \\
& = \mathbb{E}_{f \left( X \right) \mid f \left( X \right) \neq p_k} \left[ \mathbb{P} \left( Y = k \mid f \left( X \right) \right) \right] \\
& = \frac{1}{n-1} \sum_{i \neq k}\mathbb{P} \left( Y = k \mid f \left( X \right) = p_i \right) \\
& = \frac{1}{n-1} \left( \gamma + \sum_{i=1}^{n-2} 0 \right) \\
& = \beta.
\end{split} 
\end{equation}

We also have $\mathbb{P} \left( f_k \left( X \right) = \alpha \right) = \mathbb{P} \left( f \left( X \right) = p_k \right) = \frac{1}{n}$ and $\mathbb{P} \left( f_k \left( X \right) = \beta \right) = \mathbb{P} \left( f_k \left( X \right) \neq \alpha \right) = \mathbb{P} \left( f \left( X \right) \neq p_k \right) = \frac{n-1}{n}$,
from which follows altogether that $\operatorname{CWCE}_p \left( f \right) = 0$.

Finally, we also have
\begin{equation}
\begin{split}
    \operatorname{CE}_2 \left( f \right)
    & = \sqrt{ \mathbb{E} \left[ \| f \left( X \right) - \mathbb{P}_{Y \mid f \left( X \right)}\|^2_2 \right]} \\
    & = \sqrt{ \frac{1}{n} \sum_{i=1}^n \| p_i - \mathbb{P}_{Y \mid f \left( X \right)=p_i}\|^2_2} \\
    & = \sqrt{ \frac{1}{n} \sum_{i=1}^n \left( \left( \alpha - \alpha \right)^2 + \left( \beta - \gamma \right)^2 + (n-2)\left( \beta - 0 \right)^2 \right)} \\
    & = \sqrt{\left( \beta - \gamma \right)^2 + (n-2)\beta^2} \\
    & = \sqrt{(n-1)\beta^2 - 2\beta\gamma + \gamma^2} \\
    & = \sqrt{(n-1)\frac{(1-\alpha)^2}{(n-1)^2} - 2\frac{1-\alpha}{n-1}(1-\alpha) + (1-\alpha)^2} \\
    & = \sqrt{\frac{(1-\alpha)^2}{n-1} - 2\frac{(1-\alpha)^2}{n-1} + \frac{(n-1)(1-\alpha)^2}{n-1}} \\
    & = \sqrt{\frac{(n-1)(1-\alpha)^2 - (1-\alpha)^2}{n-1}} \\
    & = \sqrt{\frac{(n-2)(1-\alpha)^2}{n-1}} \\
    & = \sqrt{\frac{n-2}{n-1}} (1-\alpha) \\
    & = \sqrt{1 - \frac{1}{n-1}} (1-\alpha). \\
\end{split}
\end{equation}

\end{proof}

\subsection{Proposition \ref{prop:transform}}

Proposition \ref{prop:neg_example} tells us about the existence of settings such that common errors are insufficient to capture miscalibration.
We might still wonder how likely it is to encounter such a situation in practice.
Indeed, we can come up with a recalibration transformation that is \textit{perfect} according to these errors and accuracy-preserving but not calibrated.
For this, assume that $f \colon \mathcal{X} \to \mathcal{P}_n$ is a trained model.
Define $t^f \colon \mathcal{P}_n \to \mathcal{P}_n$ to replace the largest entry in its input with the accuracy of model $f$.
The other entries are set such that the output is a unit vector.
A more formal definition is provided in the proof.

\begin{proposition}
    For all models $f \colon \mathcal{X} \to \mathcal{P}_n$ and $E \in \left\{ \text{MMCE}, \text{KS}, \text{ECE}, \text{TCE}_p \mid 1 \leq p \in \mathbb{R} \right\}$
    we have
    \begin{equation*}
    \begin{split}
        E \left( t^f \circ f \right) = 0 \quad \text{and} \quad \text{ACC} \left( t^f \circ f \right) = \text{ACC} \left( f \right).
    \end{split}
    \end{equation*}
    But, 
    $t^f \circ f$ is not calibrated in general.
\label{prop:transform}
\end{proposition}

\begin{proof}
    Assume we are given a model $f \colon \mathcal{X} \to \mathcal{P}_n$.
    
    Define $\sigma \colon \mathcal{P}_n \times \mathcal{P}_n \to \mathcal{P}_n$ to order the entries of its second input according to the values given in the first input.
    Let $\sigma^{-1} \colon \mathcal{P}_n \times \mathcal{P}_n \to \mathcal{P}_n$ revert the ordering in the second input according to the entries of its first input.
    For easier notation, we will write $\sigma_u \left( v \right) \coloneqq \sigma \left(u, v \right)$ and $\sigma^{-1}_u \left( v \right) \coloneqq \sigma^{-1} \left(u, v \right)$, which gives $\forall u,v \in \mathcal{P} \colon \; \sigma^{-1}_u \circ \sigma_u \left( v \right) = v$. I.e. $\sigma^{-1}_u$ is the inverse of $\sigma_u$ given $u$.
    
    
    
    Define $c_f \coloneqq \left( \text{ACC} \left( f \right), \frac{1 - \text{ACC} \left( f \right)}{n - 1}, \dots, \frac{1 - \text{ACC} \left( f \right)}{n - 1} \right)^\intercal \in \mathcal{P}_n$.
    
    Now, we can give a formal definition of $t^f$, which is defined as $t^f \left( p \right) = \sigma^{-1}_p \left( c_f \right)$.

    \textbf{Regarding accuracy:}
    
    We will use $[.]_k$ to denote entry with index $k$ of the expression inside the brackets.
    Since we can assume $\text{ACC} \left( f \right) > \frac{1 - \text{ACC} \left( f \right)}{n - 1}$ in every practical setting, we have 
    \begin{equation}
    \begin{split}
        & \arg\max_k t^f_k \circ f \left( X \right) \\
        & = \arg\max_k \left[\sigma^{-1}_{f \left( X \right)} \left( c_f \right) \right]_k \\
        \overset{(i)}&{=} \arg\max_k \left[\sigma^{-1}_{f \left( X \right)} \circ \sigma_{f \left( X \right)} \left( f \left( X \right) \right) \right]_k \\
        & = \arg\max_k \left[ f \left( X \right) \right]_k \\
        & = \arg\max_k f_k \left( X \right).
    \end{split}
    \end{equation}
    (i) $c_f$ and $\sigma_{f \left( X \right)} \left( f \left( X \right) \right)$ have their largest entry at index $k=1$.
    
    This states that $t^f$ is accuracy-preserving.
    
    \textbf{Regarding zero TCE:}

    Note that ACC$\left( f \right) = \mathbb{P} \left( Y = \arg\max_k f_k \left( X \right) \right)$.
    Using this, we have $\mathbb{P} \left( Y = \arg\max_k t^f_k \circ f \left( X \right) \mid \max_k t^f_k \circ f \left( X \right) \right) = \mathbb{P} \left( Y = \arg\max_k f_k \left( X \right) \mid \text{ACC} \left( f \right) \right) = \mathbb{P} \left( Y = \arg\max_k f_k \left( X \right) \right) = \text{ACC} \left( f \right) = \max_k t^f_k \circ f \left( X \right)$.
    It follows TCE$_p \left( t^f \circ f \right) = 0$. \\
    
    Proof for the other errors follows from Theorem \ref{th:ce_relations}.

    
\end{proof}

Even though $t^f$ is the perfect transformation according to ECE and accuracy, it is not the correct choice if the whole model prediction is relevant and supposed to be calibrated.

\subsection{Proposition \ref{prop:ece_estimator}}
\label{app:sec:mu_n}

We will write $\hat{Y} = \arg\max_k f_k \left( X \right)$ for the top-label prediction of classifier $f$.

Define 
\begin{equation}
\mu_{\left(n\right)} = \sum_{i=1}^m p_i \left\{ \sqrt{\frac{2}{\pi}} \sigma_i \exp \left( - \frac{\mu_i^2}{2 \sigma_i^2} \right) + \mu_i \left[ 1 - 2 \Phi \left( - \frac{\mu_i}{\sigma_i} \right) \right] \right\}
\end{equation}
with
\begin{equation}
\mu_i = \underbrace{\mathbb{E} \left[ f_C \left( X \right) \mid f_C \left( X \right) \in B_i \right]}_{= \text{conf}_i} - \underbrace{\mathbb{P} \left( Y = \hat{Y} \mid f_C \left( X \right) \in B_i \right)}_{= \text{acc}_i}
\label{eq:mu_i}
\end{equation}
as the true unknown difference between model confidence and model accuracy in bin $i$,
\begin{equation}
\sigma_i^2 = \frac{1}{n_i} \underbrace{\mathbb{V} \left[ f_C \left( X \right) \mid f_C \left( X \right) \in B_i \right]}_{V^\text{conf}_i \coloneqq } + \frac{1}{n_i} \underbrace{\text{acc}_i \left( 1 - \text{acc}_i \right)}_{V^\text{acc}_i \coloneqq }
\label{eq:sigma_i}
\end{equation}
as the combined model and accuracy sample variance in bin $i$,
and $\Phi$ as the cumulative distribution function (cdf) of a standard normal distribution.

The ECE for data size $n$ and $m$ bins is estimated by

\begin{equation}
\begin{split}
    \hat{\text{ECE}}_{\left(n \right)} = \sum_{i=1}^m \hat{p}_b \left\lvert \hat{\text{acc}}_b - \hat{\text{conf}}_b \right\rvert
\end{split}
\end{equation}

where $\hat{p}_i = \frac{n_i}{n}$ is the estimated bin frequency, $\hat{\text{acc}}_i = \frac{1}{n_i} \sum_j \mathbbm{1} \left\{ \hat{Y}_j = Y_j \land \hat{Y}_j \in B_i \right\}$ the estimated bin accuracy, $\hat{\text{conf}}_i = \frac{1}{n_i} \sum_j \hat{Y}_j \mathbbm{1} \left\{ \hat{Y}_j \in B_i \right\}$ the estimated bin confidence, and $n_i = \sum_j \mathbbm{1} \left\{ \hat{Y}_j \in B_i \right\}$ is the number of data instances in bin $i$.
We assume equal width binning, i.e. $B_i = \left( \frac{i - m}{m}, \frac{i}{m} \right]$.

We have
\begin{equation*}
\begin{split}
    \mathbb{E} \left[ \hat{\text{ECE}}_{\left(n \right)} \right] \approx \mu_{\left(n\right)} \geq \text{ECE} \quad, \quad
    \frac{\mathrm{d} \mu_{\left(n\right)}}{\mathrm{d} n} < 0 \quad, \quad \frac{\mathrm{d}^2 \mu_{\left(n\right)}}{\left(\mathrm{d} n \right)^2} > 0 \quad \text{and} \quad \frac{\mathrm{d}^2 \mu_{\left(n\right)}}{\mathrm{d} n \; \mathrm{d} \text{ECE}} > 0.
\end{split}
\end{equation*}

\begin{proof}

First,

\begin{equation}
\begin{split}
    & \mathbb{E} \left[ \hat{\text{ECE}}_{\left(n \right)} \right] \\
    \overset{\text{def}}&{=} \mathbb{E} \left[ \sum_{i=1}^m \hat{p}_i \left\lvert \hat{\text{acc}}_i - \hat{\text{conf}}_i \right\rvert \right] \\
    \overset{\text{def}}&{=} \mathbb{E} \left[ \sum_{i=1}^m \frac{1}{n} \sum_{j=1}^n \mathbbm{1} \left\{ \hat{Y}_j \in B_i \right\} \left\lvert \hat{\text{acc}}_i - \hat{\text{conf}}_i \right\rvert \right] \\
    & = \frac{1}{n} \sum_{j=1}^n \mathbb{E} \left[ \sum_{i=1}^m \mathbbm{1} \left\{ \hat{Y}_j \in B_i \right\} \left\lvert \hat{\text{acc}}_i - \hat{\text{conf}}_i \right\rvert \right] \\
    & = \frac{1}{n} \sum_{i=1}^n \mathbb{E} \left[ \sum_{i=1}^m \mathbbm{1} \left\{ \hat{Y}_j \in B_i \right\} \mathbb{E} \left[ \left\lvert \hat{\text{acc}}_i - \hat{\text{conf}}_i \right\rvert \mid \hat{Y}_j \right] \right] \\
    \overset{\text{(i)}}&{\approx} \frac{1}{n} \sum_{j=1}^n \sum_{i=1}^m \mathbb{E} \left[ \mathbbm{1} \left\{ \hat{Y}_j \in B_i \right\} \right] \mathbb{E} \left[ \left\lvert \hat{\text{acc}}_i - \hat{\text{conf}}_i \right\rvert \right] \\
    \overset{\text{iid}}&{=} \sum_{i=1}^m \mathbb{P} \left( \hat{Y} \in B_i \right) \mathbb{E} \left[ \left\lvert \hat{\text{acc}}_i - \hat{\text{conf}}_i \right\rvert \right] \\
\label{eq:exp_ece_1}
\end{split}
\end{equation}

(i) 'knowing' a single summand in a mean estimator does not change much. 

As one can see, the ECE estimator approximately consists of several $\mathbb{E} \left[ \left\lvert \hat{\text{acc}}_i - \hat{\text{conf}}_i \right\rvert \right]$.
According to the central limit theorem (CLT), we have $\lim_{n_i \to \infty} \left( \frac{\hat{\text{acc}}_i - \text{acc}_i}{\sqrt{\sfrac{ V^\text{acc}_i}{n_i}}} \right) \sim \mathcal{N} \left( 0, 1 \right)$ and $\lim_{n_i \to \infty} \left( \frac{\hat{\text{conf}}_i - \text{conf}_i }{\sqrt{\sfrac{ V^\text{conf}_i }{n_i}}} \right) \sim \mathcal{N} \left( 0, 1 \right)$.
We assume $\hat{\text{acc}}_i$ and $\hat{\text{conf}}_i$ approximately follow the normal distributions given by the CLT, i.e. $\hat{\text{acc}}_i \sim \mathcal{N} \left( \text{acc}_i, \frac{V^\text{acc}_i}{n_i} \right)$ and $\hat{\text{conf}}_i \sim \mathcal{N} \left( \text{conf}_i, \frac{ V^\text{conf}_i}{n_i} \right)$. \\
This gives $\hat{\text{acc}}_i - \hat{\text{conf}}_i \sim \mathcal{N} \left( \text{acc}_i - \text{conf}_i, \frac{V^\text{conf}_i + V^\text{acc}_i}{n_i} \right)$. \footnote{http://www.stat.ucla.edu/$\sim$nchristo/introstatistics/introstats\_normal\_linear\_combinations.pdf}
If $X \sim \mathcal{N} \left( \mu, \sigma^2 \right)$, then $\left\lvert X \right\rvert$ is a folded normal distribution (FN) with $\mathbb{E} \left[ \left\lvert X \right\rvert \right] = \sqrt{\frac{2}{\pi}} \sigma \exp \left( - \frac{\mu^2}{2 \sigma^2} \right) + \mu \left[ 1 - 2 \Phi \left( - \frac{\mu}{\sigma} \right) \right]$ with $\Phi$ the cdf of a standard normal distribution \citep{Tsagris2014}. 
We also have 
\begin{equation}
    \sqrt{\frac{2}{\pi}} \sigma \exp \left( - \frac{\mu^2}{2 \sigma^2} \right) + \mu \left[ 1 - 2 \Phi \left( - \frac{\mu}{\sigma} \right) \right] = \mathbb{E} \left[ \left\lvert X \right\rvert \right] \geq \left\lvert \mathbb{E} \left[ X \right] \right\rvert = \left\lvert \mu \right\rvert
\label{eq:mu_ub_ece}
\end{equation}

(by Jensen's inequality) and

\begin{equation}
\begin{split}
    \frac{\mathrm{d}}{\mathrm{d} \sigma} \mathbb{E} \left[ \left\lvert X \right\rvert \right]
    & = \frac{\mathrm{d}}{\mathrm{d} \sigma} \left( \sqrt{\frac{2}{\pi}} \sigma \exp \left( - \frac{\mu^2}{2 \sigma^2} \right) + \mu \left[ 1 - 2 \Phi \left( - \frac{\mu}{\sigma} \right) \right] \right) \\ 
    & = \sqrt{\frac{2}{\pi}} \exp \left( - \frac{\mu^2}{2 \sigma^2} \right) + \sqrt{\frac{2}{\pi}} \sigma \exp \left( - \frac{\mu^2}{2 \sigma^2} \right) \frac{\mu^2}{\sigma^3} - \mu 2 \phi \left( - \frac{\mu}{\sigma} \right)\frac{\mu}{\sigma^2} \\ 
    & = \sqrt{\frac{2}{\pi}} \exp \left( - \frac{\mu^2}{2 \sigma^2} \right) + \sqrt{\frac{2}{\pi}} \exp \left( - \frac{\mu^2}{2 \sigma^2} \right) \frac{\mu^2}{\sigma^2} - \frac{2}{\sqrt{2 \pi}} \exp \left( - \frac{\mu^2}{2 \sigma^2} \right)\frac{\mu^2}{\sigma^2} \\ 
    & = \sqrt{\frac{2}{\pi}} \exp \left( - \frac{\mu^2}{2 \sigma^2} \right).
\label{eq:fn_deriv}
\end{split}
\end{equation}

Consequently, $\left\lvert \hat{\text{acc}}_i - \hat{\text{conf}}_i \right\rvert$ follows approximately a folded normal distribution with 

\begin{equation}
\begin{split}
    \mathbb{E} \left[ \left\lvert \hat{\text{acc}}_i - \hat{\text{conf}}_i \right\rvert \right] \approx \sqrt{\frac{2}{\pi}} \sigma_i \exp \left( - \frac{\mu_i^2}{2 \sigma_i^2} \right) + \mu_i \left[ 1 - 2 \Phi \left( - \frac{\mu_i}{\sigma_i} \right) \right]
\label{eq:abs_acc_conf}
\end{split}
\end{equation}

where $\mu_i$ and $\sigma_i$ are defined as above in equations \eqref{eq:mu_i} and \eqref{eq:sigma_i}. \\
Consequently, by combining equations  \eqref{eq:mu_ub_ece}, \eqref{eq:abs_acc_conf} and \eqref{eq:exp_ece_1} we get the first result:

\begin{equation}
\begin{split}
    \mathbb{E} \left[ \hat{\text{ECE}}_{\left(n \right)} \right] & \approx \underbrace{\sum_{i=1}^m p_i \left\{ \sqrt{\frac{2}{\pi}} \sigma_i \exp \left( - \frac{\mu_i^2}{2 \sigma_i^2} \right) + \mu_i \left[ 1 - 2 \Phi \left( - \frac{\mu_i}{\sigma_i} \right) \right] \right\}}_{= \mu_{\left(n\right)}} \\
    & \geq \sum_{i=1}^m p_i \left\lvert \mu_i \right\rvert \\
    & = \sum_{i=1}^m p_i \left\lvert \text{acc}_i - \text{conf}_i \right\rvert \\
    & = \text{ECE}
\end{split}    
\end{equation}

As we can see, the average outcome depends on $\sigma_i$, which further depends on $n_i$, i.e. the data size influences our expected result.
To get the next result, which shows the trend of this influence, we calculate the first derivative:

\begin{equation}
\begin{split}
    & \frac{\mathrm{d}}{\mathrm{d} n} \mu_{\left( n \right)} \\
    & = \frac{\mathrm{d}}{\mathrm{d} n} \sum_{i=1}^m p_i \left\{ \sqrt{\frac{2}{\pi}} \sigma_i \exp \left( - \frac{\mu_i^2}{2 \sigma_i^2} \right) + \mu_i \left[ 1 - 2 \Phi \left( - \frac{\mu_i}{\sigma_i} \right) \right] \right\} \\
    & = \sum_{j=1}^m \frac{\mathrm{d} n_j}{\mathrm{d} n} \frac{\mathrm{d} \sigma_j}{\mathrm{d} n_j} \frac{\mathrm{d}}{\mathrm{d} \sigma_j} \sum_{i=1}^m p_i \left\{ \sqrt{\frac{2}{\pi}} \sigma_i \exp \left( - \frac{\mu_i^2}{2 \sigma_i^2} \right) + \mu_i \left[ 1 - 2 \Phi \left( - \frac{\mu_i}{\sigma_i} \right) \right] \right\} \\
    & = \sum_{j=1}^m \frac{\mathrm{d} n_j}{\mathrm{d} n} \frac{\mathrm{d} \sigma_j}{\mathrm{d} n_j} \frac{\mathrm{d}}{\mathrm{d} \sigma_j} p_j \left\{ \sqrt{\frac{2}{\pi}} \sigma_j \exp \left( - \frac{\mu_j^2}{2 \sigma_j^2} \right) + \mu_j \left[ 1 - 2 \Phi \left( - \frac{\mu_j}{\sigma_j} \right) \right] \right\} \\
    \overset{\eqref{eq:fn_deriv}}&{=} \sum_{j=1}^m \frac{\mathrm{d} n_j}{\mathrm{d} n} \frac{\mathrm{d} \sigma_j}{\mathrm{d} n_j} p_j \sqrt{\frac{2}{\pi}} \exp \left( - \frac{\mu_j^2}{2 \sigma_j^2} \right) \\
    & = \sum_{j=1}^m \frac{\mathrm{d} n_j}{\mathrm{d} n} \frac{- \sqrt{V^\text{conf}_j + V^\text{acc}_j }}{2 n_j^{\sfrac{3}{2}}} p_j \sqrt{\frac{2}{\pi}} \exp \left( - \frac{\mu_j^2}{2 \sigma_j^2} \right) \\
    & = \sum_{j=1}^m 1 \cdot \underbrace{\frac{- \sqrt{V^\text{conf}_j + V^\text{acc}_j }}{2 n_j^{\sfrac{3}{2}}}}_{< 0} \underbrace{p_j \sqrt{\frac{2}{\pi}} \exp \left( - \frac{\mu_j^2}{2 \sigma_j^2} \right)}_{> 0}.
\label{eq:mu_1st_deriv}
\end{split}
\end{equation}

This means $\mu_{\left( n \right)}$ is monotonically decreasing with growing data size.

Next, we have

\begin{equation}
\begin{split}
    & \frac{\mathrm{d}^2}{\left( \mathrm{d} n \right)^2} \mu_{\left( n \right)} \\
    \overset{\eqref{eq:mu_1st_deriv}}&{=} \frac{\mathrm{d}}{ \mathrm{d} n } \sum_{j=1}^m \frac{- \sqrt{V^\text{conf}_j + V^\text{acc}_j }}{2 n_j^{\sfrac{3}{2}}}p_j \sqrt{\frac{2}{\pi}} \exp \left( - \frac{\mu_j^2}{2 \sigma_j^2} \right) \\
    & = \frac{\mathrm{d}}{ \mathrm{d} n } \sum_{j=1}^m \frac{- \sqrt{V^\text{conf}_j + V^\text{acc}_j }}{2 n_j^{\sfrac{3}{2}}}p_j \sqrt{\frac{2}{\pi}} \exp \left( - \frac{\mu_j^2 n_j}{2 V^\text{conf}_j + 2 V^\text{acc}_j} \right) \\
    & = \sum_{i=1}^m \frac{\mathrm{d} n_i}{ \mathrm{d} n } \frac{\mathrm{d}}{ \mathrm{d} n_i } \sum_{j=1}^m \frac{- \sqrt{V^\text{conf}_j + V^\text{acc}_j }}{2 n_j^{\sfrac{3}{2}}}p_j \sqrt{\frac{2}{\pi}} \exp \left( - \frac{\mu_j^2 n_j}{2 V^\text{conf}_j + 2 V^\text{acc}_j} \right) \\
    & = \sum_{i=1}^m 1 \cdot \frac{\mathrm{d}}{ \mathrm{d} n_i } \frac{- \sqrt{V^\text{conf}_i + V^\text{acc}_i }}{2 n_i^{\sfrac{3}{2}}}p_i \sqrt{\frac{2}{\pi}} \exp \left( - \frac{\mu_i^2 n_i}{2 V^\text{conf}_i + 2 V^\text{acc}_i} \right) \\
    & = \sum_{i=1}^m \underbrace{\frac{3 \sqrt{V^\text{conf}_i + V^\text{acc}_i }}{4 n_i^{\sfrac{5}{2}}}p_i \sqrt{\frac{2}{\pi}} \exp \left( - \frac{\mu_i^2 n_i}{2 V^\text{conf}_i + 2 V^\text{acc}_i} \right)}_{> 0} \\
    & \quad + \underbrace{\frac{\sqrt{V^\text{conf}_i + V^\text{acc}_i }}{2 n_i^{\sfrac{3}{2}}}p_i \sqrt{\frac{2}{\pi}} \exp \left( - \frac{\mu_i^2 n_i}{2 V^\text{conf}_i + 2 V^\text{acc}_i} \right) \frac{\mu_i^2}{2 V^\text{conf}_i + 2 V^\text{acc}_i}}_{> 0}.
\end{split}
\end{equation}

This means, in combination with the previous result, $\mu_{\left( n \right)}$ is a strictly convex and monotonically decreasing function of the data size $n_b$.
The ECE estimate is increasingly sensitive to the data size for smaller data sizes, while for larger data sizes the sensitivity vanishes.

Next, we analyze how the goodness of calibration influences this behaviour.
Recall that $\mu_i = \text{acc}_i - \text{conf}_i$, i.e. $\mu_i^2$ is the ground truth squared calibration error of bin $i$.
We have 

\begin{equation}
\begin{split}
    & \frac{\mathrm{d}^2}{\mathrm{d} n \mathrm{d} \mu_i^2} \mu_{\left( n \right)} \\
    \overset{\eqref{eq:mu_1st_deriv}}&{=} \frac{\mathrm{d}}{ \mathrm{d} \mu_i^2 } \sum_{j=1}^m \frac{- \sqrt{V^\text{conf}_j + V^\text{acc}_j }}{2 n_j^{\sfrac{3}{2}}}p_j \sqrt{\frac{2}{\pi}} \exp \left( - \frac{\mu_j^2}{2 \sigma_j^2} \right) \\
    & = \frac{\mathrm{d}}{ \mathrm{d} \mu_i^2 } \frac{- \sqrt{V^\text{conf}_i + V^\text{acc}_i }}{2 n_i^{\sfrac{3}{2}}}p_i \sqrt{\frac{2}{\pi}} \exp \left( - \frac{\mu_i^2}{2 \sigma_i^2} \right) \\
    & = \underbrace{\frac{\sqrt{V^\text{conf}_i + V^\text{acc}_i }}{4 n_i^{\sfrac{3}{2}} \sigma_i^2}p_i \sqrt{\frac{2}{\pi}} \exp \left( - \frac{\mu_i^2}{2 \sigma_i^2} \right)}_{> 0}
\label{eq:2nd_deriv_mu}
\end{split}
\end{equation}

Consequently, if $\mu_i^2$ increases (i.e. calibration gets worse), the gradient $\frac{\mathrm{d}}{\mathrm{d} n} \mathbb{E} \left[ \hat{\text{ECE}}_{\left( n \right)} \right]$ monotonically approaches zero from beneath.
Contrary, the gradient is the highest when $\mu_i = 0$.
In other words, the sensitivity of the ECE estimate w.r.t. the data size monotonically depends on the goodness of calibration.
With better calibration, the sensitivity gradually gets worse.

Further, we have

\begin{equation}
\begin{split}
    & \frac{\mathrm{d} \mu_i^2}{ \mathrm{d} \text{ECE}} \\
    & = \frac{\mathrm{d}}{ \mathrm{d} \text{ECE}} \left( \text{acc}_i - \text{conf}_i \right)^2 \\
    & = \frac{\mathrm{d}}{ \mathrm{d} \text{ECE}} \left\lvert \text{acc}_i - \text{conf}_i \right\rvert^2 \\
    & = \frac{\mathrm{d}}{ \mathrm{d} \text{ECE}} \left( \frac{\text{ECE} - \sum_{j \neq i} p_j \left\lvert \text{acc}_j - \text{conf}_j \right\rvert}{p_i} \right)^2 \\
    & = \underbrace{2 \frac{\text{ECE} - \sum_{j \neq i} p_j \left\lvert \text{acc}_j - \text{conf}_j \right\rvert}{p_i^2}}_{> 0}.
\label{eq:d_mu_i_d_ece}
\end{split}
\end{equation}

Combining equations \eqref{eq:2nd_deriv_mu} and \eqref{eq:d_mu_i_d_ece} gives $\frac{\mathrm{d}^2}{\mathrm{d} n \mathrm{d} \text{ECE}} \mu_{\left( n \right)} > 0$ as stated in the proposition.

\end{proof}

\subsection{Lemma \ref{lemma:cal_sharp_decomp}}

Let $\mathcal{P}$ be a set of arbitrary distributions for which exists a proper score $S$. Assume we have random variables $Q$ and $Y$ with $Q, \mathbb{P}_Y, \mathbb{P}_{Y \mid Q} \in \mathcal{P}$ for which $g_S \left( \mathbb{P}_Y \right), \mathbb{E} \left[ g_S \left( \mathbb{P}_{Y \mid Q} \right) \right], \mathbb{E} \left[ \left\lvert S \left(Q, Y \right) \right\rvert \right], \mathbb{E} \left[ \left\lvert S \left(\mathbb{P}_Y, Y \right) \right\rvert \right] < \infty$.
The last two expectations are required for Fubini's theorem.

\begin{equation*}
\begin{split}
    \mathbb{E} \left[ S \left(Q, Y \right) \right] & = \int S \left(q, y \right) \mathrm{d} \mathbb{P}_{Y, Q} \left(y, q \right) \\
    \overset{\text{Fubini}}&{=} \int \int S \left(q, y \right) \mathrm{d} \mathbb{P}_{Y \mid Q=q} \left(y \right) \mathrm{d} \mathbb{P}_{Q} \left(q \right) \\
    \overset{\text{def }\ref{def:s_S}}&{=} \int s_S \left(q, \mathbb{P}_{Y \mid Q=q} \right) \mathrm{d} \mathbb{P}_{Q} \left(q \right) \\
    & = \mathbb{E} \left[ s_S \left(Q, \mathbb{P}_{Y \mid Q} \right) \right] \\
    & = \mathbb{E} \left[ s_S \left(\mathbb{P}_{Y \mid Q}, \mathbb{P}_{Y \mid Q} \right) \right] + \mathbb{E} \left[ s_S \left(Q, \mathbb{P}_{Y \mid Q} \right) \right] - \mathbb{E} \left[ s_S \left(\mathbb{P}_{Y \mid Q}, \mathbb{P}_{Y \mid Q} \right) \right] \\
    \overset{\text{def }\ref{def:div_ent}}&{=} \mathbb{E} \left[ s_S \left(\mathbb{P}_{Y \mid Q}, \mathbb{P}_{Y \mid Q} \right) \right] + \mathbb{E} \left[ d_S \left(Q, \mathbb{P}_{Y \mid Q} \right) \right] \\
    & = s_S \left(\mathbb{P}_{Y}, \mathbb{P}_{Y} \right) - s_S \left(\mathbb{P}_{Y}, \mathbb{P}_{Y} \right) + \mathbb{E} \left[ s_S \left(\mathbb{P}_{Y \mid Q}, \mathbb{P}_{Y \mid Q} \right) \right] + \mathbb{E} \left[ d_S \left(Q, \mathbb{P}_{Y \mid Q} \right) \right] \\
    \overset{\text{def }\ref{def:s_S}}&{=} s_S \left(\mathbb{P}_{Y}, \mathbb{P}_{Y} \right) - \int S \left(\mathbb{P}_{Y}, y \right) \mathrm{d} \mathbb{P}_{Y} \left(y \right) + \mathbb{E} \left[ s_S \left(\mathbb{P}_{Y \mid Q}, \mathbb{P}_{Y \mid Q} \right) \right] + \mathbb{E} \left[ d_S \left(Q, \mathbb{P}_{Y \mid Q} \right) \right] \\
    & = s_S \left(\mathbb{P}_{Y}, \mathbb{P}_{Y} \right) - \int S \left(\mathbb{P}_{Y}, y \right) \underbrace{\int \mathrm{d} \mathbb{P}_{Q \mid Y=y} \left(q \right)}_{=1} \mathrm{d} \mathbb{P}_{Y} \left(y\right) + \mathbb{E} \left[ s_S \left(\mathbb{P}_{Y \mid Q}, \mathbb{P}_{Y \mid Q} \right) \right] + \mathbb{E} \left[ d_S \left(Q, \mathbb{P}_{Y \mid Q} \right) \right] \\
    & = s_S \left(\mathbb{P}_{Y}, \mathbb{P}_{Y} \right) - \int S \left(\mathbb{P}_{Y}, y \right) \mathrm{d} \mathbb{P}_{Y, Q} \left(y, q \right) + \mathbb{E} \left[ s_S \left(\mathbb{P}_{Y \mid Q}, \mathbb{P}_{Y \mid Q} \right) \right] + \mathbb{E} \left[ d_S \left(Q, \mathbb{P}_{Y \mid Q} \right) \right] \\
    \overset{\text{Fubini}}&{=} s_S \left(\mathbb{P}_{Y}, \mathbb{P}_{Y} \right) - \int \int S \left(\mathbb{P}_{Y}, y \right) \mathrm{d} \mathbb{P}_{Y \mid Q=q} \left(y \right) \mathrm{d} \mathbb{P}_{Q} \left(q\right) + \mathbb{E} \left[ s_S \left(\mathbb{P}_{Y \mid Q}, \mathbb{P}_{Y \mid Q} \right) \right] + \mathbb{E} \left[ d_S \left(Q, \mathbb{P}_{Y \mid Q} \right) \right] \\
    \overset{\text{def }\ref{def:s_S}}&{=} s_S \left(\mathbb{P}_{Y}, \mathbb{P}_{Y} \right) - \int s_S \left(\mathbb{P}_{Y}, \mathbb{P}_{Y \mid Q=q}\right) \mathrm{d} \mathbb{P}_{Q} \left(q\right) + \mathbb{E} \left[ s_S \left(\mathbb{P}_{Y \mid Q}, \mathbb{P}_{Y \mid Q} \right) \right] + \mathbb{E} \left[ d_S \left(Q, \mathbb{P}_{Y \mid Q} \right) \right] \\
    & = s_S \left(\mathbb{P}_{Y}, \mathbb{P}_{Y} \right) - \mathbb{E} \left[ s_S \left(\mathbb{P}_{Y}, \mathbb{P}_{Y \mid Q} \right) \right] + \mathbb{E} \left[ s_S \left(\mathbb{P}_{Y \mid Q}, \mathbb{P}_{Y \mid Q} \right) \right] + \mathbb{E} \left[ d_S \left(Q, \mathbb{P}_{Y \mid Q} \right) \right] \\
    \overset{\text{def }\ref{def:div_ent}}&{=} \underbrace{g_S \left(\mathbb{P}_Y \right)}_\text{generalized entropy} - \underbrace{\mathbb{E} \left[ d_S \left(\mathbb{P}_Y, \mathbb{P}_{Y \mid Q} \right) \right]}_\text{sharpness} + \underbrace{\mathbb{E} \left[ d_S \left(Q, \mathbb{P}_{Y \mid Q} \right) \right]}_\text{calibration}. \\
\end{split}
\end{equation*}

\subsection{Theorem \ref{th:ub}}

For all proper calibration errors with $\inf_{P \in \mathcal{P}} g_S \left( P \right) \in \mathbb{R}$, there exists an associated \textbf{calibration upper bound} $$ \mathcal{U}_S \left( f \right) \geq \text{CE}_S \left( f \right)$$
defined as $\mathcal{U}_S \left( f \right) \coloneqq \mathbb{E} \left[ S \left( f \left( X \right), Y \right) \right] - \inf_{P \in \mathcal{P}} g_S \left(P\right)$.
Under a classification setting and further mild conditions, it is asymptotically equal to the CE$_S$ with increasing model accuracy, i.e. $$ \lim_{\text{ACC} \left(f\right) \to 1} \mathcal{U}_S \left( f \right) - \text{CE}_S \left( f \right) = 0. $$

\begin{proof}

\textbf{Regarding existence of upper bound}

Assuming $\inf_{Q \in \mathcal{P}} g_S \left( Q \right) \in \mathbb{R}$.

\begin{equation}
\begin{split}
    \text{CE}_S \left( f \right) \overset{\text{le }\ref{le:cal_decomp}}&{=} \mathbb{E} \left[ S \left( f \left( X \right), Y \right) \right] - \mathbb{E} \left[ g_S \left( \mathbb{P}_{Y \mid f \left( X \right)} \right) \right] \\
    & \leq \mathbb{E} \left[ S \left( f \left( X \right), Y \right) \right] - \mathbb{E} \left[ \inf_{Q \in \mathcal{P}} g_S \left( Q \right) \right] \\
    & = \mathbb{E} \left[ S \left( f \left( X \right), Y \right) \right] - \inf_{Q \in \mathcal{P}} g_S \left( Q \right) \\
    \overset{\text{th } \ref{th:ub}}&{=} \mathcal{U}_S \left( f \right)
\label{eq:ub}
\end{split}
\end{equation}

\textbf{Regarding accuracy limes}

Assuming mild conditions $g_S \colon \mathcal{P}_n \to \mathbb{R}$ is continuous and $g_S \left( e_1 \right) = g_S \left( e_2 \right) = \dots = g_S \left( e_n \right)$.
See Figure 2 in \citet{gneitingscores} for an example when this is violated.
$S$ does not have to be symmetric for this to hold.

\begin{equation}
\begin{split}
    & \lim_{\text{ACC} \left(f\right) \to 1} \text{CE}_S \left( f \right) - \mathcal{U}_S \left( f \right) \\
    \overset{\text{th }\ref{th:ub}}&{=} \lim_{\text{ACC} \left(f\right) \to 1}  \text{CE}_S \left( f \right) - \mathbb{E} \left[ S \left( f \left( X \right), Y \right) \right] + \inf_{Q \in \mathcal{P}_n} g_S \left( Q \right) \\
    \overset{\text{le }\ref{le:cal_decomp}}&{=} \lim_{\text{ACC} \left(f\right) \to 1}  \mathbb{E} \left[ S \left( f \left( X \right), Y \right) \right] - \mathbb{E} \left[ g_S \left( \mathbb{P}_{Y \mid f \left( X \right)} \right) \right] - \mathbb{E} \left[ S \left( f \left( X \right), Y \right) \right] + \inf_{Q \in \mathcal{P}_n} g_S \left( Q \right) \\
    & = \lim_{\text{ACC} \left(f\right) \to 1} \inf_{Q \in \mathcal{P}_n} g_S \left( Q \right) - \mathbb{E} \left[ g_S \left( \mathbb{P}_{Y \mid f \left( X \right)} \right) \right] \\
    & = \inf_{Q \in \mathcal{P}_n} g_S \left( Q \right) - \lim_{\text{ACC} \left(f\right) \to 1} \mathbb{E} \left[ g_S \left( \mathbb{P}_{Y \mid f \left( X \right)} \right) \right] \\
    & = \inf_{Q \in \mathcal{P}_n} g_S \left( Q \right) - \mathbb{E} \left[ g_S \left( \lim_{\text{ACC} \left(f\right) \to 1} \mathbb{P}_{Y \mid f \left( X \right)} \right) \right] \\
    \overset{\text{(i)}}&{=} \inf_{Q \in \mathcal{P}_n} g_S \left( Q \right) - \mathbb{E} \left[ g_S \left( e_{i \left( X \right)} \right) \right] \\
    \overset{\text{(ii)}}&{=} \inf_{Q \in \mathcal{P}_n} g_S \left( Q \right) - \mathbb{E} \left[ g_S \left( e_1 \right) \right] \\
    & = \inf_{Q \in \mathcal{P}_n} g_S \left( Q \right) - g_S \left( e_1 \right) \\
    \overset{\text{(iii)}}&{=} \inf_{Q \in \mathcal{P}_n} g_S \left( Q \right) - \inf_{Q \in \mathcal{P}_n} g_S \left( Q \right) \\
    & = 0 \\
\end{split}
\end{equation}

(i) Perfect accuracy results in deterministic predictions, i.e. $\forall z \in \mathcal{P}_n \colon \; \lim_{\text{ACC} \left(f\right) \to 1} \mathbb{P}_{Y \mid f \left( X \right)=z} \in \left\{ e_i \mid n \geq i \in \mathbb{N} \right\}$. If we define $i \colon \mathcal{X} \to \mathbb{N}_{\leq n}$ as $i \left( X \right) \coloneqq \arg\max_k \lim_{\text{ACC} \left(f\right) \to 1} \mathbb{P} \left(Y = k \mid f \left( X \right) \right)$, then we have $e_{i \left( X \right)} = \lim_{\text{ACC} \left(f\right) \to 1} \mathbb{P}_{Y \mid f \left( X \right)}$. \\
(ii) Follows from initial condition. \\
(iii) Since $g_S$ is concave and by the definition of $\mathcal{P}_n$, we have 
\begin{equation}
\forall z \in \mathcal{P}_n \exists \lambda_1, \dots, \lambda_{n} \geq 0, \sum_k \lambda_k = 1 \colon \; g_S \left(z\right) = g_S \left(\sum_k \lambda_k e_k \right) \geq \sum_k \lambda_k g_S \left( e_k \right) = \sum_k \lambda_k g_S \left( e_1 \right) = g_S \left( e_1 \right).
\end{equation}
From this follows that $g_S \left( e_1 \right) = \inf_{Q \in \mathcal{P}_n} g_S \left( Q \right)$.

\end{proof}

\subsection{Proposition \ref{prop:ub_grad}}

    Given injective functions $h, h' \; \colon \; \mathcal{P} \to \mathcal{P}$ we have 
    $$ \mathcal{U}_S \left( h \circ f \right) - \mathcal{U}_S \left( f \right) = \text{CE}_S \left( h \circ f \right) - \text{CE}_S \left( f \right) \quad \text{, }$$
    $$ \mathcal{U}_S \left( h \circ f \right) > \mathcal{U}_S \left( h' \circ f \right) \iff \text{CE}_S \left( h \circ f \right) > \text{CE}_S \left( h' \circ f \right)$$
    and (assuming $S$ is differentiable)
    $$\frac{\mathrm{d} \mathcal{U}_S \left( h \circ f \right)}{\mathrm{d} h} = \frac{\mathrm{d} \text{CE}_S \left( h \circ f \right)}{\mathrm{d} h}.$$

\begin{proof}

\begin{equation}
\begin{split}
    & \mathcal{U}_S \left( h \circ f \right) - \mathcal{U}_S \left( h' \circ f \right) \\
    \overset{\text{th }\ref{th:ub}}&{=} \mathbb{E} \left[ S \left( h \circ f \left( X \right), Y \right) \right] - \inf_{Q \in \mathcal{P}_n} g_S \left( Q \right) - \mathbb{E} \left[ S \left( h' \circ f \left( X \right), Y \right) \right] + \inf_{Q \in \mathcal{P}_n} g_S \left( Q \right) \\
    & = \mathbb{E} \left[ S \left( h \circ f \left( X \right), Y \right) \right] - \mathbb{E} \left[ S \left( h' \circ f \left( X \right), Y \right) \right] \\
    & = \mathbb{E} \left[ S \left( h \circ f \left( X \right), Y \right) \right] - \mathbb{E} \left[ g_S \left( \mathbb{P}_{Y \mid f \left( X \right)} \right) \right] - \mathbb{E} \left[ S \left( h' \circ f \left( X \right), Y \right) \right] + \mathbb{E} \left[ g_S \left( \mathbb{P}_{Y \mid f \left( X \right)} \right) \right] \\
    \overset{\text{(i)}}&{=} \mathbb{E} \left[ S \left( h \circ f \left( X \right), Y \right) \right] - \mathbb{E} \left[ g_S \left( \mathbb{P}_{Y \mid h \circ f \left( X \right)} \right) \right] - \mathbb{E} \left[ S \left( h' \circ f \left( X \right), Y \right) \right] + \mathbb{E} \left[ g_S \left( \mathbb{P}_{Y \mid h' \circ f \left( X \right)} \right) \right] \\
    \overset{\text{le }\ref{le:cal_decomp}}&{=} \text{CE}_S \left( h \circ f \right) - \text{CE}_S \left( h' \circ f \right)
\end{split}
\end{equation}

from which follows $ \mathcal{U}_S \left( h \circ f \right) - \mathcal{U}_S \left( f \right) = \text{CE}_S \left( h \circ f \right) - \text{CE}_S \left( f \right)$ and
$ \mathcal{U}_S \left( h \circ f \right) > \mathcal{U}_S \left( h' \circ f \right) \iff \text{CE}_S \left( h \circ f \right) > \text{CE}_S \left( h' \circ f \right)$.
Further we have for differentiable $S$

\begin{equation}
\begin{split}
    & \frac{\mathrm{d} \mathcal{U}_S \left( h \circ f \right)}{\mathrm{d} h} \\
    \overset{\text{th }\ref{th:ub}}&{=} \frac{\mathrm{d} \mathbb{E} \left[ S \left( h \circ f \left( X \right), Y \right) \right] - \inf_{Q \in \mathcal{P}_n} g_S \left( Q \right)}{\mathrm{d} h} \\
    & = \frac{\mathrm{d} \mathbb{E} \left[ S \left( h \circ f \left( X \right), Y \right) \right]}{\mathrm{d} h} \\
    & = \frac{\mathrm{d} \mathbb{E} \left[ S \left( h \circ f \left( X \right), Y \right) \right] - \mathbb{E} \left[ g_S \left( \mathbb{P}_{Y \mid f \left( X \right)} \right) \right]}{\mathrm{d} h} \\
    \overset{\text{(i)}}&{=} \frac{\mathrm{d} \mathbb{E} \left[ S \left( h \circ f \left( X \right), Y \right) \right] - \mathbb{E} \left[ g_S \left( \mathbb{P}_{Y \mid h \circ f \left( X \right)} \right) \right]}{\mathrm{d} h} \\
    \overset{\text{le }\ref{le:cal_decomp}}&{=} \frac{\mathrm{d} \text{CE}_S \left( h \circ f \right)}{\mathrm{d} h} \\
\end{split}
\end{equation}

(i) Since $h$ is injective, we have $\forall z \in \mathcal{P}_n \colon \; \left\{ x \in \mathcal{X} \mid f \left( x \right) = z \right\} = \left\{ x \in \mathcal{X} \mid h \circ f \left( x \right) = h \left( z \right) \right\}$ and $\left\{ \left(x, y \right) \in \mathcal{X} \times \mathcal{Y} \mid f \left( x \right) = z \right\} = \left\{ \left(x, y\right) \in \mathcal{X} \times \mathcal{Y} \mid h \circ f \left( x \right) = h \left( z \right) \right\}$.
Consequently $\mathbb{P} \left( Y \mid f \left( X \right) = z \right) = \frac{\mathbb{P} \left( Y, f \left( X \right) = z \right)}{\mathbb{P} \left( f \left( X \right) = z \right)} = \frac{\mathbb{P} \left( Y, h \circ f \left( X \right) = h \left(z\right) \right)}{\mathbb{P} \left( h \circ f \left( X \right) = h \left(z\right) \right)} = \mathbb{P} \left( Y \mid h \circ f \left( X \right) = h \left(z\right) \right)$.
\end{proof}

\section{Recalibration transformations}
\label{app:recal_trafos}

\subsection{calibrated and accuracy-preserving}
\label{zero_ce_trafo}

The binary case is directly given in the multi-class case, but if we only have a scalar output, which is common for binary classification, deriving the transformation is not that trivial.
Consequently, we handle this case separately.

We will also make use of the following lemma.

\begin{lemma}
    For random variables $Y$ and $X$, we have
    $$ \mathbb{P} \left( Y \mid \mathbb{P} \left( Y \mid X \right) \right) = \mathbb{P} \left( Y \mid X \right). $$
\label{le:exp_exp}
\end{lemma}

\begin{proof}
Proof directly follows from Proposition 1 in \citet{vaicenavicius2019evaluating} with $h \equiv \mathrm{id}$.
\end{proof}

\subsubsection{Binary case (scalar output)}

Assume we are given $f \colon \mathcal{X} \to \left[0, 1 \right]$.

Define $t^f \colon \left[0, 1 \right] \to \left[0, 1 \right]$ as
\begin{equation}
\begin{split}
    t^f \left( p \right) = 
        \begin{cases}
        \mathbb{P} \left( Y = 1 \mid f \left( X \right) < 0.5 \right) & \text{, if } p < 0.5 \\
        \mathbb{P} \left( Y = 1 \mid f \left( X \right) \geq 0.5 \right) & \text{, else}
        \end{cases}
\end{split}
\end{equation}

The first line has as unbiased estimator the precision (or positive predictive value), the second the false omission rate.

This gives 

\begin{equation}
\begin{split}
    \mathbb{P} \left( Y = 1 \mid t^f \circ f \left( X \right) \right) & = \begin{cases}
        \mathbb{P} \left( Y = 1 \mid \mathbb{P} \left( Y = 1 \mid f \left( X \right) < 0.5 \right) \right) & \text{, if } f \left( X \right) < 0.5 \\
        \mathbb{P} \left( Y = 1 \mid \mathbb{P} \left( Y = 1 \mid f \left( X \right) \geq 0.5 \right) \right) & \text{, else}
    \end{cases} \\
    & = \begin{cases}
        \mathbb{P} \left( Y = 1 \mid f \left( X \right) < 0.5 \right) & \text{, if } f \left( X \right) < 0.5 \\
        \mathbb{P} \left( Y = 1 \mid f \left( X \right) \geq 0.5 \right) & \text{, else}
        \end{cases} \\
        & = t^f \circ f \left( X \right)
\end{split}
\end{equation}

i.e. $t^f \circ f$ is calibrated.
Further, if $\mathbb{P} \left( Y = 1 \mid f \left( X \right) < 0.5 \right) < \mathbb{P} \left( Y = 1 \mid f \left( X \right) \geq 0.5 \right)$, then $t^f \circ f$ has the same accuracy as $f$.
This can be assumed as given for any meaningful classifier.
The reduction in sharpness directly follows from the analog proof in the multi-class case.

\subsubsection{Multi-class case (vector output)}

Let $r \colon \mathcal{P}_n \to A$ with $A = \left\{ a \in \left\{0, 1 \right\}^K \mid \sum_k a_k = 1 \right\}$ be defined as $r \left( p \right) \coloneqq e_{\arg\max_k p_k}$.
In words, $r$ returns a vector of only zeros except a '1' at index $\arg\max_k p_k$ for input $p \in \mathcal{P}_n$.

Define $t^f \colon \mathcal{P}_n \to \mathcal{P}_n$ as
\begin{equation}
\begin{split}
    t^f \left( p \right) = \mathbb{P} \left( Y \mid r \circ f \left( X \right) = r \left( p \right) \right)
\end{split}
\end{equation}

(For easier notation, we say $\mathbb{P} \left( Y \right) \in \mathcal{P}_n$ )

Given a dataset $\left\{ \left( X_1, Y_1 \right), \dots, \left( X_m, Y_m \right) \right\}$, an unbiased estimator of $\mathbb{P} \left( Y \mid r \circ f \left( X \right) = a \right)$ $\forall a \in A$ is
$P_a = \frac{1}{\left\lvert I_a \right\rvert} \sum_{i \in I_a} e_{Y_i}$ with $I_a = \left\{ i \in \left\{1, \dots, m \right\} \mid r \circ f \left( X_i \right) = a \right\} $.
And since $\left\lvert A \right\rvert = n$, estimation is also feasible for higher number of classes.

We also have

\begin{equation}
\begin{split}
    \mathbb{P} \left( Y \mid t^f \circ f \left( X \right) \right) & =
    \mathbb{P} \left( Y \mid \mathbb{P} \left( Y \mid r \circ f \left( X \right) = r \circ f \left( X \right) \right) \right) \\
    & = \mathbb{P} \left( Y \mid \mathbb{P} \left( Y \mid r \circ f \left( X \right) \right) \right) \\
    & = \mathbb{P} \left( Y \mid r \circ f \left( X \right) \right) \\
    & = t^f \circ f \left( X \right)
\end{split}
\end{equation}

Consequently, $t^f \circ f$ is calibrated.

If $\arg\max_k f_k \left( X \right) = \arg\max_k \mathbb{P} \left( Y = k \mid \arg\max_k f_k \left( X \right) \right)$, then $\arg\max_k f_k \left( X \right) = \arg\max_k \mathbb{P} \left( Y = k \mid r \circ f \left( X \right) \right) = \arg\max_k \mathbb{P} \left( Y = k \mid r \circ f \left( X \right) = r \circ f \left( X \right) \right) = \arg\max_k t^f_k \circ f \left( X \right)$, i.e. $t^f$ is accuracy preserving.
Recall that $\arg\max_k f_k \left( X \right)$ is the predicted top-label, making $\arg\max_k \mathbb{P} \left( Y = k \mid \arg\max_k f_k \left( X \right) \right)$ the most likely outcome given a predicted top-label.
So, we can restate the above as: $t^f$ is accuracy preserving if for every predicted top-label the most likely outcome is that label.
This should hold in every meaningful practical setting, or else $t^f$ might as well improve the accuracy.

$t^f \circ f$ has lower sharpness as $f$ w.r.t. a proper score $S$.
This is a special case of the following proposition, where we write $\text{SHARP}_S \left( f \right)$ as the sharpness of model $f$ given by the sharpness term in Lemma \ref{lemma:cal_sharp_decomp} of a proper score $S$.

\begin{proposition}
    Assume Lemma \ref{lemma:cal_sharp_decomp} holds given a proper score $S$.
    For a function $m \colon \mathcal{P}_n \to \mathcal{P}_n$ and model $f \colon \mathcal{X} \to \mathcal{P}_n$, we have
    $$ \text{SHARP}_S \left( f \right) \geq \text{SHARP}_S \left( m \circ f \right).$$
\end{proposition}

\begin{proof}

Since we assumed Lemma \ref{lemma:cal_sharp_decomp} holds, the conditions for Fubini's theorem are met.
We will use:

\begin{equation}
\begin{split}
    & \text{SHARP}_S \left( f \right) \\
    \overset{\text{le }\ref{lemma:cal_sharp_decomp}}&{=} \mathbb{E} \left[ d_S \left( \mathbb{P}_{Y}, \mathbb{P}_{Y \mid f \left( X \right)} \right) \right] \\
    \overset{\text{def }\ref{def:div_ent}}&{=} \mathbb{E} \left[ s_S \left( \mathbb{P}_{Y}, \mathbb{P}_{Y \mid f \left( X \right)} \right) \right] - \mathbb{E} \left[ g_S \left( \mathbb{P}_{Y \mid f \left( X \right)} \right) \right] \\
    \overset{\text{def }\ref{def:s_S}}&{=} \int \int S \left( \mathbb{P}_{Y}, y \right) \mathrm{d} \mathbb{P}_{Y \mid f \left( X \right)=z} \left( y \right) \mathrm{d} \mathbb{P}_{f \left( X \right)} \left( z \right) - \mathbb{E} \left[ g_S \left( \mathbb{P}_{Y \mid f \left( X \right)} \right) \right] \\
    & = \int S \left( \mathbb{P}_{Y}, y \right) \mathrm{d} \mathbb{P}_{Y, f \left( X \right)} \left( y, z \right) - \mathbb{E} \left[ g_S \left( \mathbb{P}_{Y \mid f \left( X \right)} \right) \right] \\
    \overset{\text{Fubini}}&{=} \int S \left( \mathbb{P}_{Y}, y \right) \int \mathrm{d} \mathbb{P}_{f \left( X \right) \mid Y=y} \left( z \right) \mathrm{d} \mathbb{P}_{Y} \left( y \right) - \mathbb{E} \left[ g_S \left( \mathbb{P}_{Y \mid f \left( X \right)} \right) \right] \\
    & = \int S \left( \mathbb{P}_{Y}, y \right) \mathrm{d} \mathbb{P}_{Y} \left( y \right) - \mathbb{E} \left[ g_S \left( \mathbb{P}_{Y \mid f \left( X \right)} \right) \right] \\
    \overset{\text{def }\ref{def:div_ent}}&{=} g_S \left( \mathbb{P}_{Y} \right) - \mathbb{E} \left[ g_S \left( \mathbb{P}_{Y \mid f \left( X \right)} \right) \right] \\
\label{eq:sharp_dec}
\end{split}
\end{equation}

Now, we can show

\begin{equation}
\begin{split}
    & \text{SHARP}_S \left( f \right) \\
    \overset{\text{eq }\eqref{eq:sharp_dec}}&{=} g_S \left( \mathbb{P}_{Y} \right) - \mathbb{E} \left[ g_S \left( \mathbb{P}_{Y \mid f \left( X \right)} \right) \right] \\
    & = g_S \left( \mathbb{P}_{Y} \right) - \mathbb{E}_{m \circ f \left( X \right)} \left[ \mathbb{E}_{f \left( X \right)} \left[ g_S \left( \mathbb{P}_{Y \mid f \left( X \right)} \right) \mid m \circ f \left( X \right) \right] \right] \\
    \overset{\text{Jensen}}&{\geq} g_S \left( \mathbb{P}_{Y} \right) - \mathbb{E}_{m \circ f \left( X \right)} \left[ g_S \left( \mathbb{E}_{f \left( X \right)} \left[ \mathbb{P}_{Y \mid f \left( X \right)} \mid m \circ f \left( X \right) \right] \right) \right] \\
    & = g_S \left( \mathbb{P}_{Y} \right) - \mathbb{E}_{m \circ f \left( X \right)} \left[ g_S \left( \mathbb{E}_{f \left( X \right)} \left[ \mathbb{E} \left[ e_Y \mid f \left( X \right) \right] \mid m \circ f \left( X \right) \right] \right) \right] \\
    & = g_S \left( \mathbb{P}_{Y} \right) - \mathbb{E}_{m \circ f \left( X \right)} \left[ g_S \left( \mathbb{E} \left[ e_Y \mid m \circ f \left( X \right) \right] \right) \right] \\
    & = g_S \left( \mathbb{P}_{Y} \right) - \mathbb{E}_{m \circ f \left( X \right)} \left[ g_S \left( \mathbb{P}_{Y \mid m \circ f \left( X \right)} \right) \right] \\
    \overset{\text{eq }\eqref{eq:sharp_dec}}&{=} \text{SHARP}_S \left( m \circ f \right) \\
\end{split}
\end{equation}

\end{proof}

If the underlying score is the log score, then the sharpness is the mutual information between predictions and target random variable.
Consequently, we can interpret the sharpness as generalized mutual information.
This gives the proposition the following intuitive meaning:
There exists no function, that can transform a random variable in a way such that the mutual information with another random variable is increased.
Or, in other words, we cannot add 'information' to a random variable by transforming it in a deterministic way.

\section{Proper U-scores}
\label{app:u-scores}

In this section we introduce a generalization of proper scores.
Based on U-statistics, we define proper U-scores.
This allows us to naturally extend the definition of proper calibration errors to be based on proper U-scores instead of just proper scores.
Consequently, we can cover more calibration errors with desired properties.
For example, we can show that the squared KCE \citep{widmann2019calibration} is a proper calibration error based on a U-score (but not on a conventional score).
The squared KCE has an unbiased estimator, thus, this extension of the definition of proper calibration errors has substantial practical value.

\subsection{Background}

Let $X_1, \dots, X_n$ be $n$ iid random variables and $\phi \left( x_1, \dots, x_r \right)$ a function with $r \leq n$.
Let $\mathbf{P} = \left\{ a \in \left\{1, \dots, n \right\}^r \mid a_1 < \dots < a_r \right\}$ be the set of $r$ sized ordered permutations out of $n$, i.e. $\left\lvert \mathbf{P} \right\rvert = \binom{n}{r}$.
Then $U = \frac{1}{\left\lvert \mathbf{P} \right\rvert} \sum_{a \in \mathbf{P}} \phi \left( X_{a_1}, \dots, X_{a_r} \right)$ is a unbiased minimum-variance estimator (UMVE) of $\mathbb{E} \left[ \phi \left( X_1, \dots, X_r \right) \right]$ and called U-statistic \citep{10.1214/aoms/1177730196}.

\subsection{Contributions}

Assume we have two measure spaces $\left( \mathcal{X}, \mathcal{F}_X \right)$ and $\left( \mathcal{Y}, \mathcal{F}_Y \right)$, and corresponding $\mathcal{P}_X$ and $\mathcal{P}_Y$ sets of possible probability measures.
We want to score a conditional distribution $P \colon \mathcal{X} \to \mathcal{P}_Y$ given another conditional distribution $Q \colon \mathcal{X} \to \mathcal{P}_Y$.

\begin{definition}
    A \textbf{U-scoring rule} $S$ is a function of the form $$S \; \colon \; \mathcal{P}_Y^r \times \mathcal{Y}^r \to \overline{\mathbb{R}}$$
    with $r \in \mathbb{N}$ and $\overline{\mathbb{R}} \coloneq \mathbb{R} \cup \left\{- \infty, \infty \right\}.$
\end{definition}

It takes $r$ predictions and events and returns a score.
For $r=1$, U-scoring rules are scoring rules in the common definition.

\begin{definition}
    A \textbf{U-scoring function} $s_S$ based on a U-scoring rule $S$ is defined as 
    \begin{equation}
    \begin{split}
        s_S \; \colon \; \mathcal{P}_Y^{2r} & \to \overline{\mathbb{R}} \\
        \left(P_1, \dots, P_r, Q_1, \dots, Q_r \right) & \mapsto \int_{\mathcal{Y}^r} S \left(P_1, \dots, P_r, y_1, \dots, y_r \right) \mathrm{d} \left(Q_1 \times \dots \times Q_r \right) \left( y \right)
    \end{split}
    \end{equation}
\end{definition}

For $r=1$, U-scoring functions are scoring functions in the common definition.
If $Q_1, \dots, Q_r$ are the distributions of $Y_1, \dots, Y_r$ we can also write $s \left(P_1, \dots, P_r, Q_1, \dots, Q_r \right) = \mathbb{E} \left[ S \left(P_1, \dots, P_r, Y_1, \dots, Y_r \right) \right]$.

\begin{definition}
    A U-scoring function $s_S$ (and its U-scoring rule $S$) is defined to be \textbf{proper} if and only if 
    \begin{equation}
    \begin{split}
        & \forall \mathbb{P} \in \mathcal{P}_X, \quad X_1, \dots, X_r \overset{iid}{\sim} \mathbb{P}, \quad \forall P,Q \colon \mathcal{X} \to \mathcal{P}_Y \; \colon \\
        & \quad \quad \mathbb{E} s_S \left(P \left( X_1 \right), \dots, P \left( X_r \right), Q \left( X_1 \right), \dots, Q \left( X_r \right) \right) \\
        & \quad \quad \geq \mathbb{E} s_S \left(Q \left( X_1 \right), \dots, Q \left( X_r \right), Q \left( X_1 \right), \dots, Q \left( X_r \right) \right)
    \end{split}
    \end{equation}
    and \textbf{strictly proper} if and only if additionally
    \begin{equation}
    \begin{split}
        & \forall \mathbb{P} \in \mathcal{P}_X, \quad X_1, \dots, X_r \overset{iid}{\sim} \mathbb{P}, \quad \forall P,Q \colon \mathcal{X} \to \mathcal{P}_Y \; \colon \\
        & Q \neq P \\
        & \implies \mathbb{E} s_S \left(P \left( X_1 \right), \dots, P \left( X_r \right), Q \left( X_1 \right), \dots, Q \left( X_r \right) \right) \\
        & \quad \quad > \mathbb{E} s_S \left(Q \left( P_1 \right), \dots, Q \left( P_r \right), Q \left( P_1 \right), \dots, Q \left( P_r \right) \right) \\
    \label{eq:strct_prop_u}
    \end{split}
    \end{equation}
\end{definition}

In words, $s_S$ (or $S$) is proper if comparing $Q$ with itself gives the best expected value, and strictly proper if no other $P \neq Q$ can achieve this value.
The U-statistic of a proper $s_S$ is a UMVE \citep{10.1214/aoms/1177730196}.
For $r=1$, proper U-scores are identical to proper scores if $\mathcal{P}_X$ is sufficiently large.
This holds since for function $f \colon \mathcal{X} \to \mathbb{R}$ and appropriate $\mathcal{P}_X$ we have: $\left( \forall \mu \in \mathcal{P}_X \colon \int f \mathrm{d} \mu = 0 \right) \iff f = 0$.

\begin{definition}
    $g \left( Q_1, \dots, Q_r \right) = s \left(Q_1, \dots, Q_r, Q_1, \dots, Q_r  \right)$ is called the (generalized or associated) entropy.
\end{definition}

\begin{definition}
    Given a proper U-score $S$, the associated \textbf{U-divergence} $d$ is defined as
    \begin{equation}
    \begin{split}
        d_S \; \colon \; \mathcal{P}_Y^{2r} & \to \overline{\mathbb{R}}_{\geq 0} \\
        \left(P_1, \dots, P_r, Q_1, \dots, Q_r \right) & \mapsto s_S \left(P_1, \dots, P_r, Q_1, \dots, Q_r \right) - g_S \left( Q_1, \dots, Q_r \right). \\
    \end{split}
    \end{equation}
\end{definition}

If $S$ is a strictly proper U-score, $Q_1, \dots, Q_r$ iid and $P_1, \dots, P_r$ iid, then $\mathbb{E} d_S$ is zero if and only if $\forall i \in \left\{1, \dots, r \right\} \colon \; Q_i \overset{a.s.}{=} P_i$. This follows directly by setting $P_i = P \left( X_i \right)$ and $Q_i = Q \left( X_i \right)$ in equation \eqref{eq:strct_prop_u}. \\

Assuming $P_1, \dots, P_r$ are random variables and $\mathbb{P}_{Y \mid P_1}, \dots, \mathbb{P}_{Y \mid P_r} \in \mathcal{P}_Y$ are the conditional distribution of independent random variables $Y_1, \dots, Y_r \sim \mathbb{P}_Y$, where each $Y_i$ only depends on $P_i$.
Under the condition that $g_S \left( \mathbb{P}_Y, \dots, \mathbb{P}_Y \right), \mathbb{E} \left[ g_S \left( \mathbb{P}_{Y \mid P_1}, \dots, \mathbb{P}_{Y \mid P_r} \right) \right], \mathbb{E} \left[ \left\lvert S \left(P_1, \dots, P_r, Y_1, \dots, Y_r \right) \right\rvert \right], \mathbb{E} \left[ \left\lvert S \left(\mathbb{P}_Y, \dots, \mathbb{P}_Y, Y_1, \dots, Y_r \right) \right\rvert \right] < \infty$, we have the decomposition

\begin{equation}
\begin{split}
    & \mathbb{E} \left[ S \left(P_1, \dots, P_r, Y_1, \dots, Y_r \right) \right] \\
    & = \mathbb{E} \left[ s_S \left( P_1, \dots, P_r, \mathbb{P}_{Y \mid P_1}, \dots, \mathbb{P}_{Y \mid P_r} \right) \right] \\
    & = g_S \left( \mathbb{P}_{Y}, \dots, \mathbb{P}_{Y} \right) \\
    & \quad \quad + \mathbb{E} \left[ d_S \left( P_1, \dots, P_r, \mathbb{P}_{Y \mid P_1}, \dots, \mathbb{P}_{Y \mid P_r} \right) \right] \\
    & \quad \quad - \mathbb{E} \left[ d_S \left( \mathbb{P}_{Y}, \dots, \mathbb{P}_{Y}, \mathbb{P}_{Y \mid P_1}, \dots, \mathbb{P}_{Y \mid P_r} \right) \right]. \\
\end{split}
\end{equation}

Proof is identical to proof of Lemma \ref{lemma:cal_sharp_decomp}.
The first term is the generalized entropy, the second the calibration, and the third the sharpness term.

Thus, every proper U-score $S$ induces a proper calibration error defined as

\begin{equation}
\begin{split}
    & \text{CE}_S \left( f \right) \\
    & = \mathbb{E} \left[ d_S \left( f \left( X_1 \right), \dots, f \left( X_r \right), \mathbb{P}_{Y \mid f \left( X_1 \right)}, \dots, \mathbb{P}_{Y \mid f \left( X_r \right)} \right) \right] \\
    & \text{ with iid } X_1, \dots, X_r.
\end{split}
\end{equation}

Since proper U-scores are identical to proper scores for $r=1$, this definition of proper calibration errors does not contradict definitions or findings in the main paper.
For any strictly proper U-score $S$, CE$_S$ of model $f$ is zero if and only if $f$ is calibrated.
This directly follows from the property of the U-divergence.
But, it should be noted that we cannot assume every property holding for $r=1$ also holds for $r \in \mathbb{N}$.
Investigating this can be seen as potential future work.

An example with $r=2$:

For positive definite kernel matrix $k$, define

\begin{equation}
\begin{split}
    & S \left( P_1, P_2, y_1, y_2 \right) = \left( P_1 - e_{y_1} \right)^\intercal k \left( P_1, P_2 \right) \left( P_2 - e_{y_2} \right)
\end{split}
\end{equation}

which gives

\begin{equation}
\begin{split}
    & g_S \left( Q_1, Q_2 \right) = 0\\
\end{split}
\end{equation}

and

\begin{equation}
\begin{split}
    & d_S \left( P_1, P_2, Q_1, Q_2 \right) \\
    & = \left( P_1 - Q_1 \right)^\intercal k \left( P_1, P_2 \right) \left( P_2 - Q_2 \right) \\
\end{split}
\end{equation}

and the calibration term

\begin{equation}
\begin{split}
    & \mathbb{E} \left[ d_S \left( P_1, P_2, \mathbb{P}_{Y \mid P_1}, \mathbb{P}_{Y \mid P_2} \right) \right] \\
    & = \mathbb{E} \left[ \left( P_1 - \mathbb{P}_{Y \mid P_1} \right)^\intercal k \left( P_1, P_2 \right) \left( P_2 -  \mathbb{P}_{Y \mid P_2} \right) \right] \\
\end{split}
\end{equation}

If $P_1, P_2 \sim \mathbb{P}_{f \left( X \right)}$, then this is the squared KCE (SKCE) of $f$ \citep{widmann2019calibration}.
\citet{widmann2019calibration} proved that the SKCE is zero if and only if $f$ is calibrated, and $f$ is calibrated if $f \left( X \right) = \mathbb{P}_{Y \mid f \left( X \right)}$, which includes $f \left( X \right) = \mathbb{P}_{Y \mid X}$.
Consequently, the associated divergence is not uniquely minimized by the target distribution.
Thus, the score of the SKCE is proper but not strictly proper.

Interestingly, $\mathbb{E} \left[ d_S \left( \mathbb{P}_{Y}, \mathbb{P}_{Y}, \mathbb{P}_{Y \mid f \left( X \right)}, \mathbb{P}_{Y \mid f \left( X^\prime \right)} \right) \right] = 0$ for $X, X^\prime$ iid, i.e. the score of the SKCE only measures calibration and ignores sharpness.
This fact is consistent with all previous findings of the SKCE.

\section{Extended experiments}
\label{app:plots}

In this section, we provide further details of the experimental setup and report additional results.
This includes results in the squared space, where the upper bound estimator is minimum-variance unbiased.
Further, we present results on the Friedman 1 regression problem, which is also used by \citet{widmann2021calibration}.

\subsection{Details on the ECE estimator simulation in Figure \ref{fig:ece_sim}}

We simulate model predictions of a 100 class classification problem with validation set size of 10'000 and test set size of 10'000.
For this, we sample the model predictions from a multivariate logistic normal distribution \citep{10.2307/2335470}, since it is a lot more flexible in its covariance matrix than a dirichlet distribution.
This brings the samples closer to real-world model predictions.
We sample the covariance matrix from an inverse-wishart distribution with a scale matrix of $I_{100} / 0.01$.
The scale matrix was tempered in such a way to receive model predictions with $\approx$ 87.6 \% classification accuracy.
We will explain the label sampling in the following.
Again, we aimed for realistic values.

Now, we want a model-target relation of which we know that the model is calibrated.
For this, we iterate over every model prediction and use each model prediction as a categorical distribution from which we sample the label.
Consequently, each model prediction is the ground truth distribution of each label.
This gives us calibrated prediction-target pairs, which we used to estimate the ECE of the perfectly calibrated 'model' in Figure \ref{fig:ece_sim} (blue line).
Next, to gradually decrease the level of calibration, we scale the predictions via different temperatures in the logit space.
Thus, we know that the 'model' of mediocre calibration (orange line) is worse than the 'model' of perfect calibration, and better than the 'model' of bad calibration (green line).

\subsection{Details on experimental setup of Section \ref{sec:exp}}
The experiments are conducted across several model-dataset combinations, for which logit sets are openly accessible \footnote{https://github.com/markus93/NN\_calibration/ and https://github.com/AmirooR/IntraOrderPreservingCalibration} \citep{kull2019beyond, rahimi2020intra}.
This includes the models LeNet 5 \citep{lecun1998gradient}, ResNet 50 (with and without pretraining), ResNet 50 NTS, ResNet 101 (with and without pretraining) ResNet 110, ResNet 110 SD, ResNet 152, ResNet 152 SD \citep{he2016deep}, Wide ResNet 32 \citep{zagoruyko2016wide}, DenseNet 40, DenseNet 161 \citep{huang2017densely}, and PNASNet5 Large \citep{liu2018progressive} and the datasets CIFAR10, CIFAR100 \citep{krizhevsky2009learning}, and ImageNet \citep{deng2009imagenet}.
We did not conduct model training by ourselves, and refer to \citep{kull2019beyond} and \citep{rahimi2020intra} for further details.
Validation and test set splits are predefined in every logit set.
We include TS, ETS, and DIAG as injective recalibration methods.
For optimization of TS and ETS, we modified the available implementation of \citet{zhang2020mix} and used the validation set as calibration set.
For DIAG, we used the exact implementation of \citet{rahimi2020intra}.

For every dataset we investigate ten ticks of different (sampled) test set sizes.
The ticks are determined to be equally apart in the $\log_2$ space.
The minimum is always 100 and the maximum the full available test set size.
We use repeated sampling with subsequent averaging to counteract the increased estimation variance for low test set sizes.
The estimated standard errors are also shown in the plots, but they are often barely visible.
The number of samples in each tick is along the following:

\begin{itemize}
    \item Tick 1 ($n=100$): 20000
    \item Tick 2: 15842
    \item Tick 3: 12168
    \item Tick 4: 8978
    \item Tick 5: 6272
    \item Tick 6: 4050
    \item Tick 7: 2312
    \item Tick 8: 1058
    \item Tick 9: 288
    \item Tick 10 (full test set): 2
\end{itemize}

The seeds for the sampling of the experiments have been saved.
Since we choose the amount of samples such that the estimation standard error is low, we expect similar results no matter the chosen seed.

All experiments have been computed on a machine with 1007 GB RAM and two Intel(R) Xeon(R) Gold 6230R CPU @ 2.10GHz.

\subsection{Estimated model calibration}

Calibration errors according to different estimators and for different model-dataset combinations are shown in figure \ref{fig:ce_root_1} and first row of figure \ref{fig:ce_squared_1} (in squared space).
These experiments confirm that the proposed upper bound is stable across a multitude of settings.

\begin{figure*}[h]
\vskip 0.2in
\centering
    \begin{subfigure}{.33\textwidth}
    \centering
    \includegraphics[width=\columnwidth]{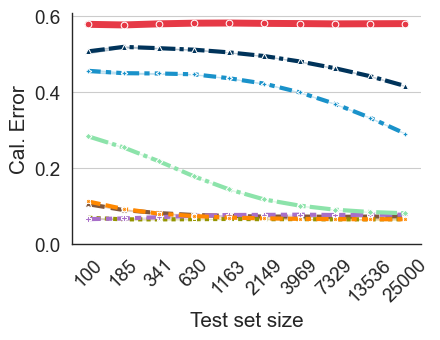}
    \caption{ResNet 152 on ImageNet}
    \end{subfigure}%
    \begin{subfigure}{.33\textwidth}
    \centering
    \includegraphics[width=\columnwidth]{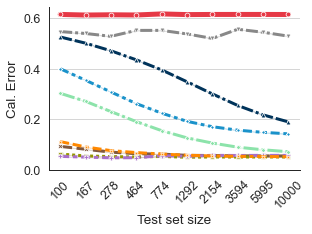}
    \caption{LeNet 5 on CIFAR10}
    \end{subfigure}%
    \begin{subfigure}{.33\textwidth}
    \centering
    \includegraphics[width=\columnwidth]{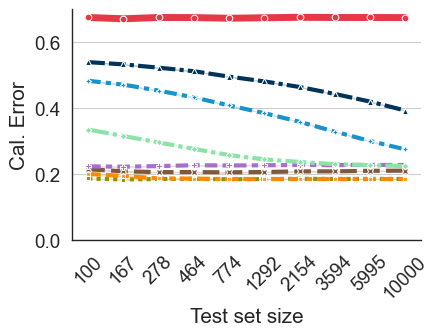}
    \caption{ResNet 110 on CIFAR100}
    \end{subfigure} \\
    \begin{subfigure}{.33\textwidth}
    \centering
    \includegraphics[width=\columnwidth]{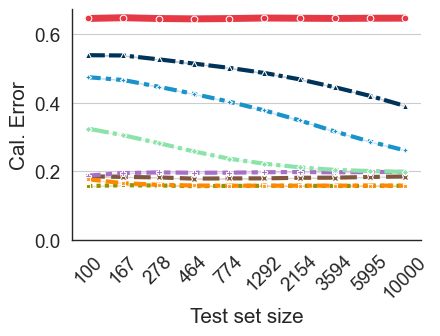}
    \caption{ResNet 110 SD on CIFAR100}
    \end{subfigure}%
    \begin{subfigure}{.33\textwidth}
    \centering
    \includegraphics[width=\columnwidth]{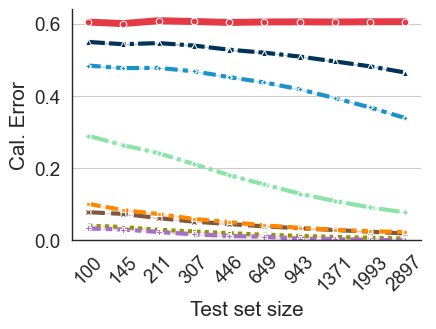}
    \caption{ResNet 50 on BIRDS}
    \end{subfigure}%
    \begin{subfigure}{.33\textwidth}
    \centering
    \includegraphics[width=\columnwidth]{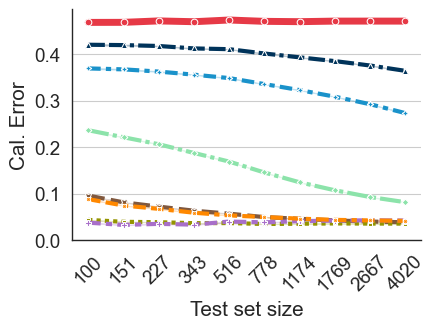}
    \caption{ResNet 101 on CARS}
    \end{subfigure} \\
    \begin{subfigure}{.33\textwidth}
    \centering
    \includegraphics[width=\columnwidth]{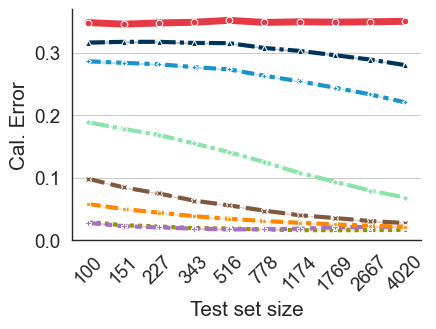}
    \caption{Pretrained ResNet 101 on CARS}
    \end{subfigure}%
    \begin{subfigure}{.33\textwidth}
    \centering
    \includegraphics[width=\columnwidth]{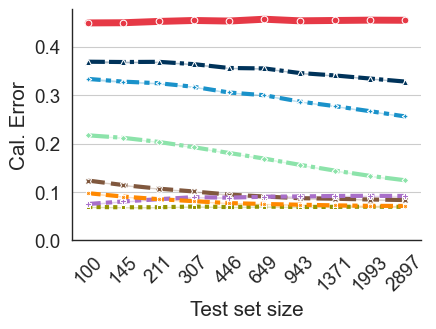}
    \caption{ResNet 50 NTS on BIRDS}
    \end{subfigure} %
    \begin{subfigure}{.33\textwidth}
    \centering
    \includegraphics[width=\columnwidth]{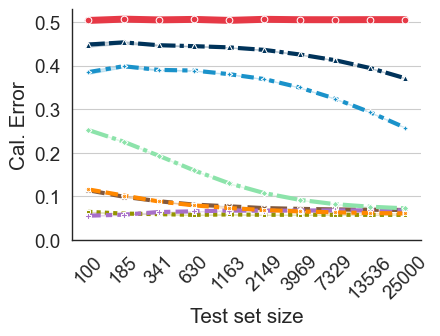}
    \caption{PNASNet5 Large on ImageNet}
    \end{subfigure} \\
    \begin{subfigure}{.33\textwidth}
    \centering
    \includegraphics[width=\columnwidth]{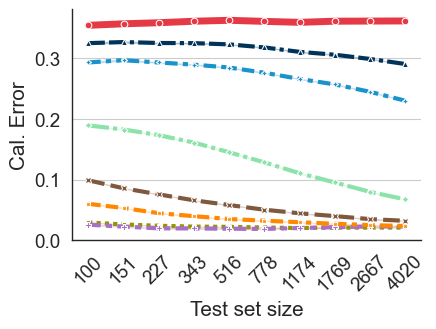}
    \caption{Pretrained ResNet 50 on CARS}
    \end{subfigure}%
    \begin{subfigure}{.33\textwidth}
    \centering
    \includegraphics[width=\columnwidth]{figures/all_resnet110_SD_c100.png}
    \caption{ResNet 110 SD on CIFAR100}
    \end{subfigure}%
    \begin{subfigure}{.33\textwidth}
    \centering
    \includegraphics[width=\columnwidth]{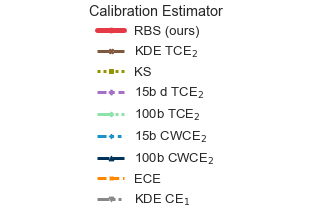}
    \end{subfigure}
\caption{Different calibration error estimates versus the test set size. The red line corresponds to the square root of the Brier score which is an upper bound of $\text{CE}_2$. The other estimators are lower bounds.}
\vskip -0.2in
\label{fig:ce_root_1}
\end{figure*}

\subsection{Recalibration improvement}
In the main text we investigated recalibration improvement of common estimators for the calibration error and compared their reliability to RBS. According to Proposition \ref{prop:ub_grad} and since RBS is asymptotically unbiased and consistent, it can be regarded as a reliable approximation of the real improvement of the recalibration methods. However, if we move to the squared space, our proposed upper bound is even provably reliable since it has a minimum-variance unbiased estimator. This motivates further experiments comparing existing calibration errors in the squared space, which we describe in the following. Here, we first report additional results comparing common estimators to RBS; we then report results in the squared space. We start with a formal description of the problem and experimental setup.\\

Let $D$ be a sampled subset of the full test set.
Let $f$ be the underlying model and $h$ an optimized recalibration method.
Let $e$ be an calibration error estimator taking a dataset and a model as inputs.
The recalibration improvement according to estimator $e$ is estimated via $e \left(D, f \right) - e \left(D, h \circ f \right)$.

\paragraph{Recalibration improvement of common estimators} We compute the recalibration improvement of common estimators on several test set samples of a given size and plot the average of these on the y-axis. We extend the results reported in the main text by covering additional datasets, models and architectures. These extended experiments confirm the findings reported in the main text, namely that only RBS reliably quantifies the improvement in calibration error after recalibration (Fig. \ref{fig:RC_delta_1}; standard errors are shown).
\paragraph{Recalibration improvement in the squared space}The recalibration improvement in the squared space according to estimator $e$ is estimated via $\left(e \left(D, f \right)\right)^2 - \left(e \left(D, h \circ f \right) \right)^2$.
The results are depicted in Figure \ref{fig:ce_squared_1}.
For CIFAR10, we also included the KDE estimator for CE$_2^2$ according to \citep{popordanoska2022}.
Only the Brier score (square of RBS) yields provably unbiased estimates of the true recalibration improvement w.r.t. CE$_2^2$.
In contrast to our approach, all other estimators are sensitive to test set size and/or substantially underestimate the true recalibration improvement in squared space.

For larger subsets of the CIFAR100 test set, the automatic bandwidth optimization for KDE CE$_2$ does not return a valid bandwidth.
These cases are omitted from Figure \ref{fig:sq_ETS_r32_c100} and \ref{fig:sq_ETS_RC_r32_c100}.
We also omitted KDE CE$_1$ as it shows similar behaviour as KDE CE$_2$ but is shifted substantially towards the negative like in the CIFAR10 case (Fig. \ref{fig:sq_TS_RC_d40_c10}).

\begin{figure*}[h]
\vskip 0.2in
\centering
    \begin{subfigure}{.33\textwidth}
    \centering
    \includegraphics[width=\columnwidth]{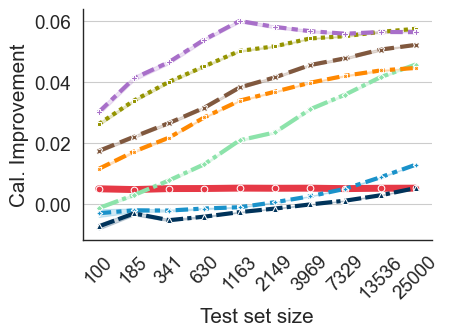}
    \caption{TS applied to ResNet 152 \\ on ImageNet}
    \end{subfigure}%
    \begin{subfigure}{.33\textwidth}
    \centering
    \includegraphics[width=\columnwidth]{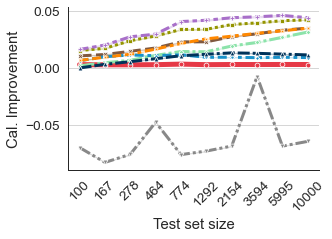}
    \caption{TS applied to LeNet 5 \\ on CIFAR10}
    \end{subfigure}%
    \begin{subfigure}{.33\textwidth}
    \centering
    \includegraphics[width=\columnwidth]{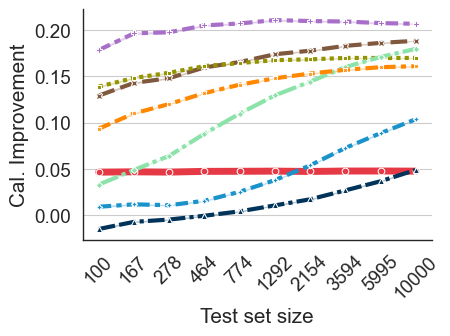}
    \caption{TS applied to ResNet 110 \\ on CIFAR100}
    \end{subfigure} \\
    \begin{subfigure}{.33\textwidth}
    \centering
    \includegraphics[width=\columnwidth]{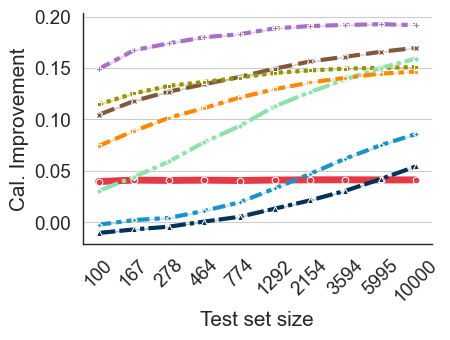}
    \caption{TS applied to ResNet 110 SD \\ on CIFAR100}
    \end{subfigure}%
    \begin{subfigure}{.33\textwidth}
    \centering
    \includegraphics[width=\columnwidth]{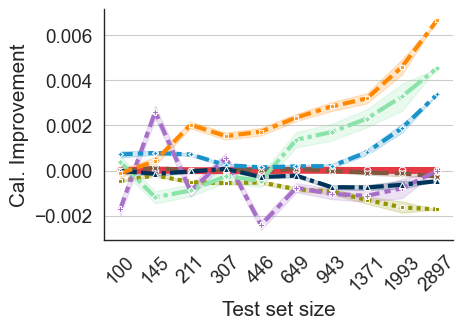}
    \caption{TS applied to ResNet 50 \\ on BIRDS}
    \end{subfigure}%
    \begin{subfigure}{.33\textwidth}
    \centering
    \includegraphics[width=\columnwidth]{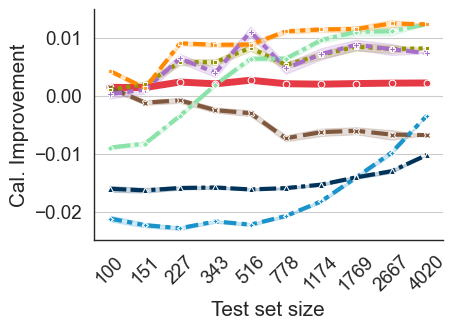}
    \caption{DIAG applied to ResNet 101 \\ on CARS}
    \end{subfigure} \\
    \begin{subfigure}{.33\textwidth}
    \centering
    \includegraphics[width=\columnwidth]{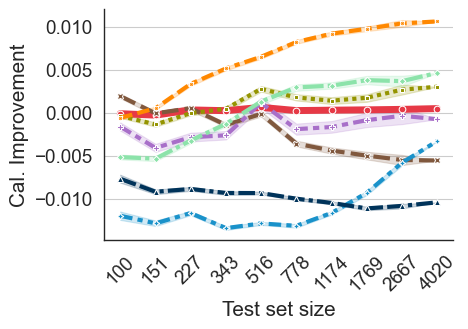}
    \caption{DIAG applied to \\ pretrained ResNet 101 on CARS}
    \end{subfigure}%
    \begin{subfigure}{.33\textwidth}
    \centering
    \includegraphics[width=\columnwidth]{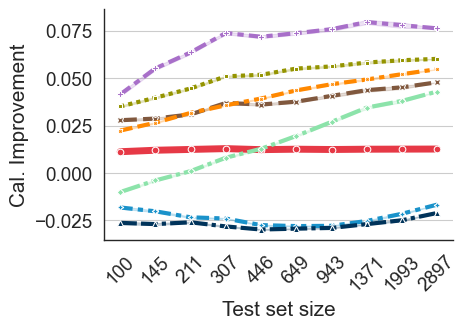}
    \caption{DIAG applied to \\ ResNet 50 NTS on BIRDS}
    \end{subfigure} %
    \begin{subfigure}{.33\textwidth}
    \centering
    \includegraphics[width=\columnwidth]{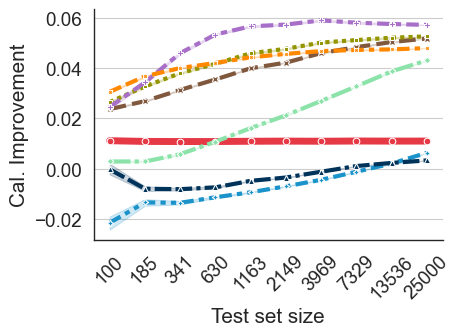}
    \caption{DIAG applied to PNASNet5 Large \\ on ImageNet}
    \end{subfigure} \\
    \begin{subfigure}{.33\textwidth}
    \centering
    \includegraphics[width=\columnwidth]{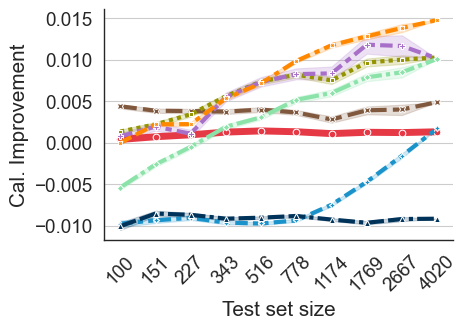}
    \caption{DIAG applied to \\ pretrained ResNet 50 on CARS}
    \end{subfigure}%
    \begin{subfigure}{.33\textwidth}
    \centering
    \includegraphics[width=\columnwidth]{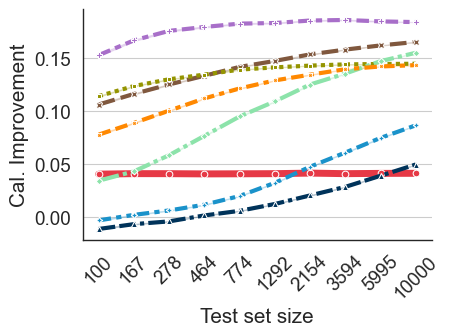}
    \caption{ETS applied to ResNet 110 SD \\ on CIFAR100}
    \end{subfigure}%
    \begin{subfigure}{.33\textwidth}
    \centering
    \includegraphics[width=\columnwidth]{figures/legend.png}
    \end{subfigure}
\caption{Different calibration improvement estimates versus the test set size. The red line corresponds to the square root of the Brier score.}
\vskip -0.2in
\label{fig:RC_delta_1}
\end{figure*}

\begin{figure*}[h]
\vskip 0.2in
\centering
    \begin{subfigure}{.33\textwidth}
    \centering
    \includegraphics[width=\columnwidth]{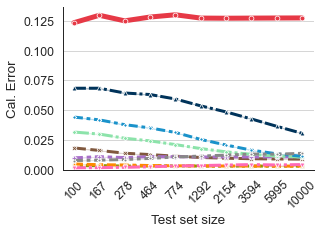}
    \caption{DenseNet 40 on CIFAR10}
    \end{subfigure}%
    \begin{subfigure}{.33\textwidth}
    \centering
    \includegraphics[width=\columnwidth]{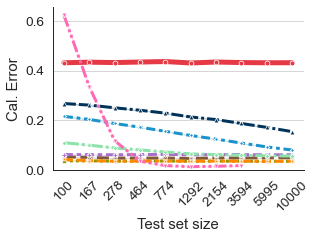}
    \caption{Wide ResNet 32 on CIFAR100}
    \label{fig:sq_ETS_r32_c100}
\end{subfigure}%
    \begin{subfigure}{.33\textwidth}
    \centering
    \includegraphics[width=\columnwidth]{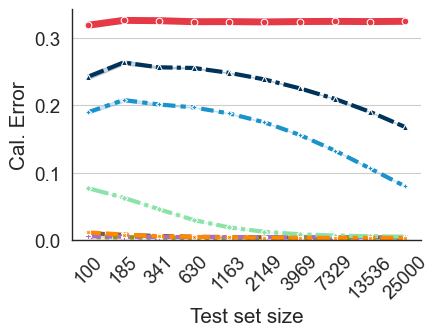}
    \caption{DenseNet 161 on ImageNet}
    \end{subfigure} \\
    \begin{subfigure}{.33\textwidth}
    \centering
    \includegraphics[width=\columnwidth]{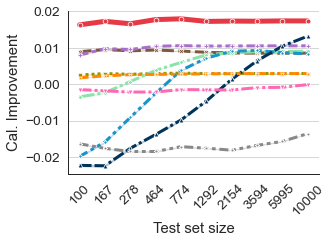}
    \caption{DenseNet 40 on CIFAR10}
    \label{fig:sq_TS_RC_d40_c10}
    \end{subfigure}%
    \begin{subfigure}{.33\textwidth}
    \centering
    \includegraphics[width=\columnwidth]{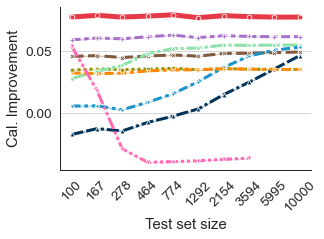}
    \caption{Wide ResNet 32 on CIFAR100}
    \label{fig:sq_ETS_RC_r32_c100}
    \end{subfigure}%
    \begin{subfigure}{.33\textwidth}
    \centering
    \includegraphics[width=\columnwidth]{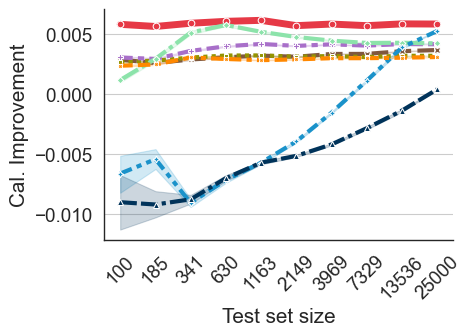}
    \caption{DenseNet 161 on ImageNet}
    \end{subfigure} \\
    \begin{subfigure}{.33\textwidth}
    \centering
    \includegraphics[width=\columnwidth]{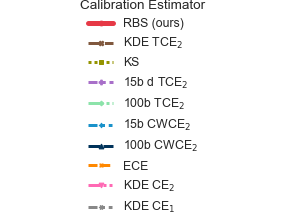}
    \end{subfigure}
\caption{
\textbf{First row:} Different squared calibration error estimates versus the test set size. The red line corresponds to the Brier score which is an upper bound of $\text{CE}_2^2$. The other errors are lower bounds. \textbf{Second row:} Estimated improvements in the squared space of injective recalibration methods in different settings. Our approach captures the true improvement w.r.t. $\text{CE}_2^2$ in an unbiased manner.
}
\vskip -0.2in
\label{fig:ce_squared_1}
\end{figure*}

\subsection{Variance regression}

In the following, we give more details on the variance regression experiment in the main paper, but also add further results of the Friedman 1 regression problem.

In all following scenarios, we are interested in the effect of recalibration towards predictive uncertainty.
For this, we use Platt scaling ($x \to wx + b$ with parameters $w,b \in \mathbb{R}$) of the variance output and optimize $w$ and $b$ with the L-BFGS optimizer on the validation set.
Further, since Platt scaling is injective, we apply Theorem \ref{th:ub} and Proposition \ref{prop:ub_grad} to treat the DSS score as an calibration error for recalibration.
Consequently, optimizing Platt scaling with the DSS score is equivalent to optimizing the associated calibration error.

We will use this recalibration procedure in each iteration during model training, but without modifying the model for the next training step.

\citet{widmann2021calibration} used the MSE as training objective, while we use the DSS, as it is a natural extension of the MSE to variance regression.

We repeat each experiment with five distinct seeds and aggregate the results, giving the characteristic error bands in each figure.

\subsubsection{UCI dataset \textit{Residential Building}}

The Residential Building dataset consists of 107 features and 372 data instances.
To have similar conditions as the Friedman 1 regression problem in the next section, we split the dataset into a training, validation, and test set with sizes 100, 100, and 173.
We use a fully-connected mixture density network as \citet{widmann2021calibration}, except we are also using an output node for the variance prediction, and reduce its size for a more stable training.
Specifically, it consists of 107 input nodes, 200 hidden nodes, and 2 output nodes.
Similar to \citet{widmann2021calibration}, we use Adam as model optimizer with default parameters ($0.001$ learning rate, $0.9$ first momentum decay, $0.999$ second momentum decay).

We show similar results as in Figure \ref{fig:recal_regr} but with aggregations from different runs with distinct seeds.
The evaluations are depicted in \ref{fig:app_recal_regr} and repeat the findings in the main paper.

\begin{figure*}[t]
\vskip 0.2in
\centering
    \begin{subfigure}{.33\textwidth}
    \centering
    \includegraphics[width=\columnwidth]{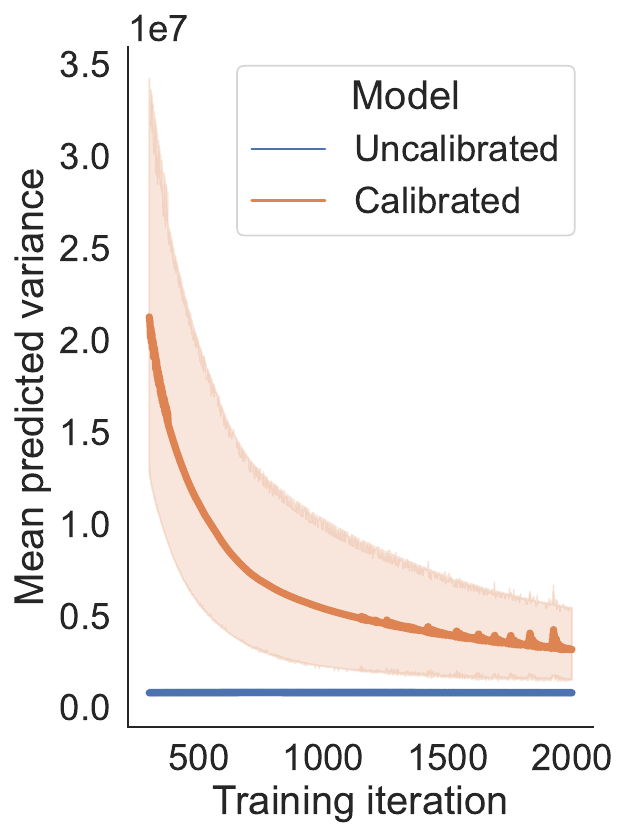}
    \end{subfigure}%
    \begin{subfigure}{.34\textwidth}
    \centering
    \includegraphics[width=\columnwidth]{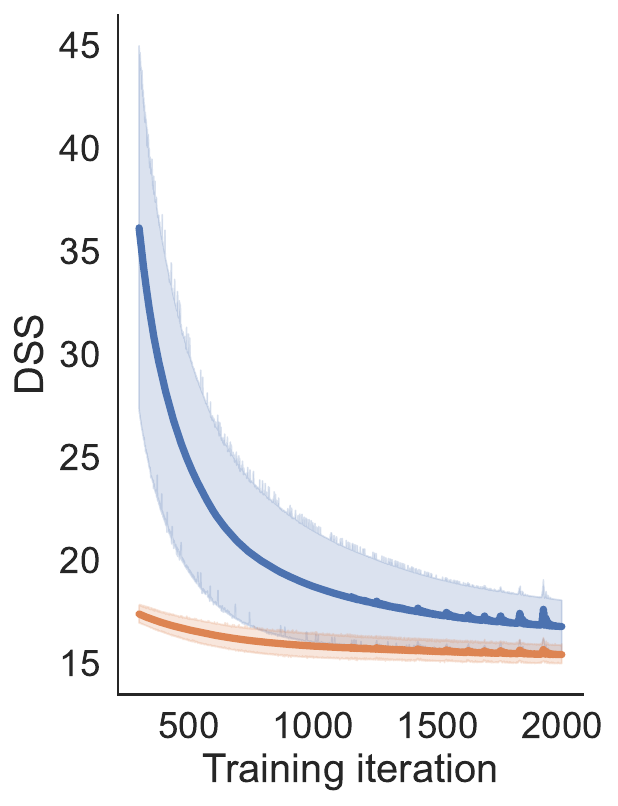}
    \end{subfigure}%
    \begin{subfigure}{.32\textwidth}
    \centering
    \includegraphics[width=\columnwidth]{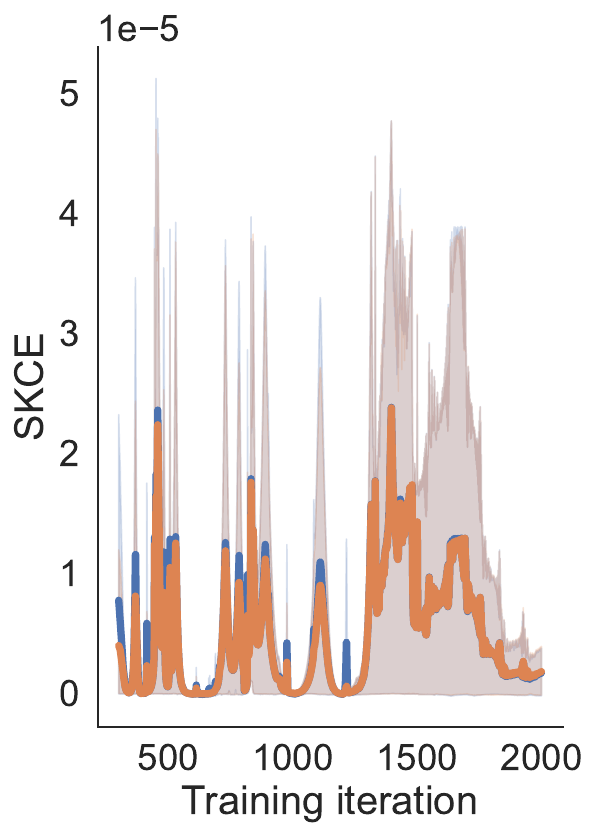}
    \end{subfigure}
\caption{
    \textbf{Left:} Average predicted variance throughout model training before and after recalibration.
    Initially, due to a bad fit, recalibration adjusts the variance accordingly for better communicated uncertainty.
    Once the model fit improves, the predicted variance requires less adjustment due to less uncertainty in each prediction.
    \textbf{Middle:} DSS communicates reasonably changes in the variance due to recalibration.
    \textbf{Right:} SKCE fails to capture the variance trend and behaves erratically.
}
\label{fig:app_recal_regr}
\vskip -0.2in
\end{figure*}

\subsubsection{Friedman 1}

The Friedman 1 regression problem consists of ten feature variables but only five influence the target variable \citep{friedman1991multivariate}.
The target variable is given by

\begin{equation}
    Y = 10 \sin( \pi X_1 X_2) + 20 (X_3 - 0.5)^2 + 10 X_4 + 5 X_5 + \epsilon
\end{equation}

with input $X_i \sim U\left(0, 1 \right)$ independently uniformly distributed for $i=1, \dots, 10$, and noise $\epsilon \sim \mathcal{N} \left(0, 1 \right)$.
It was used lately to investigate model calibration in the regression case \citep{widmann2021calibration}.
We slightly modify the Friedman 1 regression problem to have varying variance depending on the sixth feature variable, i.e. $\epsilon \sim \mathcal{N} \left(0, 0.5 + X_6 \right)$.
This makes it non-trivial to give an estimate of the predictive uncertainty in the form of the predicted variance.
We sample a training, validation, and test set, each of size 100 similar to \citet{widmann2021calibration}.

We use the same fully-connected mixture density network as \citet{widmann2021calibration}, except we are also using an output node for the variance prediction.
Further, we use the same training details as \citet{widmann2021calibration}.
We repeat each run three times and aggregate the results.

We again compare DSS, SKCE, and average predicted variance throughout model training with and without recalibration.
We depict the performance according to various errors during model training in Figure \ref{fig:recal_regr_f}.
As can be seen, recalibration adjusts overfitting of the predicted variance.
Consequently, the uncertainty communication of the model is improved.
Further, the SKCE seems to be less influenced by the variance calibration and more so by the mean calibration.
This is a significant drawback when uncertainty communication is done via the predicted variance.

\begin{figure*}[t]
\vskip 0.2in
\centering
    \begin{subfigure}{.33\textwidth}
    \centering
    \includegraphics[width=\columnwidth]{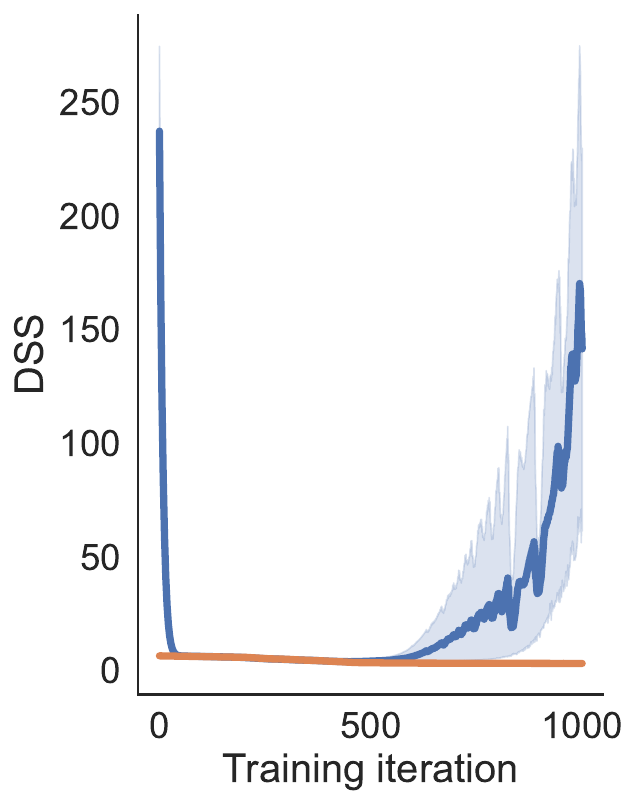}
    \end{subfigure}%
    \begin{subfigure}{.34\textwidth}
    \centering
    \includegraphics[width=\columnwidth]{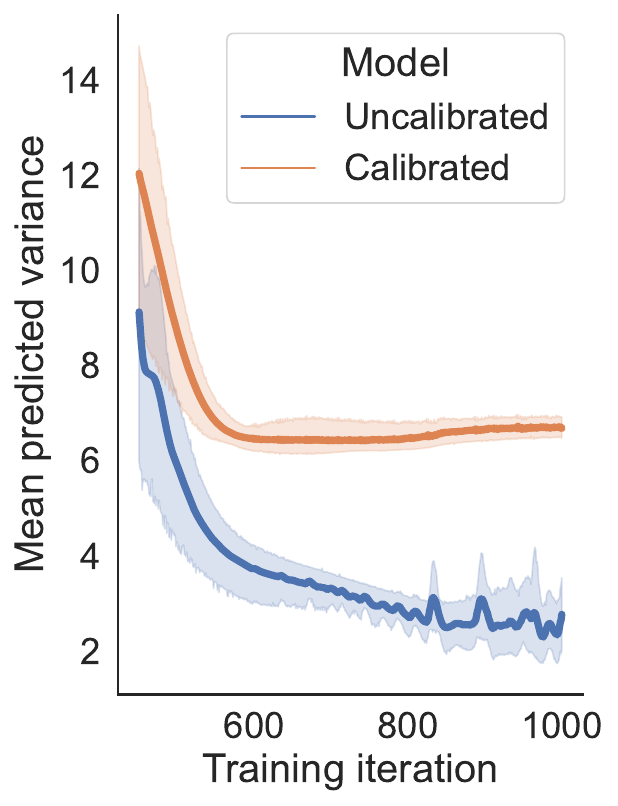}
    \end{subfigure} \\
    \begin{subfigure}{.32\textwidth}
    \centering
    \includegraphics[width=\columnwidth]{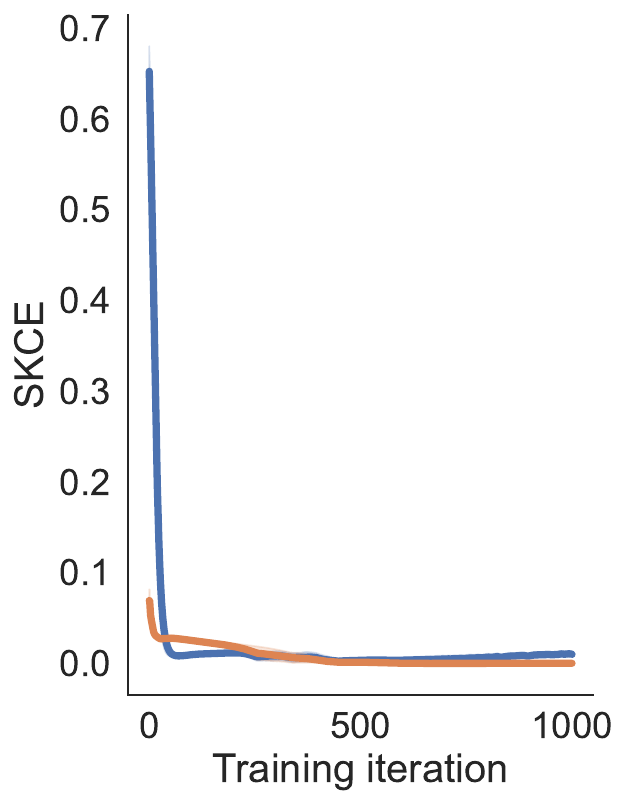}
    \end{subfigure}
    \begin{subfigure}{.32\textwidth}
    \centering
    \includegraphics[width=\columnwidth]{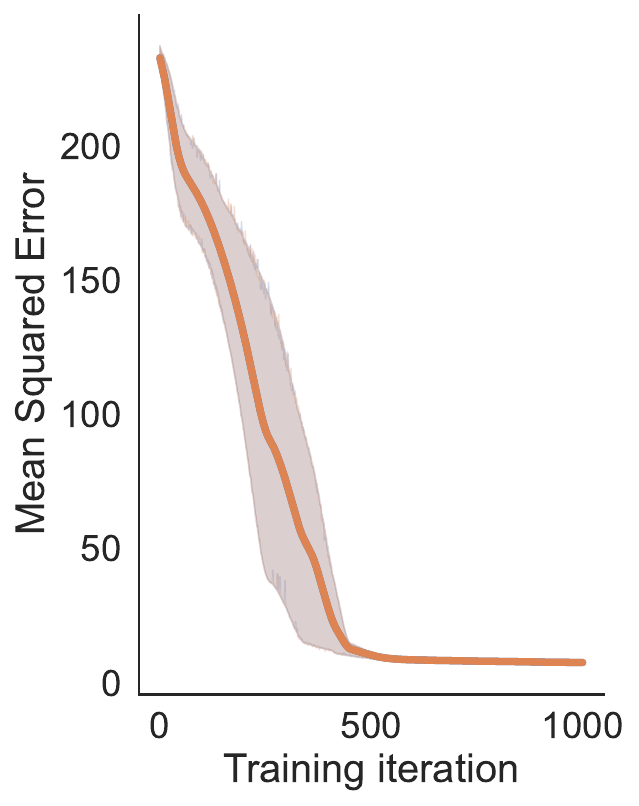}
    \end{subfigure}
\caption{
    \textbf{Upper left:} DSS shows that overfitting occurs at some point during training.
    Recalibration successfully adjusts this overfit. This indicates that most of the overfit is regarding variance and not mean prediction.
    \textbf{Upper right:} Average predicted variance starting from the point of overfitting.
    Recalibration adjusts the steadily decreasing predicted variance to a constant level.
    \textbf{Lower left:} SKCE signals improved calibration at the start of training but remains relatively unchanged by the variance overfit.
    \textbf{Lower right:} The MSE curve confirms that the predicted mean is not overfitted and suggests the SKCE is more sensitive to the calibration of the mean than the calibration of the variance estimate.
    Our recalibration does not influence the predicted mean, thus we omit the recalibrated model from this subfigure.
}
\label{fig:recal_regr_f}
\vskip -0.2in
\end{figure*}

\end{document}